\documentclass[10pt,twocolumn,letterpaper]{article}

\usepackage{cvpr}
\usepackage{times}
\usepackage{graphicx}
\usepackage{amsmath}
\usepackage{amssymb}
\usepackage{color}
\usepackage{makecell}
\usepackage{multicol,lipsum}
\usepackage{titling}

\newtheorem{theorem}{Theorem}

\newtheorem{lemma}{Lemma}
\newtheorem{proposition}{Proposition}

\newtheorem{definition}{Definition}
\newtheorem{remark}{Remark}
\newenvironment{proof}{\paragraph{Proof:}}{\hfill$\square$}

\newcommand{\bigzero}{\mbox{\normalfont\Huge  $0$}}
\newcommand{\bigeye}{\mbox{\normalfont\Huge  $I$}}
\newcommand{\rvline}{\hspace*{-\arraycolsep}\vline\hspace*{-\arraycolsep}}
\usepackage[pagebackref=true,breaklinks=true,colorlinks,bookmarks=false]{hyperref}

\date{}

\begin{document}

\title{\textbf{On the Instability of Relative Pose Estimation and RANSAC’s Role}}

\author{Hongyi Fan\\
School of Engineering\\
Brown University\\
{\tt\small hongyi\_fan@brown.edu}
\and
Joe Kileel\\
Department of Mathematics\\
University of Texas at Austin\\
{\tt\small jkileel@math.utexas.edu}
\and
Benjamin Kimia\\
School of Engineering\\
Brown University\\
{\tt\small benjamin\_kimia@brown.edu}
}

\maketitle
\pagenumbering{arabic}

\begin{abstract}
   In this paper we study the numerical instabilities of the 5- and 7-point problems for essential and fundamental matrix estimation in multiview geometry. In both cases we characterize the ill-posed world scenes where the condition number for epipolar estimation is infinite.  We also characterize the ill-posed instances in terms of the given image data. To arrive at these results, we present a general framework for analyzing the conditioning of minimal problems in multiview geometry, based on Riemannian manifolds. Experiments with synthetic and real-world data then reveal a striking conclusion: that Random Sample Consensus (RANSAC) in Structure-from-Motion (SfM) does not only serve to filter out outliers, but RANSAC also selects for well-conditioned image data, sufficiently separated from the ill-posed locus that our theory predicts.  Our findings suggest that, in future work, one could try to accelerate and increase the success of RANSAC by testing only well-conditioned image data.
\end{abstract}

\section{Introduction}
The past two decades have seen an explosive growth of multiview geometry applications such as the reconstruction of 3D object models for use in
video games~\cite{Ablan:3DPhoto:book},
film~\cite{Kitagawa:Mocap:book},
archaeology~\cite{pollefeys2001image},
architecture~\cite{Luhmann:Photogrammetry:book}, and urban modeling (\eg, Google
Street View); match-moving in augmented reality and cinematography for mixing virtual content and real video~\cite{Dobbert:Matchmoving:book}; the organization of a collection of photographs with respect to a scene known as SfM \cite{myozyecsil2017survey}
(\eg, as pioneered in photo tourism~\cite{agarwal2011building});
robotic manipulation~\cite{Horn:Robot:Vision}; and meteorology from cameras in
automobile manufacture and autonomous driving~\cite{Luhmann:Photogrammetry:book}. One of the key building block of a multiview system is the relative pose estimation of two cameras~\cite{hartleyzisserman,szeliski2010computer}. A methodology that is dominant in applications is the use of RANSAC~\cite{raguram2008comparative} to form hypotheses from a few randomly selected correspondences in two views, say 5 in calibrated camera pose estimation~\cite{nister:PAMI:2004} and 7 in uncalibrated camera pose estimation~\cite{stewart1999robust, myseitz2006comparison}, and validate these hypotheses using the remaining putative correspondence. The chief stated reason for using RANSAC is robustness against outliers, see the recent works~\cite{barath2020magsacpp, mishkintalk,brachmann2019neural}. 
The pose of multiple cameras can then be recovered in either a locally incremental~\cite{schoenberger2016sfm} or globally averaging manner~\cite{kasten2019algebraic}. 
This approach has been quite successful in many applications. 
\begin{figure}[t]
    \centering
    (a) \includegraphics[width=0.42\linewidth]{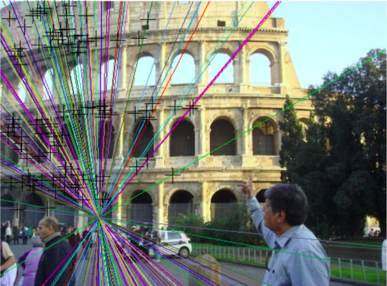}  \includegraphics[width=0.42\linewidth]{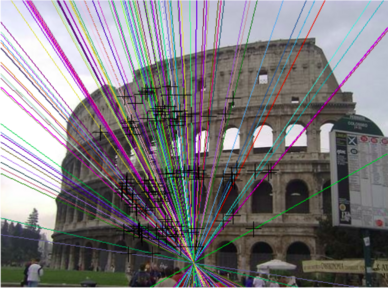} \\
    (b) \includegraphics[width=0.42\linewidth]{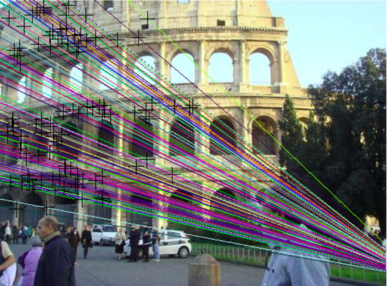}  \includegraphics[width=0.42\linewidth]{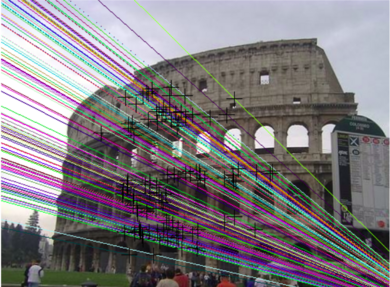}
    \caption{Erroneous relative pose estimation result on real data, given 100 image correspondences all of which are inliers. (a) Ground truth epipolar geometry. (b) Erroneous estimated epipolar geometry from the 7-point algorithm and LO-RANSAC \cite{chum2003locally}.}
    \label{fig:Teaser}
\end{figure}

There are, however, a non-negligible number of scenarios where this approach fails, \eg, in producing the relative pose between two cameras. As an example, when the number of candidate correspondences drops to below (say) 50 correspondences, for images of homogeneous and low textured surfaces, the pose estimation process fails. Similarly, when there is repeated texture in the scene, there is a large number of outlier candidate correspondences and again the process fails. It is curious why the estimation should fail, even if only a few correspondences are available: after all RANSAC can select from $\binom{50}{5} \approx 2.1$ million combinations, so there are plenty of veritable correspondences available if the ratio of outliers is low. In an experiment with no outliers, either with synthetic or real data (see Figure~\ref{fig:Teaser}), we discovered that in fact the process still fails! This is quite puzzling, unless the role of RANSAC goes beyond weeding out the outliers. Indeed, we will argue that a main role for RANSAC is to stabilize the estimation process. We will show that the process of estimating pose is typically unstable, with a gradation of instability depending on the specific choice of 5 points or 7 points. The role of RANSAC is to integrate non-selected matches to gauge the veritability of the unstable estimation process outcome: if a large number of non-selected candidate matches agree, then the hypothesis is both free of outliers but -- perhaps more importantly -- it is a stable estimate. 
Our results suggest that instability is the reason why so many points are needed: not to find outlier-free data, but to evaluate the stability of the hypothesis.  Similarly, in a scenario with repeated texture, the sheer number of outliers both reduces the chance of forming a veritable hypothesis, but also the ability to assess \nolinebreak stability. 

In this paper,  we inspect the general issue of numerical stability in minimal problems in multiview geometry.  We build a  framework that connects well-conditioned minimal point configurations with the condition number of the inverse Jacobian of a forward projection map. 
Using this framework, we compute  condition number formulas for the 5-point and 7-point minimal problems.  Further, we  investigate the issue of ill-posedness, \ie when the condition number is infinite.  We obtain characterizations for a world scene to be ill-posed, and requirements for a minimal image point configuration to be ill-posed.  Along with these theoretical results, we propose a way to measure the stability of a minimal image point set, by measuring the distance from one point to a ``degenerate curve" on the image computed using the other points. 
We note some analysis of the degeneracies of two-view geometry has already appeared, \eg \cite{kahl2002critical, maybank2012theory}.  
Our analysis is novel in its focus on minimal problems,  condition number and image data. 
Our work also suggests a way to accelerate RANSAC, by only testing hypotheses which come from sufficiently well-conditioned image \nolinebreak data.

The rest of the paper is organized as follows. Section~\ref{sec:framework} introduces our general theoretical framework for analyzing the conditioning of an arbitrary minimal problem. Section~\ref{sec:examples} recalls the two classic problems for estimating the relative pose, the 5-point problem for calibrated cameras and 7-point problem for the uncalibrated case. Section~\ref{sec:main-results} presents the results of our analysis for relative pose estimation, characterizing ill-posed world scenes and image data as well as proposing a potential way for testing for  well-conditioned image data. Then Section \ref{sec:experimentals} shows experimental results on synthetic and real data, as a proof-of-concept for our  theory. 

\section{Theoretical Framework} \label{sec:framework}

In this section we present a theoretical framework for analyzing the numerical stability of minimal problems in multiview geometry.  
The most relevant mathematical structure are Riemannian manifolds, which we use to describe the totality of world scenes, image data and epipolar quantities to be estimated.  Riemannian geometry is appropriate because it allows us to discuss intrinsic distances.  We will use tangent spaces, differentials and the inverse function theorem.

The framework builds on the theory of condition number and ill-posed inputs initiated by Demmel  \cite{demmel1987condition} and extended by Burgisser  \cite{burgisser2013condition}. However we have tailored it to the setting of minimal problems, where there are world scenes in-between the input image data and output epipolar \nolinebreak quantity.

\subsection{Spaces and Maps} \label{sec:spaces}

Let $\mathcal{W}, \mathcal{X}$, $\mathcal{Y}$ be Riemannian manifolds, with geodesic distances $d_{\mathcal{W}}(\cdot, \cdot)$, $d_{\mathcal{X}}(\cdot, \cdot)$, $d_{\mathcal{Y}}(\cdot, \cdot)$, tangent spaces denoted by $T(\mathcal{W}, w)$, $T(\mathcal{X}, x)$, $T(\mathcal{Y}, y)$ for points $w \in \mathcal{W}$, $x \in \mathcal{X}$, $y \in \mathcal{Y}$, and inner products on said tangent spaces denoted by $\left\langle \cdot, \cdot \right\rangle_{\mathcal{W},w}$, $\left\langle \cdot, \cdot \right\rangle_{\mathcal{X},x}$,  $\left\langle \cdot, \cdot \right\rangle_{\mathcal{Y},y}$. 
In anticipation of our upcoming applications to multiview geometry, we shall call
\begin{itemize}
    \item $\mathcal{W}$ the \textbf{world scene space};
    \item $\mathcal{X}$ the \textbf{image data space}; 
    \item $\mathcal{Y}$ the \textbf{epipolar space}.
\end{itemize}
To best model minimal problems in multiview geometry, we restrict to the case $\dim(\mathcal{W}) = \dim(\mathcal{X})$ (see Remark~\ref{rem:minimal}).

Assume we are given 
 a differentiable map $\Phi$ from world scenes to image data whose domain is an open dense subset of $\mathcal{W}$.  We indicate the situation using a dashed right \nolinebreak arrow:
\begin{equation}
    \Phi : \mathcal{W} \dashrightarrow \mathcal{X}.
\end{equation}
We assume that the image  $\Phi(\operatorname{Dom}(\Phi))$ contains an open dense subset of the codomain $\mathcal{X}$, and  summarize this property by calling $\Phi$ dominant.  We call $\Phi$ the \textbf{forward map}.

Furthermore, we assume that we are provided a differentiable map from world scenes to epipolar matrices, again defined only on an open dense subset of $\mathcal{W}$:
\begin{equation}
    \Psi : \mathcal{W} \dashrightarrow \mathcal{Y}.
\end{equation}
Again, assume $\Psi$ is dominant.  We call  $\Psi$ the \textbf{epipolar map}.

Given image data $x \in \mathcal{X}$, we call a function $\Theta : \mathcal{X} \supseteq \operatorname{Dom}(\Theta) \rightarrow \mathcal{W}$ a \textbf{3D reconstruction map} locally defined around $x$ if  $\operatorname{Dom}(\Theta)$ is an open neighborhood of $x$ in $\mathcal{X}$ and $\Theta$ is a section of the forward map, that is:
\begin{equation}
   \Phi \circ \Theta = \operatorname{id}_{\operatorname{Dom}(\Theta)}.
\end{equation}
In this case, composing $\Theta$ with the epipolar map gives a (locally defined) map from image data to the epipolar space:
\begin{equation} \label{eq:solution-map}
  \mathbf{S} := \Psi \circ \Theta : \mathcal{X}
    \supseteq \operatorname{Dom}(\Theta) \rightarrow \mathcal{Y}.
\end{equation}
We call $\mathbf{S}$ a \textbf{solution map} (locally defined around $x$).  The name is justified because in minimal problems in multiview geometry the quantity of interest we want to compute is typically an epipolar matrix/tensor, and the input is image data.  

\begin{remark} \label{rem:minimal}
Minimal problems in multiview geometry are modeled as follows: given image data $x \in \mathcal{X}$, we want to compute all  compatible real epipolar matrices/tensors, \ie
\begin{equation}
    \Psi(\Phi^{-1}(x)) = \{ \Psi(w) : w \in \mathcal{W}, \Phi(w) = x\} \subseteq \mathcal{Y}.
\end{equation}
These become hypotheses in RANSAC.
By calling the problem ``minimal", we mean that for $x$ in an open dense subset of $\mathcal{X}$, the output $\Psi(\Phi^{-1}(x))$ is a finite set and not always empty.
In vision problems, minimality is a consequence of additional structure (not required to discuss stability): $\mathcal{W}, \mathcal{X}, \mathcal{Y}$ typically also can be viewed as quasi-projective algebraic varieties \cite{harris2013algebraic} and $\Phi$, $\Psi$ are given by algebraic functions.  With $\dim(\mathcal{W}) = \dim(\mathcal{X})$ and the dominance of $\Phi$, this implies that generic fibers of $\Phi$ are finite sets and the problem is minimal.  For example, see \cite[Def.~2]{duff2020pl}.
\end{remark}

Our goal is to analyze the sensitivity of solution maps $\mathbf{S}$ for minimal problems to noise in the input $x$.  
We shall give a quantitative condition number formula.  We will also describe the locus of ill-posed inputs, where a solution map may not even exist locally or has infinite condition number.  

\subsection{Ill-Posed Locus}
Given image data $x \in \mathcal{X}$ and a prescribed world scene $w \in \mathcal{W}$ such that $\Phi(w) = x$, the next lemma shows there exists a unique continuous 3D reconstruction map $\Theta$ with $\Theta(x) = w$. Further, $\Theta$ is continuously differentiable ($C^1$).

\begin{lemma}  \label{lem:first}
Assume that the forward map $\Phi$ is $C^1$, and that at the world scene $w \in \mathcal{W}$ the forward map differentiates to an isomorphism on tangent spaces.  That is, the differential 
\begin{equation}
    D\Phi(w): T(\mathcal{W}, w) \rightarrow T(\mathcal{X}, \Phi(w))
\end{equation}
is a linear isomorphism.  Then there exist open neighborhoods $\mathcal{U}$ of $w$ in $\mathcal{W}$ and $\mathcal{V}$ of $\Phi(w)$ in $\mathcal{X}$ such that 
$\Phi : \mathcal{U} \rightarrow \mathcal{V}$ 
is bijection, the inverse function is $C^1$, and
\begin{equation} 
    D\left( (\Phi|_{\mathcal{U}})^{-1} \right)(\Phi(w)) = \left( D \Phi(w) \right)^{-1}.
\end{equation}
\end{lemma} \label{lem:first-lemma}

\noindent The lemma follows from the inverse function theorem for manifolds \cite{lee2013smooth}.
In words: if the forward Jacobian 
$D\Phi(w)$ is invertible, then the forward map $\Phi$ is locally invertible and its local inverse is differentiable with Jacobian $(D\Phi(w))^{-1}$.

We now come to a central concept in our framework:
\begin{definition} \label{def:ill-posed}
We say that a world scene $w \in \mathcal{W}$ is \textup{\textbf{ill-posed}} if the differential $D\Phi(w)$ is not invertible.  
We say that image data $x \in \mathcal{X}$ is \textup{\textbf{ill-posed}} if there exists a world scene $w \in \Phi^{-1}(x)$ such that $w$ is ill-posed.
\end{definition}
Ill-posed world scenes are those failing the condition in the above lemma; therefore, \textup{a priori} we do not know if the forward map is locally invertible around ill-posed world scenes.  Meanwhile ill-posed image data are those such that there is at least compatible world scene that is ill-posed; hence  there could be problematic behavior around an ill-posed world scene (We emphasize that other world scenes in $\Phi^{-1}(x)$ need not be ill-posed).
In a moment, we will see that all of the numerical instabilities in minimal problems must occur at (or near) the ill-posed scenes and image data.

\subsection{Condition Number}

Our other central theoretical concept is the condition number. 
We first explain this quite generally (and intuitively), following  \cite[Ch.~14]{burgisser2013condition}.
To this end, let $G : \mathcal{X} \supseteq \operatorname{Dom}(G) \rightarrow \mathcal{Y}$ be any map defined on an open neighborhood of $x$ in $\mathcal{X}$.

\begin{definition}
The \textup{\textbf{condition number}} of $G$ at $x$ is defined by
\begin{equation}
    \operatorname{cond}(G, x) \, := \, \, \lim_{\delta \rightarrow 0^{+}} \, \sup_{\substack{\widetilde{x} \in \mathcal{X}\\ d_{\mathcal{X}}(\widetilde{x}, x) < \delta}} \frac{d_{\mathcal{Y}}\left(G(\widetilde{x}), G(x)\right)}{d_{\mathcal{X}}\left(\widetilde{x}, x\right)}.
\end{equation}
\end{definition}
In a slogan: the condition number captures the limiting \textit{worst-case amplification} of input error in $x$ that the function $G$ can produce in its output $G(x)$, when distances are measured according to the intrinsic metrics on \nolinebreak $\mathcal{X}$ and \nolinebreak $\mathcal{Y}$.

If $G$ is differentiable, we have a more explicit formula.

\begin{lemma} \label{lem:cond-num}
If $G$ is differentiable then the condition number of $G$ at $x$ equals the operator norm of the differential $DG(x): T(\mathcal{X}, x) \rightarrow T(\mathcal{Y}, y)$, \ie
\begin{equation}
    \operatorname{cond}(G,x) = \max_{\substack{\dot{x} \in T(\mathcal{X}, x)\\ \| \dot{x} \| = 1 }} \| DG(x)(\dot{x}) \| =: \|DG(x)\|, 
\end{equation}
where the two norms in the middle quantity are induced by the Riemannian inner products $\langle \cdot, \cdot \rangle_{\mathcal{X},x}$ and $\langle \cdot, \cdot \rangle_{\mathcal{Y}, G(x)}$.
\end{lemma}
This is \cite[Prop.~14.1]{burgisser2013condition}, and proven using Taylor's theorem.
The lemma reduces computing the condition number of a differentiable map to computing the leading singular value of its Jacobian matrix written with respect to orthonormal bases on the tangent spaces $T(\mathcal{X},x)$ and $T(\mathcal{Y},G(x))$.

Here we are most interested in the condition number of solution maps for minimal problems as in Eq.~\eqref{eq:solution-map}.  Putting the previous two lemmas together with the chain rule  gives:
\begin{lemma} \label{lem:cond-of-S}
Let $\mathbf{S} = \Psi \circ \Theta : \mathcal{X} \supseteq \operatorname{Dom}(\Theta) \rightarrow \mathcal{Y}$ be a solution map as in  \eqref{eq:solution-map} defined around the image data $x \in \mathcal{X}$.  Let $w = \Theta(x) \in \mathcal{W}$ be the corresponding world scene.  Assume that $w$ is not ill-posed, \ie $D\Phi(w)$ is invertible.  Then, the condition number of $\mathbf{S}$ at $x$ is finite and given by
\begin{equation} \label{eq:cond-formula}
    \operatorname{cond}(\mathbf{S},x) = \| D\Psi(w) \circ D\Phi(w)^{-1} \|.
\end{equation}
In particular $\operatorname{cond}(\mathbf{S},x)$ can be infinite only if $x$ is ill-posed.
\end{lemma}

\subsection{Relation Between Ill-Posed Loci and Condition Number}

As shown in Lemma~\ref{lem:cond-of-S}, the condition number at $x \in \mathcal{X}$ of a varying epipolar matrix/tensor can be infinite only if $x$ is ill-posed as in Definition~\ref{def:ill-posed}.
If $x$ is ill-posed, the corresponding world scene $w = \Theta(x) \in \mathcal{W}$ such that $D\Phi(w)$ is rank-deficient might suffer unboundedly large relative changes as   $x$ changes.
Further, Lemma~\ref{lem:first} implies that the number of real 3D reconstructions is locally constant for inputs $x \in \mathcal{X}$ which are not ill-posed.
In other words, there can only be a change in the number of real epipolar matrices/tensors when the image data $x$ crosses over the ill-posed locus. 
Thus, the ill-posed locus captures the ``danger zone" where at least one of the solutions to the minimal problem can be unboundedly unstable, and also where real solutions can disappear into (or reappear from) the complex numbers.

In \cite{demmel1987condition}, Demmel proved that in some cases,  
the reciprocal of the distance to the ill-posed locus equals the condition number.
For example, this was shown for the problem of matrix inversion. 
Here, we do not prove a quantitative relationship between the distance to the ill-posed locus and the condition number for solving minimal problems in computer vision as such.  But we do numerically demonstrate a close relationship in the case of essential and fundamental matrix estimation in the experiments below, see Section~\ref{sec:experimentals}.

\section{Main Examples} \label{sec:examples}

We will apply our framework to study the sensitivity of minimal problems for one of the most popular tasks in multiview geometry: relative pose estimation.
In this section, we simply recall the relevant setup for the 5-point and 7-point minimal problems, by defining what the various spaces and maps are, from Section~\ref{sec:spaces}, for these cases.

\subsection{Essential Matrices and 5-Point Problem} \label{example:essential}

Here the world scenes consist of the relative pose between two calibrated pinhole cameras together with five world points, \ie
\begin{equation}
    \mathcal{W} = \operatorname{SO}(3) \times \mathbb{S}^2 \times (\mathbb{R}^3)^{\times 5} =  \{{\large{(}}R,t,X_1,\ldots, X_5{\large{)}}\},
\end{equation}
where $\operatorname{SO}(3) = \{R \in \mathbb{R}^{3 \times 3} : RR^{\top} = R^{\top}R = I\}$ is the special orthgonal group (representing the  orientation in the relative pose) and $\mathbb{S}^2 = \{ t \in \mathbb{R}^3 : \| t \|_2 = 1 \}$ is the unit sphere (representing the direction of the translation in the relative pose).
Meanwhile, the image data space consists of five image point pairs:
\begin{equation}
   \mathcal{X} = \left( \mathbb{R}^2 \times \mathbb{R}^2 \right)^{\times 5} = \{{\Large{(}}(x_1, y_1), \ldots, (x_5, y_5){\Large{)}}\}.
\end{equation}
The forward map projects the given world points via the calibrated cameras $[I \,\, 0] \in \mathbb{R}^{3 \times 4}$ and $[R \,\, t] \in \mathbb{R}^{3 \times 4}$, \ie
{\footnotesize
\begin{multline} \label{eq:E-forward}
\Phi\left(R,t,X_1,\ldots, X_5 \right) = {\Large{(}} (\beta(X_1), \beta(RX_1 + t)),  \ldots, \\   (\beta(X_5), \beta(RX_5 + t)) {\Large{)}},
\end{multline}
}
where $\beta: \mathbb{R}^3 \dashrightarrow \mathbb{R}^2$ is the quotient map 
\begin{align} \label{eq:quotient}
\beta(a_1,a_2,a_3) := (a_1/a_3, a_2/a_3)
\end{align}
defined whenever $a_3 \neq 0$.
The epipolar space consists of the manifold of real essential matrices, characterized by ten cubic equations vanishing \cite{demazure1988deux} or using singular \nolinebreak values:
{\footnotesize
\begin{align} \label{eq:essential-ideal}
     \mathcal{Y} &= \{E \in \mathbb{P}(\mathbb{R}^{3 \times 3}): 2EE^{\top}\!E - \operatorname{tr}(EE^{\top}\!)E  = 0, \nonumber  \det(E)=0 \} \nonumber \\
    &= \{E \in \mathbb{P}(\mathbb{R}^{3 \times 3}) : \sigma_1(E) = \sigma_2(E) > \sigma_3(E) = 0\}.
\end{align}
}
Lastly, the epipolar map sends a world scene to the essential matrix associated to the scene's relative pose:
\begin{equation}
    \Psi\left(  R, t, X_1, \ldots, X_5 \right) = R [t]_{\times} \in \mathbb{P}(\mathbb{R}^{3 \times 3}),
\end{equation}
where $[t]_{\times} \in \mathbb{R}^{3 \times 3}$
is the usual skew matrix representation of cross product with $t$ as in \cite[Sec.~9.6]{hartleyzisserman}.

Given five point pairs $x \in \mathcal{X}$, there are at most $40$ compatible world scenes in $\mathcal{W}$ (coming in twisted pairs \cite[Sec.~9.6.3]{hartleyzisserman}) which map to at most $10$ compatible essential matrices in $\mathcal{Y}$.  Nister gave a solver, boiling down to computing real roots of a degree $10$ univariate polynomial, in \cite{nister:PAMI:2004}.

\subsection{Fundamental Matrices and 7-Point Problem} \label{example:fundamental}

Here the world scenes consist of the relative pose of two uncalibrated pinhole cameras together with seven world points. 
Here it is less immediate than in the calibrated case (Example~\ref{example:essential}) how we should represent the relative pose in an (almost everywhere) one-to-one way using a minimal number of parameters.
However, it turns out that our main results are independent of the coordinate system choice we make for $\mathcal{W}$, thus we will represent the relative pose by the open dense subset of $\mathbb{R}^7 = \{b = (b_1, \ldots, b_7)\}$ where 
\begin{equation} \label{eq:Mb}
  M(b) :=   \begin{pmatrix} 
    1 & b_1 & b_2 & b_3 \\
    b_4 & b_5 & b_6 & b_7 \\
    0 & 0 & 0 & 1
    \end{pmatrix}
\end{equation}
has rank $3$.  
In the supplementary materials, we prove that this set gives a normal form for almost all uncalibrated relative poses, \ie for an open dense subset of pairs of uncalibrated camera matrices we can uniquely bring the bring to the form $[I \,\,  0] \in \mathbb{R}^{3 \times 4}$ and $M(b) \in \mathbb{R}^{3 \times 4}$ by multiplying the pair on right by an appropriate projective world transformation in $\operatorname{PGL}(4) = \{g \in \mathbb{P}(\mathbb{R}^{4 \times 4}): \det(g) \neq 0\}$.
Thus, 
\begin{equation}
    \mathcal{W} = \mathbb{R}^7 \times (\mathbb{R}^3)^{\times 7} = \{(b, X_1, \ldots, X_7)\}.
\end{equation}
The forward map projects the given world points via the uncalibrated cameras $[I \,\, 0]$ and $M(b)$, \ie
{\footnotesize
\begin{multline} \label{eq:F-forward}
\Phi(b, X_1, \ldots, X_7) = \\ ((\beta(X_1), \beta(M(b)\begin{pmatrix} X_1 \\ 1 \end{pmatrix})), \ldots, (\beta(X_7), \beta(M(b) \begin{pmatrix} X_7 \\ 1 \end{pmatrix}))),
\end{multline}
}
\hspace{-0.5em} where $\beta : \mathbb{R}^3 \dashrightarrow \mathbb{R}^2$ is the quotient map in \eqref{eq:quotient}.
The epipolar space consists of the manifold of real fundamental matrices, which is the same as rank-two $3 \times 3$ real matrices defined up to nonzero scale:
\begin{equation}
    \mathcal{Y} = \{ F \in \mathbb{P}(\mathbb{R}^{3 \times 3} ) : \operatorname{rank}(F) = 2\}.
\end{equation}
The epipolar map sends a world scene to the fundamental matrix associated to the scene's relative pose \cite[Eq.~17.3]{hartleyzisserman}:
\begin{align} \label{eq:F-def}
    & \Psi(b, X_1, \ldots, X_7) = F \in \mathbb{P}(\mathbb{R}^{3 \times 3}), \textit{ where }  \\
    & F_{ji} := (-1)^{i+j} \det \! \begin{pmatrix} [I \,\, 0] \textup{ with row } i \textup{ omitted} \\
M(b) \textup{ with row } j \textup{ omitted} 
\end{pmatrix}. \nonumber
\end{align}

Given seven point pairs $x \in \mathcal{X}$, there are most $3$ compatible world scenes $w \in \mathcal{W}$ mapping to at most $3$ compatible fundamental matrices in $\mathcal{Y}$. These are found by computing real roots of a cubic univariate polynomial \cite[Sec.~11.1.2]{hartleyzisserman}.

\section{Main Results} \label{sec:main-results}
We now present our main theoretical results regarding the instabilities of relative pose estimation, by applying the framework in Section~\ref{sec:framework} to the minimal problems in Section~\ref{sec:examples}.
Due to space limitations, all proofs (and certain explicit formulas) will appear in the supplementary materials.

\subsection{Condition Number Formulas}\label{sec:conditionNumsFormulas}

Here we apply the formula \eqref{eq:cond-formula} based on singular values of the Jacobian matrix to the 5-point and 7-point problems.  
This yields condition number formulas for essential and fundamental estimation.  The expressions are valid if the solution maps passes through non-ill-posed world scenes; in fact they only depend on said world scene.
We display the explicit Jacobian matrices in the supplementary materials.

\begin{proposition}[Condition number for $E$] \label{prop:formula-E}
Consider the 5-point problem in Section~\ref{example:essential}.  Let $x \in (\mathbb{R}^2 \times \mathbb{R}^2)^{\times 5}$ be given image data, and $w \in \operatorname{SO}(3) \times \mathbb{S}^2 \times (\mathbb{R}^3)^{\times 5}$ a compatible world scene which is not ill-posed.  Then there exists a unique continuous 3D reconstruction map $\Theta$ locally defined around $x$ such that $\Theta(x) = w$, and an associated uniquely defined solution map $\mathbf{S} = \Psi \circ \Theta$ from image data to essential matrices. 
The condition number of $\mathbf{S}$ can be computed as the largest singular value of an explicit $5 \times 20$ matrix whose entries are functions of $w$.  This  matrix naturally factors as a $5 \times 20$ matrix multiplied by a $20 \times 20$ matrix.  
\end{proposition}

\begin{proposition}[Condition number for $F$] \label{prop:formula-F}
Consider the 7-point problem in Section~\ref{example:fundamental}.  Let $x \in (\mathbb{R}^2 \times \mathbb{R}^2)^{\times 7}$ be given image data, and $w \in \mathbb{R}^7 \times (\mathbb{R}^3)^{\times 5}$ a compatible world scene which is not ill-posed.  Then there exists a unique continuous 3D reconstruction map $\Theta$ locally defined around $x$ such that $\Theta(x) = w$, and an associated uniquely defined solution map $\mathbf{S} = \Psi \circ \Theta$ from image data to essential matrices. 
The condition number of $\mathbf{S}$ can be computed as the largest singular value of an explicit $7 \times 28$ matrix whose entries are functions of $w$.  This  matrix naturally factors as a $7 \times 28$ matrix multiplied by the inverse of a  $28 \times 28$ matrix.  
\end{proposition}

\subsection{Ill-Posed World Scenes}\label{sec:ill-pose-world}

Here we derive geometric conditions for a world scene to be ill-posed for the 5-point or 7-point problem.  
Our characterizations are in terms of the existence of a suitable quadric surface in $\mathbb{R}^3$, which should satisfy certain properties related to the given world scene.  Recall that a quadric surface $\mathcal{Q}$ in $\mathbb{R}^3$ is specified by the vanishing of a quadratic polynomial on the coordinates of $\mathbb{R}^3$.  If $\mathbb{R}^3 = \{a=(a_1, a_2, a_3)\}$ then there is a symmetric matrix $Q \in \mathbb{R}^{4 \times 4}$ such that
\begin{equation} \label{eq:def-quad}
    \mathcal{Q} = \left\{ (a_1, a_2, a_3) : \begin{pmatrix} a & 1 \end{pmatrix} Q \begin{pmatrix} a \\ 1 \end{pmatrix} = 0 \right\} \subseteq \mathbb{R}^3.
\end{equation}

We remark that the appearance of quadric surfaces is not new when it comes to degeneracies for epipolar matrices \cite{luong1996fundamental, kahl2002critical, hartley2007critical, bertolini2019critical}.
However our results seem novel, and their proofs are independent of prior work.
The degeneracies considered here for minimal problems seem  not to have been studied.

\begin{theorem}[Ill-posed world scenes for $E$] \label{thm:illposed-world}
Consider the 5-point problem in Section~\ref{example:essential}.
Let $w = (R, t, X_1, \ldots, X_5) \in \operatorname{SO}(3) \times \mathbb{S}^2 \times (\mathbb{R}^3)^{\times 5}$ be a world scene such that $\Phi(w)$ exists where $\Phi$ is as in Eq.~\eqref{eq:E-forward}. Then $w$ is ill-posed, \ie  $D\Phi(w)$ is rank-deficient, if and only if there exists a quadric surface $\mathcal{Q} \subseteq \mathbb{R}^3$ such \nolinebreak that:
\begin{itemize}
    \item $\mathcal{Q}$ passes through the given world points $X_1, \ldots, X_5$;
    \item $\mathcal{Q}$ contains the baseline of the given relative pose; 
    \item and intersecting $\mathcal{Q}$ with any normal affine plane to $\ell$ produces a circle.
\end{itemize}
Here the baseline $\ell \subseteq \mathbb{R}^3$ is the world line passing through the two camera centers, \ie $\ell = \operatorname{Span}(-R^{\top}t)$.
\end{theorem}

The second requirement implies that $\mathcal{Q}$ is a ruled quadric surface (\ie, covered by two infinite families of lines).  Meanwhile, the third item is a non-standard condition implying that $\mathcal{Q}$ must be special within the set of ruled quadric surfaces, namely it muse be a so-called ``rectangular quadric"  \cite{maybank2012theory}.  See Figure~\ref{fig:surface} for visualizations of  \nolinebreak Theorem~\ref{thm:illposed-world}.

\begin{theorem}[Ill-posed world scenes for $F$] \label{thm:F-ill-posed-world}
Consider the 7-point problem in Section~\ref{example:fundamental}.  Let $w = (b, X_1, \ldots, X_7) \in \mathbb{R}^7 \times (\mathbb{R}^3)^{\times 7}$ be a world scene such that $\Phi(w)$ exists where  $\Phi$ is as in Eq.~\eqref{eq:F-forward}.
Then $w$ is ill-posed, \ie $D\Phi(w)$ is rank-deficient, if and only if there exists a quadric surface $\mathcal{Q} \subseteq \mathbb{R}^3$ such that:
\begin{itemize}
    \item $\mathcal{Q}$ passes through the given world points $X_1, \ldots, X_7$;
    \item and $\mathcal{Q}$ contains the baseline of the given relative pose.
\end{itemize}
Here the baseline $\ell$ is the world line passing through the two camera centers, \ie $\ell$ is spanned by the vector $v \in \mathbb{R}^3$ such that $M(b) \begin{pmatrix} v \\ 1 \end{pmatrix} = 0$ where $M(b) \in \mathbb{R}^{3 \times 4}$ is as in  Eq.~\eqref{eq:Mb}.
\end{theorem}

Now the conditions on the quadric surface are the same as in Theorem~\ref{thm:illposed-world}, except the third condition (stemming from calibration) is  absent.  Our proofs for Theorems~\ref{thm:illposed-world} and \ref{thm:F-ill-posed-world} both proceed by unwinding the requirement that there  exists a nonzero kernel vector for the forward Jacobian matrix.

\subsection{Ill-Posed Image Data}
\label{sec:ImageData}
Here we describe the locus of ill-posed image data for the 5-point and 7-point problems. 
These results rely heavily on the polynomial structure present in both minimal problems (as mentioned in Remark~\ref{rem:minimal}).  
Specifically the proofs use known facts from algebraic geometry due to Sturmfels  \cite{sturmfels2017hurwitz}.

Compared to \cite{sturmfels2017hurwitz}, the main contribution of this subsection
is that we obtain viable computational schemes for actually visualizing the loci of ill-posed image data.  
For both the cases of fundamental and essential matrices, we give methods based on numerical homotopy continuation \cite{sommese2005numerical} to solve polynomial equations.  Implemented in the Julia package HomotopyContinuation.jl \cite{breiding2018homotopycontinuation}, these terminate on a desktop computer in $\approx \!10$ and $\approx \!30$ seconds, respectively.
Specifically for fundamental matrices, we also describe ill-posed image data symbolically, using Pl\"ucker coordinates, in a method that takes $\approx \!0.5$ seconds to run on a desktop computer. Details are given in the supplementary \nolinebreak materials.

\begin{theorem}[Ill-posed image data for $E$] \label{thm:image-E}
Consider the 5-point problem in Section~\ref{example:essential}.  Let $x = ((x_1, y_1), \ldots, (x_5,y_5)) \in (\mathbb{R}^2 \times \mathbb{R}^2)^{\times 5}$ be image data.
Then $x$ is ill-posed, \ie there exists some compatible world scene which is ill-posed, only if a certain polynomial $\mathbf{P}$ in the entries of $x_1, y_1, \ldots, y_5$ vanishes.  This polynomial has degree $30$ separately in each of the points $x_1, \ldots, y_5$. 
In particular, if we fix numerical values for $x_1, y_1, \ldots, x_5$ but keep $y_5 \in \mathbb{R}^2$ as variable, then (generically) $\mathbf{P}$ specializes to a degree $30$ polynomial just in $y_5$, and its vanishing set is a degree $30$ curve in the second image plane.  Moreover given the values for $x_1, y_1, \ldots, x_5$, we can compute an explicit plot of this curve in $\mathbb{R}^2$ by plotting the real roots of the curve intersected with various vertical lines swept across the second image plane.
\end{theorem}

We call the curve in Theorem~\ref{thm:image-E} a \textbf{$4.5$-point curve}, because it is specified by four-and-a-half image point pairs, namely $x_1, y_1, \ldots, x_5$.  See Figure~\ref{fig:F_curves} for sample renderings.

\begin{theorem}[Ill-posed image data for $F$] \label{thm:F-image}
Consider the 7-point problem in Section~\ref{example:essential}.  Let $x = ((x_1, y_1), \ldots, (x_7,y_7)) \in (\mathbb{R}^2 \times \mathbb{R}^2)^{\times 7}$ be image data.
Then $x$ is ill-posed, \ie there exists some compatible world scene which is ill-posed, only if a certain polynomial $\mathbf{P}$ in the entries of $x_1, y_1, \ldots, y_7$ vanishes.  This polynomial has degree $6$ separately in each of the points $x_1, \ldots, y_7$.  
In particular, if we fix numerical values for $x_1, y_1, \ldots, x_7$ but keep $y_7 \in \mathbb{R}^2$ as variable, then (generically) $\mathbf{P}$ specializes to a degree $6$ polynomial just in $y_7$, and its vanishing set is a degree $6$ curve in the second image plane.  Moreover given the values for $x_1, y_1, \ldots, x_7$, we can compute an explicit plot of this curve in $\mathbb{R}^2$ by plotting the real roots of the curve intersected with various vertical lines swept across the second image plane.  Alternatively,  $\bf{P}$ can be expressed as an explicit degree $6$ polynomial in $\binom{9}{2}$ Pl\"ucker coordinates with $1668$ terms and integer coefficients, each at most $72$ in absolute value. We can specialize this expression by substituting in numerical values of $x_1, y_1, \ldots, x_7$, and then plot the zero set in $\mathbb{R}^2$.
\end{theorem}

We call the curve in Theorem~\ref{thm:F-image} a \textbf{$6.5$-point curve}, because it is specified by six-and-a-half image point pairs, namely $x_1, y_1, \ldots, x_7$.  See Figure~\ref{fig:F_curves} for sample renderings.

\begin{figure}
    \centering
    (a)\includegraphics[height=0.3\linewidth]{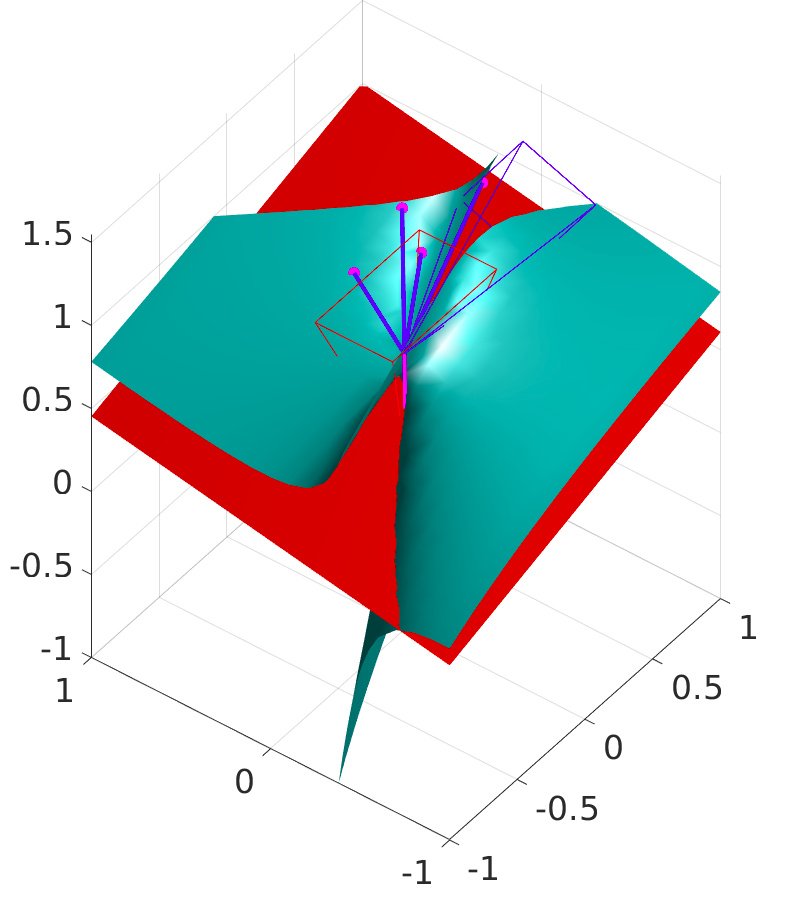}
    (b)\includegraphics[height=0.3\linewidth]{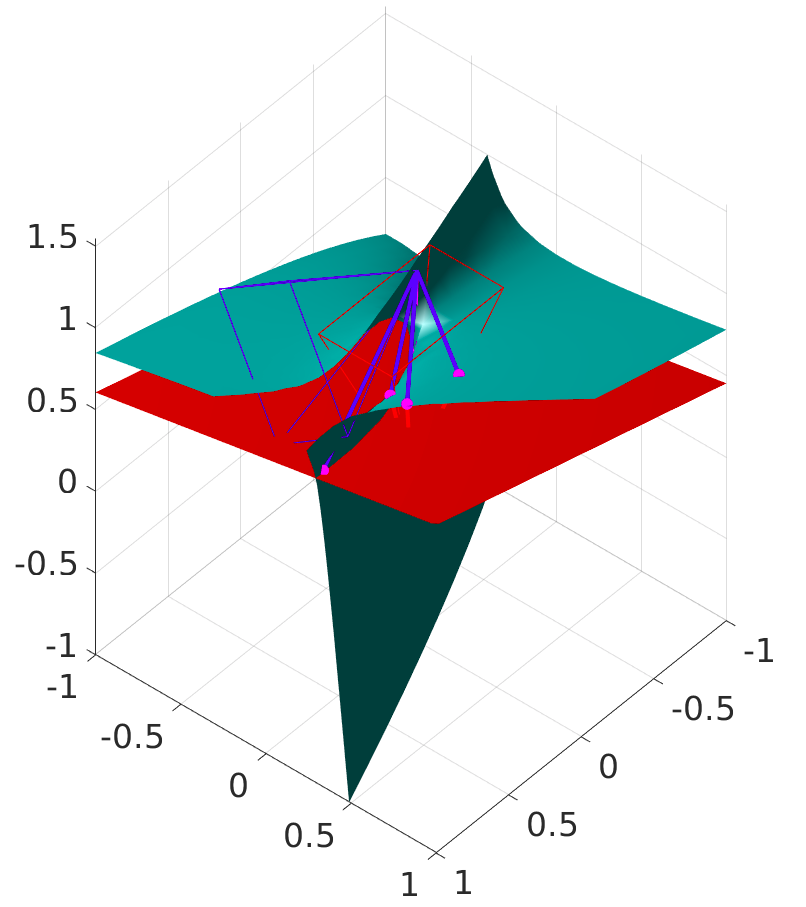}
    (c)\includegraphics[height=0.3\linewidth]{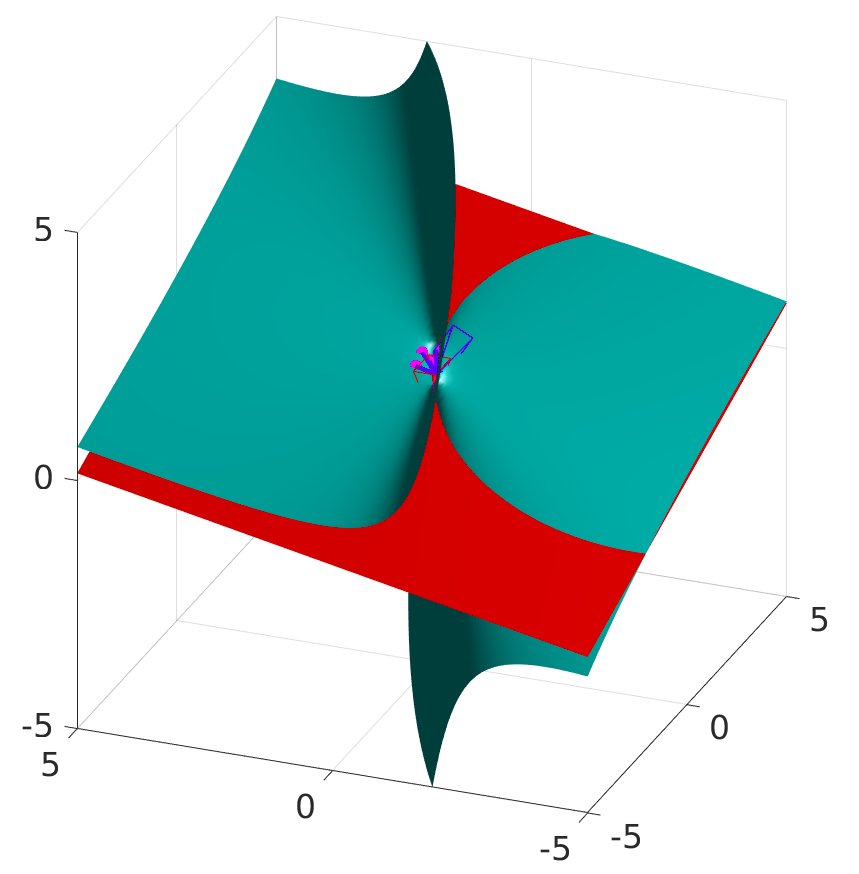}
    \caption{An illustrative  example of an ill-posed world scene in the calibrated case. Red and blue pyramid represents two cameras. Magenta points represent the given world  points. The green surface is a  quadric surface satisfying the three conditions in Theorem~\ref{thm:illposed-world}. Last, the red plane is a plane perpendicular to the baseline.}
    \label{fig:surface}
\end{figure}

\section{Experimental Results} \label{sec:experimentals}
\subsection{Synthetic Experiments}
\noindent \textbf{Data Generation:} To experiment with synthetic data, we generate random valid configurations consisting of the world scene $\large{(}R,t,X_1,\ldots, X_N \large{)}$, intrinsic matrix $K$ and 2D point pairs $(x_1, y_1,\cdots,x_N, y_N)$. Here $N=5$ or $N=7$ depending on whether the scenario is calibrated or uncalibrated.  We generate random problem instances as follows:
\begin{itemize}
\setlength\itemsep{0.5pt}
    \item $R$: orthogonal matrix in the QR decomposition of a random $3 \times 3$ matrix with i.i.d. standard normal entries;
    \item $t$: uniformly sampled vector from the unit sphere;
    \item $X_i$: uniformly sampled points with depth in [1, 20];
    \item $K$: chosen so that the image size is $640 \times 480$, focal length is $525$, and principle point is the image center;
    \item $x_i$, $y_i$: projections of $X_i$ onto two images.
\end{itemize}
We discard instances where any of the 2D points land outside the image's boundary.
Storing all these elements gives synthetic data for both the calibrated/uncalibrated cases.

\vspace{1em}

\noindent \textbf{Instability Revelation:} We first aim to demonstrate that the instabilities do empirically occur in the minimal problems for both calibrated and uncalibrated relative pose estimation. 
To this end, we generate 3000 synthetic minimal problems for both cases as described above.  
For each minimal problem instance, we add i.i.d. noise to the image points drawn from the spherical Gaussian $\mathcal{N}(0,\sigma^2 I_2)$ for different noise levels $\sigma$. 
Then we separately solve the original and perturbed problems, and compare them. 
We define an estimate to be erroneous if either of the following criteria holds: 
{\em (i)} \textbf{Large error in the solutions for the perturbed points}: 
the error in the fundamental or essential matrix after normalization is defined by $e = \operatorname{mean}(\operatorname{abs}(\operatorname{abs}(\bar{M} ./ M) - \mathbf{1}\mathbf{1}^{\!\top}))$.  Here ``$./$" denotes  element-wise division,  $M$ is the ground truth model, $\bar{M}$ is the nearest estimated model, and $\mathbf{1}\mathbf{1}^{\!\top}$ is the $3 \times 3$ matrix with each element $1$. 
Then {\em (i)} holds if $e$ exceeds a threshold \nolinebreak $\tau$.
{\em (ii)} \textbf{Change in the number of real solutions}: this behavior is troublesome because the true epipolar matrix can disappear into the complex plane if there is a variation in the number of real solutions.

Figure~\ref{fig:Revelation} shows the fraction of the erroneous estimations out of the 3000 instances at various noise levels and error thresholds. It is clear that for random perturbations, the ratio of erroneous cases cannot be ignored even 
when the noise is small.
In practice,  unstable instances would likely be weeded out by RANSAC. 
Indeed, Figure~\ref{fig:Revelation} suggests that even given \textit{all} inlier data, RANSAC is still needed to overcome the instabilities of relative pose estimation. 
The idea of the X.5-point curve is to identify the instances which may generate erroneous estimates. 
Also, note that the frequency of erroneous cases for essential matrices is much higher than for fundamental matrices. This makes sense because the 5-point problem solves a degree $10$ polynomial system, whereas the 7-point problem solves a degree $3$ \nolinebreak system.  

\begin{figure}[ht]
    \centering
    (a)\includegraphics[height = 0.45 \linewidth]{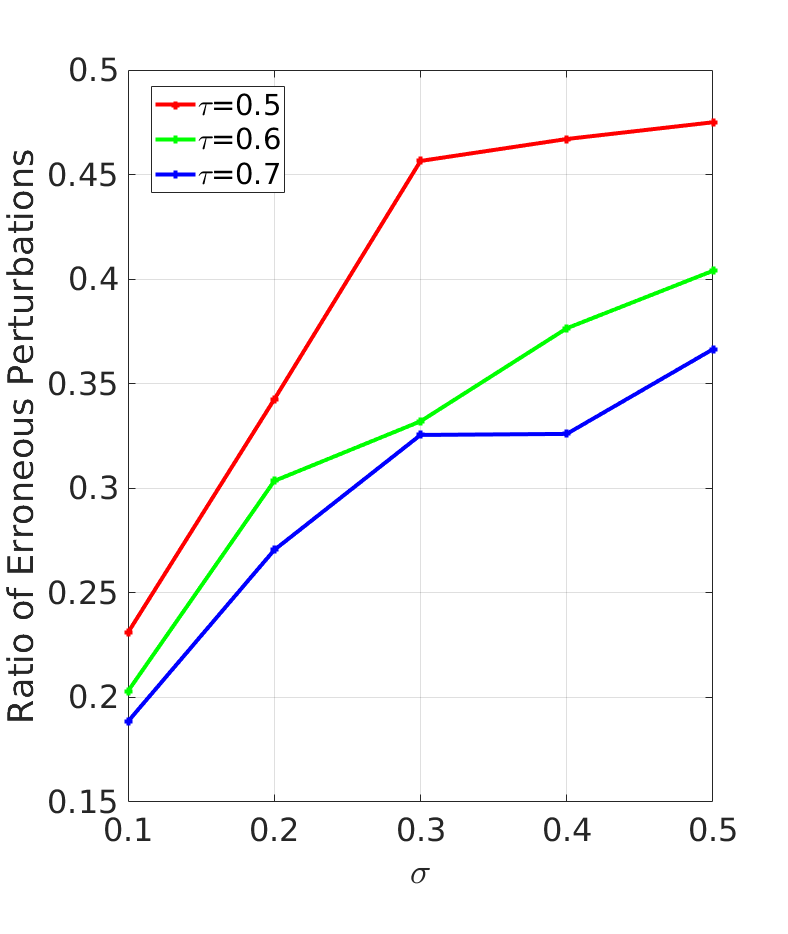}
    (b)\includegraphics[height = 0.45 \linewidth]{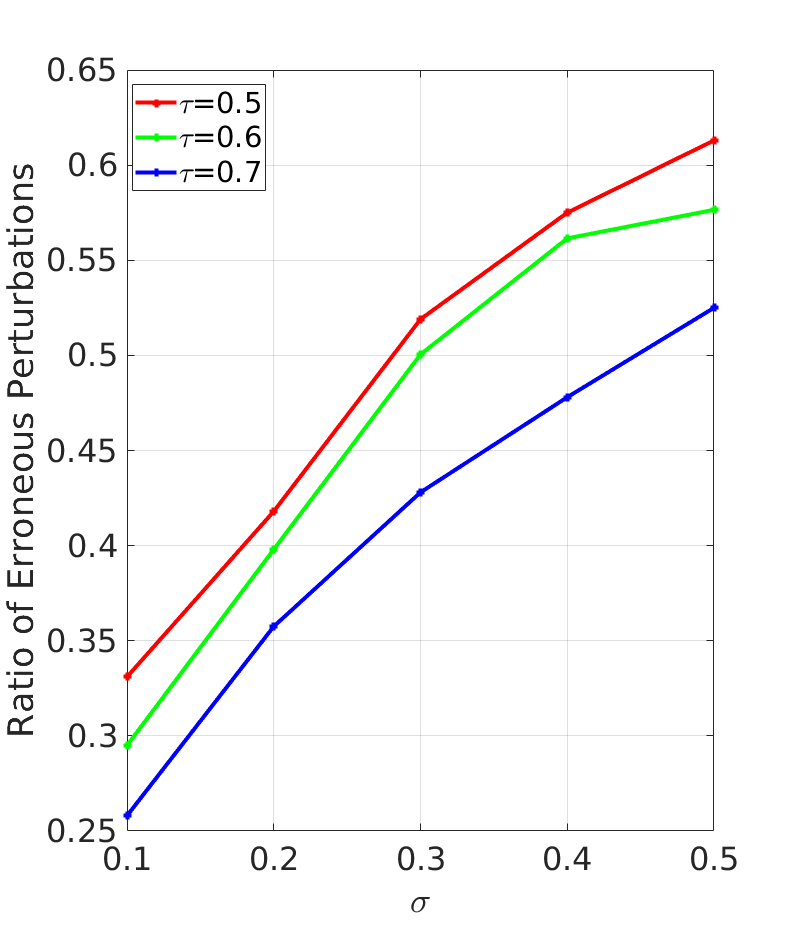}
    \caption{Ratio of 
    erroneous estimations
    out of 3000 random synthetic minimal problems at different noise levels $\sigma$ and  error thresholds $\tau$. (a)  \textup{Fundamental matrix}. (b)  \textup{Essential matrix}.}
    \label{fig:Revelation}
\end{figure}

\vspace{0.5em}

\noindent \textbf{Instability Detection:}
Applying the methods in Section~\ref{sec:ImageData}, 4.5-point degenerate curves for the uncalibrated case and 6.5-point  curves for the calibrated case can be computed for each minimal problem instance. Figure~\ref{fig:F_curves} shows several sample curves plotted on the second image plane along with the given image points. For uncalibrated case, the degree of the 6.5-point curve is 6, while the degree of the 4.5-point curve is 30 for calibrated case. The curves split the image plane into different connected components, wherein the number of real solutions is locally constant. 
In the language of~\cite{bernal2020machine}, the curves are
 ``real discriminant loci".
 
\begin{figure}[ht]
    \centering
    (a)\includegraphics[height=0.28\linewidth]{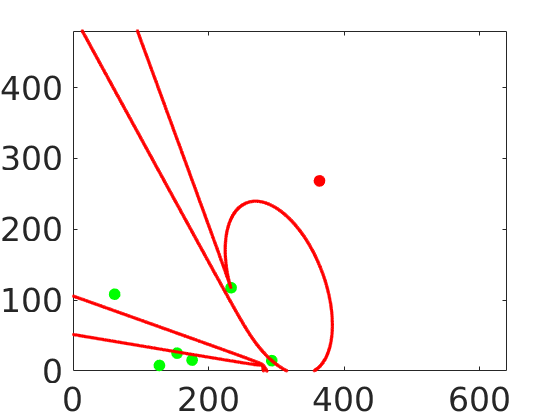}
    (b)\includegraphics[height=0.28\linewidth]{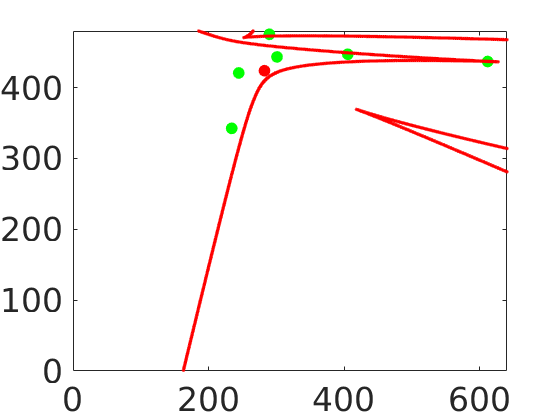}
        (c)\includegraphics[height=0.28\linewidth]{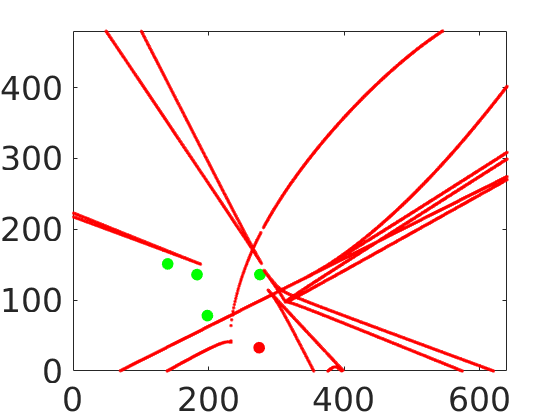}
    (d)\includegraphics[height=0.28\linewidth]{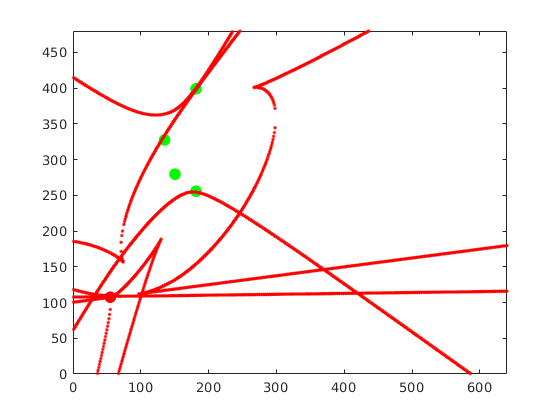}
    \caption{Sample results for the X.5-point degenerate curve in Theorems \ref{thm:image-E} and \ref{thm:F-image}. Points used in computing the curve are shown as green; the red point is the 5th/7th correspondence on the second image for calibrated/uncalibrated relative pose estimation respectively; the red curve is the X.5-point curve which we computed using homotopy continuation. (a) A stable configuration for  uncalibrated estimation. (b) An unstable configuration for uncalibrated estimation. (c) A stable configuration for calibrated case. (d) An unstable configuration for calibrated case.}  
    \label{fig:F_curves}
\end{figure}

\vspace{1em}

In another experiment, we separate the 3000 random synthetic minimal problems into three categories: stable cases, unstable cases, and the borderline cases (given that the condition number is a continuous indication of the stability). 
Here,  an instance is sorted according to the the number of erroneous estimates among $n=20$ perturbations, denoted by $\hat{n}$. If $\hat{n} \in [0, n/3]$, we count the instance as stable; if $\hat{n} \in [2n/3, n]$, we count the instance as unstable; and if $\hat{n} \in [n/3, 2n/3]$, we count the instance as borderline.
In this experiment, we use $\tau = 0.5$ and $\sigma = 0.3$.
\begin{figure}
    \centering
    (a)\includegraphics[width = 0.95 \linewidth]{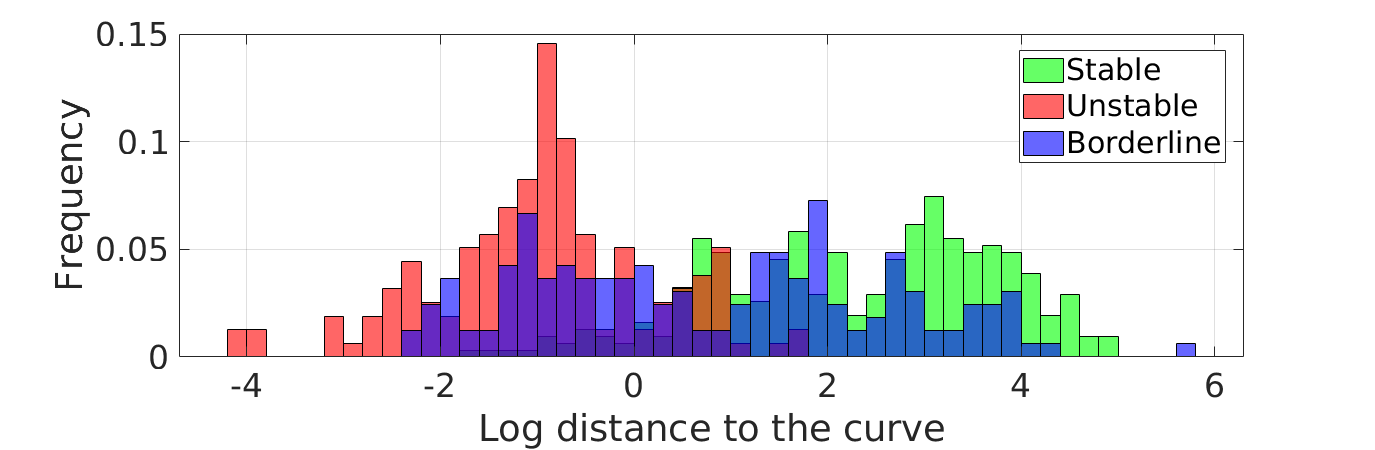}
    (b)\includegraphics[width = 0.95 \linewidth]{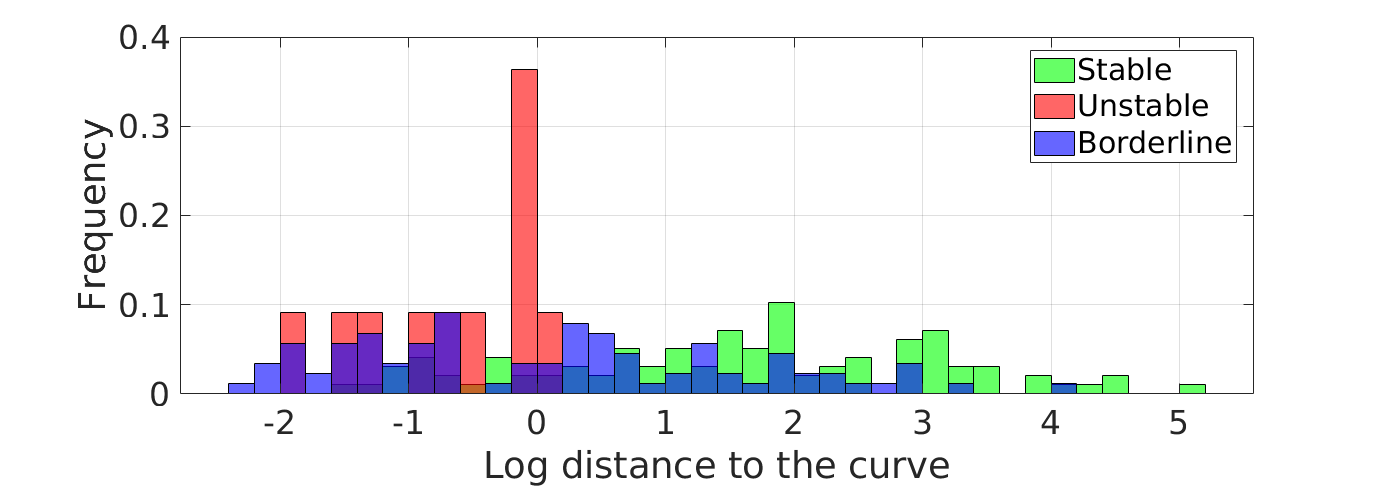}
    \caption{Histogram of distance from the last point to the degenerate curve sorted by: stable cases (green), unstable cases (red) and borderline cases (blue). (a) Uncalibrated estimation. (b) Calibrated estimation. Stable and unstable categories are separated.}
    \label{fig:statistics}
\end{figure}
For the uncalibrated case, the average distance from the 7th point to the 6.5-point curve is 2.35 pixels among unstable cases, while for the stable cases it is 22.12 pixels. For the calibrated case, the average distance from the 5th point to the 4.5-point curve is 14.95 pixels for unstable cases, while for the stable case it is is 0.32 pixels. From these statistical differences (see Figure~\ref{fig:statistics}), we observe that the stable and unstable categories can be distinguished by thresholding on the distance between the last point  to the X.5-point curve. 

\vspace{0.5em}

\noindent \textbf{Stability of Instability:}
Here we show that the degenerate curve is mostly stable to the presence of the noise so that our idea is not only theoretically correct, but can also be used in the practical setting of noisy images. Figures~\ref{fig:my_label} and~\ref{fig:stability_syn} shows that the computed curves are relatively stable to  the noise.

\begin{figure}
    \centering
    (a)\includegraphics[height = 0.21 \linewidth]{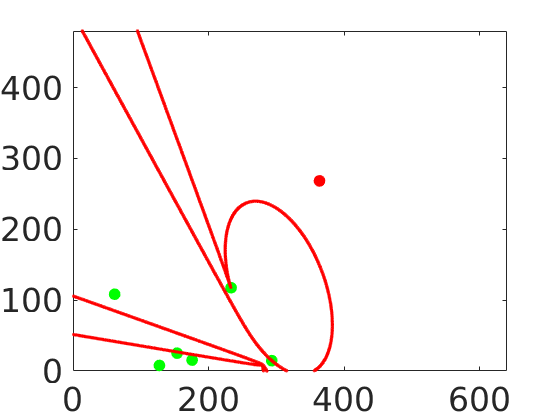} (b) \includegraphics[height = 0.21 \linewidth]{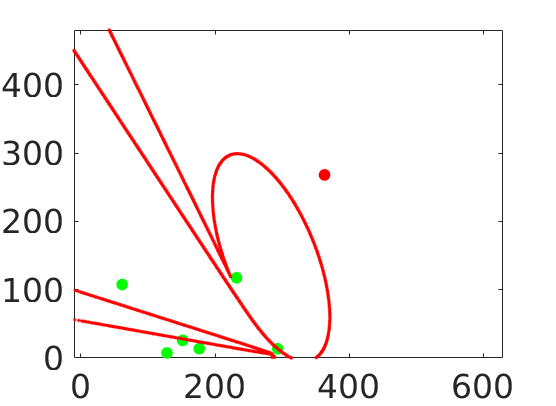} 
    \includegraphics[height = 0.21 \linewidth]{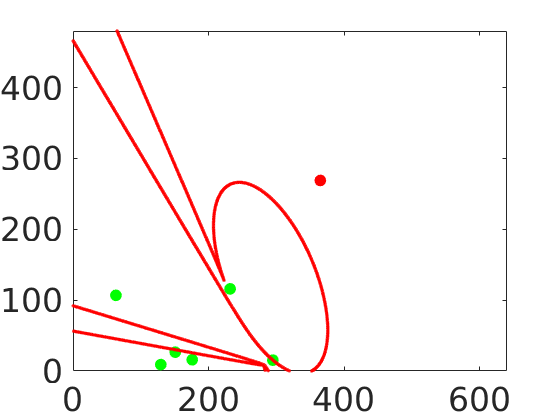}
    (c)\includegraphics[height = 0.21 \linewidth]{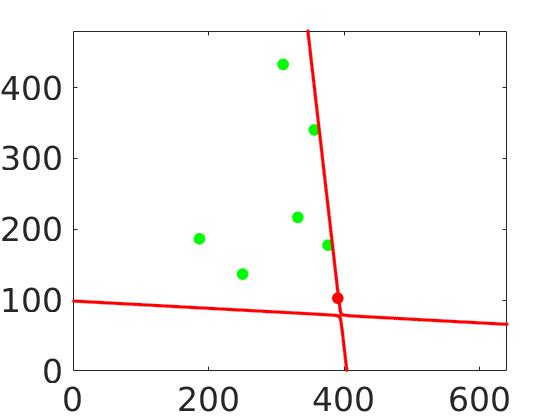} (d) \includegraphics[height = 0.21 \linewidth]{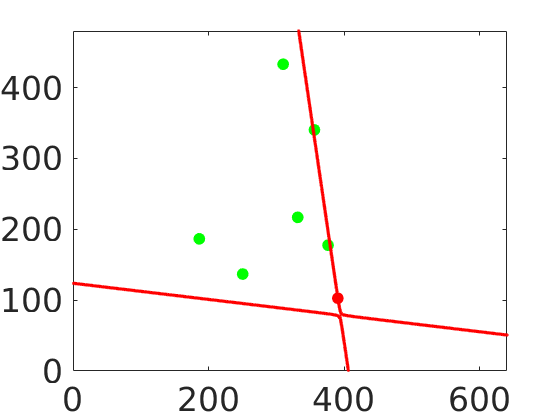} 
    \includegraphics[height = 0.21 \linewidth]{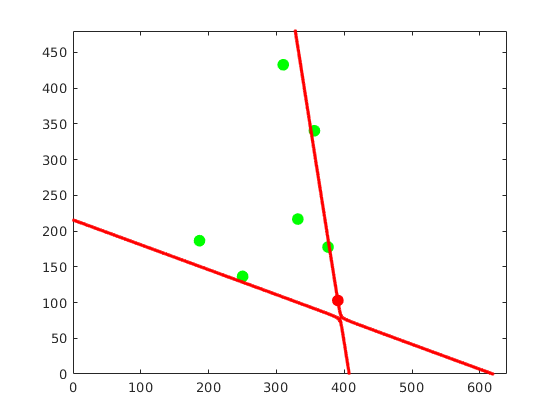}
    \caption{Illustrative result indicating the stability of the degenerate curve. 
    (a) The degenerate curve of a stable uncalibrated configuration. (c) The degenerate curve of an unstable uncalibrated configuration. (b) (d) Adding different noise perturbations to the image points, the curve did not change much. 
    Curves for calibrated estimation are shown in the supplementary materials.}
    \label{fig:my_label}
\end{figure}

\begin{figure}
    \centering
    \includegraphics[width=0.91 \linewidth]{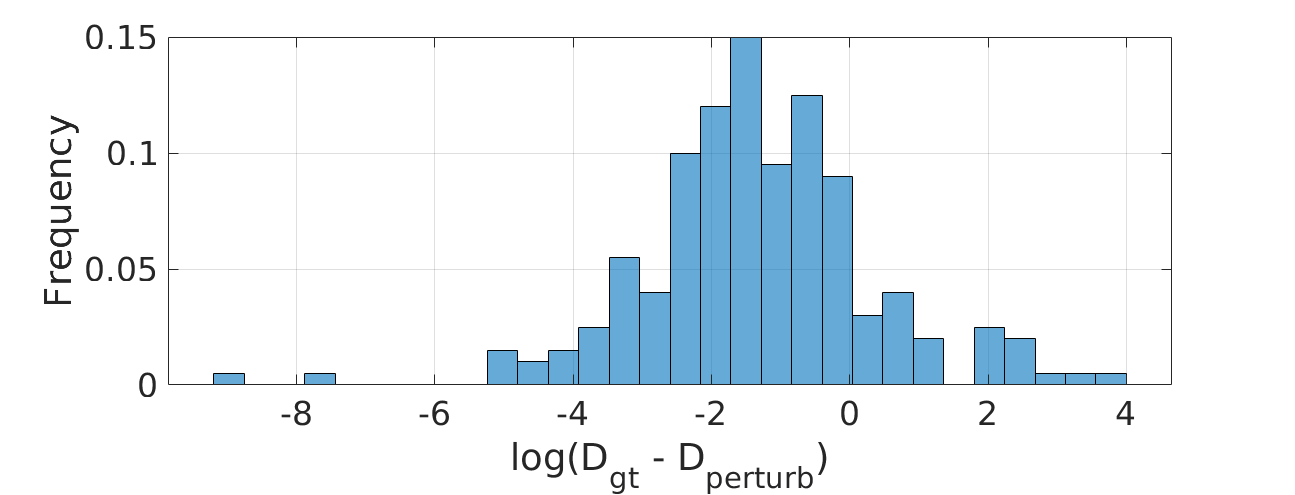}
    \caption{Histogram of the difference between two distances:  from the target point to the X.5-point curve without perturbation; and from the target point to the X.5-point curve with perturbation on the other points. The perturbation does not drastically change the distance, showing that the X.5-point curve is relatively stable to noise.  This plot combines the calibrated and uncalibrated cases.}
    \label{fig:stability_syn}
\end{figure}

\subsection{Illustration with Real Data}
Based on the synthetic results, the X.5-point curve could also be used with real images to detect near-degenerate minimal cases. To demonstrate this, we use image pairs given by the RANSAC 2020 dataset~\cite{mishkin} where standard point correspondences are available. Figure~\ref{fig:Real} shows that for a solution with large error compared to the ground truth, the remaining selected point is close to the degenerate curve. More real image results are in the supplementary materials. 
\begin{figure}
    \centering
    (a)\includegraphics[height = 0.3 \linewidth]{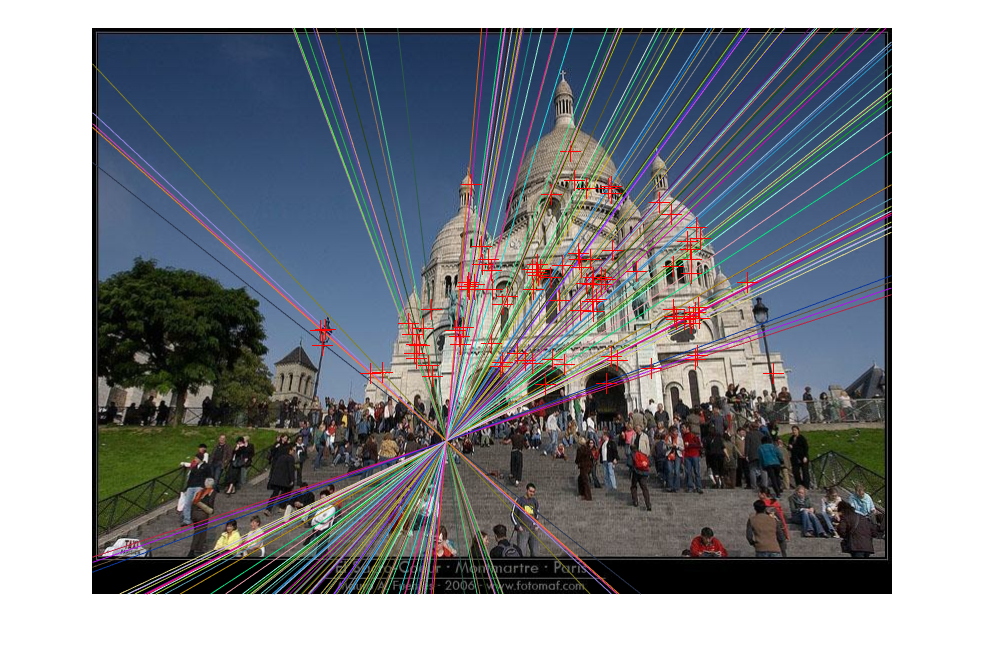} 
    \includegraphics[height = 0.3 \linewidth]{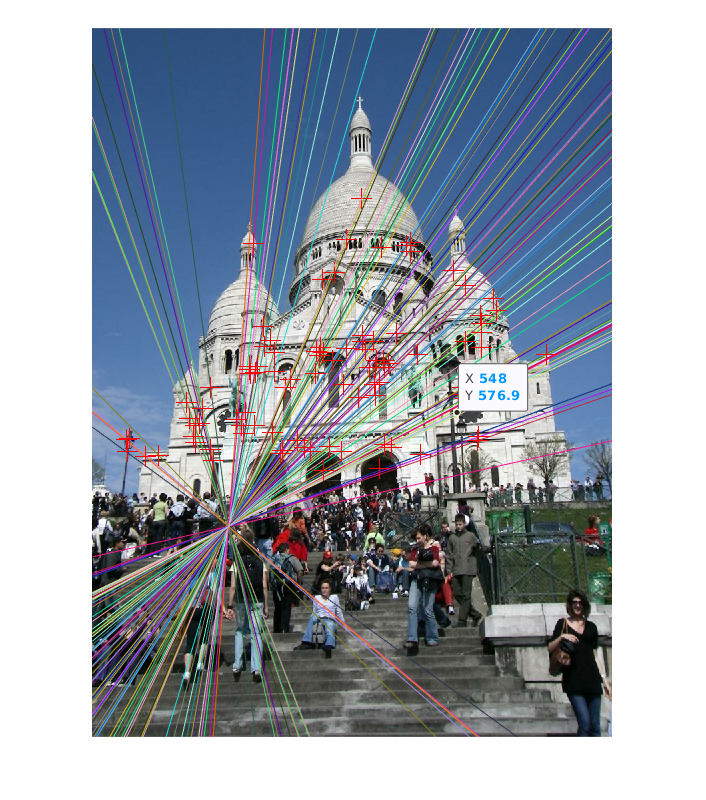} \\
    (b)\includegraphics[height = 0.3 \linewidth]{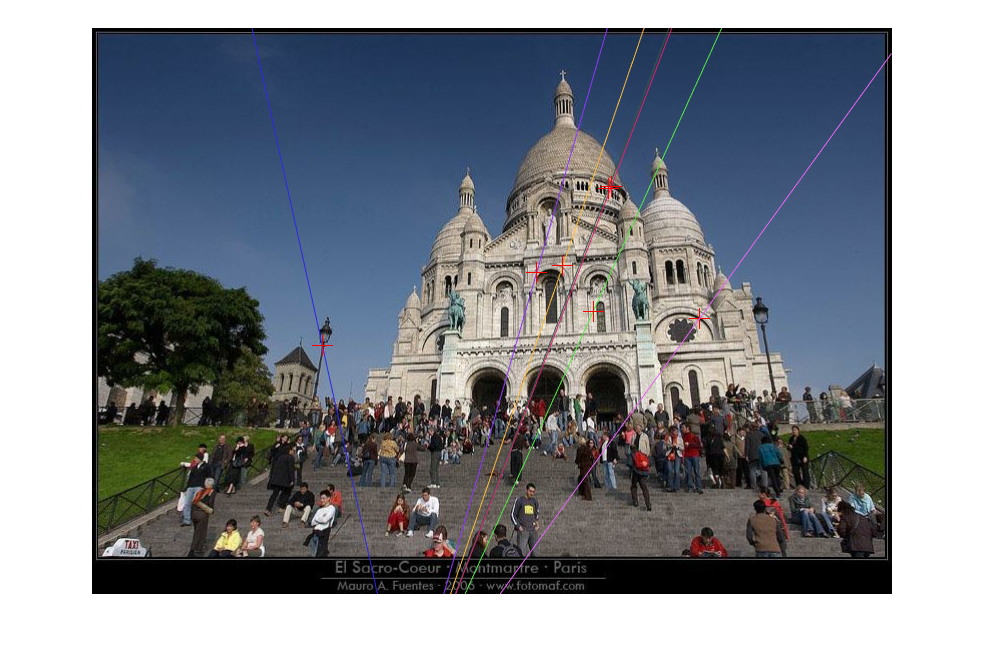} 
    \includegraphics[height = 0.3 \linewidth]{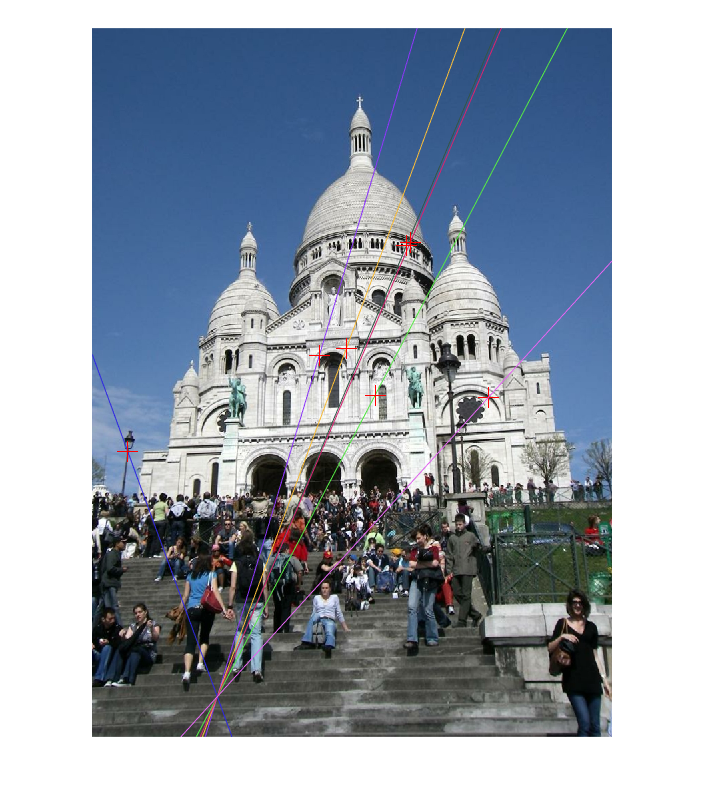} \\
    (c) \includegraphics[width = 0.75 \linewidth]{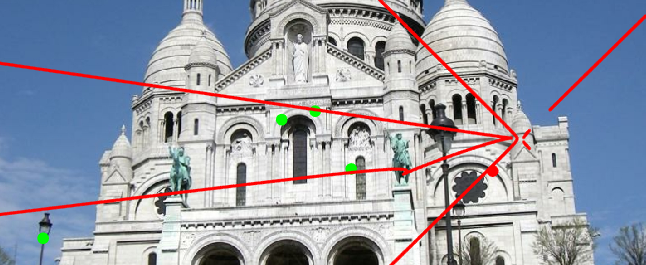}
    \caption{An example with real data to demonstrate an unstable minimal configuration  with all-inlier correspondences. (a) The ground truth epipolar geometry of a pair of images. (b) The closest solution found by the 7-point algorithm give 7 inliers. (c) Zoomed-in image showing that the remaining point is close to the degenerate curve, indicating that this is poorly-conditioned data.}
    \label{fig:Real}
\end{figure}

\section{Conclusion}
In this paper, we developed a general framework for analyzing the numerical instabilities of minimal problems in multiview geometry. 
We applied this to the problem of relative pose estimation, namely the  popular 5-point and 7-point problems. We determined condition number formulas, and we characterized the ill-posed world and image scenes. 

Numerical experiments on real and synthetic data supported our theoretical findings.  In particular we observed numerical instabilities for image data landing close to  the $4.5$- and $6.5$-point degenerate curves, which are used to describe ill-posed problem instances~in Theorems \ref{thm:image-E} and \ref{thm:F-image}.

Further, we related the numerical instabilities of minimal problems to the function of RANSAC inside SfM reconstructions.
Given all-inlier data, RANSAC is needed to overcome the ill-conditioning of relative pose estimation.

In future work, we could apply our theory to  other minimal problems, \eg partially calibrated relative pose estimation or three-view geometry.
In addition, it would be useful to develop an 
ultra-fast means of recognizing and filtering out poorly-conditioned image data.  Such could be applied before solving minimal problems and running RANSAC.

\paragraph{Acknowledgements.}
The authors are grateful to have participated in parts of 
the Algebraic Vision Research Cluster at ICERM, Brown University in Spring 2019, where they met each other and the seeds of this project were planted. 
BK and HF are supported by the NSF grant \nolinebreak IIS-1910530.

{\small
\bibliographystyle{ieee}
\bibliography{egbib}
}

\onecolumn

\begin{center}
     {\LARGE \bf 
     Supplementary Materials: Proofs and More Experiments \par}
      \vskip .5em
      \vspace*{8pt}
   \end{center}

\medskip
\renewcommand{\thesection}{S\arabic{section}}
\renewcommand{\thesubsection}{S\arabic{section}.\arabic{subsection}}
\renewcommand{\thefigure}{S\arabic{figure}}
\renewcommand{\thetable}{S\arabic{table}}
\renewcommand{\theequation}{S\arabic{equation}}
\setcounter{section}{0}
\setcounter{figure}{0}
\setcounter{table}{0}
\setcounter{equation}{0}

These supplementary materials have five sections in total, giving the additional information not covered in the main body of the paper. In Section~\ref{sec:Just}, we give the justification for Example~\ref{example:fundamental}. 
In Section~\ref{sec:sec4_1}, we provide proofs for the contents of Section~\ref{sec:conditionNumsFormulas}, and we also display explicit Jacobian matrices. Section~\ref{sec:sec4_2} supplies the proofs of Theorems~\ref{thm:illposed-world}  and \ref{thm:F-ill-posed-world}. Section~\ref{sec:sec4_3} gives the proofs of Theorems~\ref{thm:image-E} and \ref{thm:F-image}.  There, we also describe the steps for computing the degenerate X.5-point curves based on solving polynomial systems with homotopy continuation, or alternatively in the uncalibrated case, by specializing an explicit polynomial in Pl\"ucker coordinates. Finally, additional experimental results are shown in Section~\ref{sec:exp}.
For reproducibility, the code for this paper is available at
\url{https://github.com/HongyiFan/minimalInstability}.

\medskip
\medskip

\section{Justification for Example~\ref{example:fundamental}}
\label{sec:Just}
\noindent Here we justify the claim that for $(A,B)$ lying in a certain open dense subset $\mathcal{U}$ of the set of pairs of uncalibrated cameras:
 \begin{equation*}
\mathcal{C} = \{(A, B) \in \mathbb{P}(\mathbb{R}^{3 \times 4})^{\times 2} : \operatorname{rank}(A) = \operatorname{rank}(B) =3 \}, 
\end{equation*}
there exists a unique world transformation $g \in \operatorname{PGL}(4)$ and vector of parameters $b \in \mathbb{R}^7$ such that 
\begin{equation}\label{eq:desired}
(A g, B g) = \left( [I \,\, 0], M(b) \right) \, \in \, \mathbb{P}(\mathbb{R}^{3 \times 4})^{\times 2}
\end{equation}
where $M(b) \in \mathbb{P}(\mathbb{R}^{3 \times 4})$ is as defined in Eq.~\eqref{eq:Mb} of the main text.
Specifically, we claim that we can take the set to be
\begin{equation} \label{eq:explicit-U}
  \mathcal{U} =   \{ (A,B) \in \mathcal{C} \,\, : \,\, 
  \det[A; B(3,:)] \neq 0, \,\,\,\,
   B(1,:) ([A; B(3,:)]^{-1} (:,1)) \neq 0 \},
\end{equation}
where we are using Matlab notation to denote submatrices and matrix concatenations.

Firstly, we note that the conditions in Eq.~\eqref{eq:explicit-U} are independent of the choice of scales in $A$ and $B$, so they describe a well-defined subset of projective space.
Indeed if $\lambda$ and $\mu$ are nonzero scalars, then
\begin{align}
    & \det[\lambda A; (\mu B)(3,:)] = \lambda^3 \mu \det [A; B(3,:)], \nonumber \\[2pt]
    & (\mu B)(1,:)([\lambda A; (\mu B)(3,:)]^{-1}(:,1)) = \mu \lambda^{-1} B(1,:)([A; B(3,:)]^{-1}(:,1)). \label{eq:good-scale}
\end{align}

Next, let $(A,B) \in \mathcal{U}$.  Note that Eq.~\eqref{eq:desired} holds if and only there exist scales for $A, B, g$ such that in affine space we have
\begin{equation} \label{eq:affine}
[A; B] g = [[I \,\, 0]; M(b)] \, \in \, \mathbb{R}^{6 \times 4}. 
\end{equation}
Comparing rows $1, 2, 3, 6$ in Eq.~\eqref{eq:affine}, we must have $g = [A; B(3,:)]^{-1}$.
Then $(Bg)(1,1) \neq 0$ by \eqref{eq:good-scale}, and we can choose scales for $A, B$ so that $(Bg)(1,1) = 1$ by Eq.~\eqref{eq:good-scale}.
This concludes the justification for Example~\ref{example:fundamental}. \hspace{3cm} $\square$

\medskip
\medskip

\section{Proofs for ``Section~\ref{sec:conditionNumsFormulas}: Condition Number Formulas"}
\label{sec:sec4_1}
In this section, we prove Propositions~\ref{prop:formula-E} and \ref{prop:formula-F} from the main body, and we display explicit Jacobian matrices.

\smallskip

\subsection{Preliminaries on Tangent Spaces, Inner Products and Orthonormal Bases}

First we collect together basic facts about the relevant Riemannian manifolds.

\paragraph{$\bullet$ \underline{Special orthogonal group.}}
Consider $\operatorname{SO}(3)$. By linearizing the  equations $RR^{\top} = R^{\top}R = I$, 
\begin{equation*}
T(\operatorname{SO}(3), R) = \{ \delta R \in \mathbb{R}^{3 \times 3} : (\delta R) R^{\top} + R (\delta R)^{\top} = R^{\top} (\delta R) + (\delta R)^{\top}R = 0\} \subseteq \mathbb{R}^{3 \times 3}.
\end{equation*}
This tangent space may be parameterized as $R$ multiplied by skew-symmetric matrices:
\begin{equation} \label{eq:SO3-tangentspace}
    T(\operatorname{SO}(3), R) = \{ R[s]_{\times} : s \in \mathbb{R}^3 \},
\end{equation}
where $[s]_{\times} := \begin{pmatrix} 0 & -s_3 & s_2 \\ s_3 & 0 & -s_1 \\ -s_2 & s_1 & 0 \end{pmatrix}$ for $s = \begin{pmatrix} s_1 \\ s_2 \\ s_3 \end{pmatrix}$. 
The Riemannian metric's inner product on the tangent space is  the restriction of the Frobenius inner product on $\mathbb{R}^{3 \times 3}$, 
\begin{equation*} \label{eq:realnice-one}
    \langle R[s]_{\times}, R[\tilde{s}]_{\times} \rangle := \operatorname{trace}((R[s]_{\times})^{\top} R[\tilde{s}]_{\times}) = \operatorname{trace}([s]_{\times}^{\top} [\tilde s]_{\times}) = 2 \langle s, \tilde{s} \rangle,
\end{equation*}
where the rightmost inner product is the standard one on $\mathbb{R}^3$.  
An orthonormal basis for $T(\operatorname{SO}(3), R)$ is
\begin{equation} \label{eq:SO3-orthonormal}
    \frac{1}{\sqrt{2}} R [e_1]_{\times}, \frac{1}{\sqrt{2}} R [e_2]_{\times}, \frac{1}{\sqrt{2}} R [e_3]_{\times},
\end{equation}
where $e_1, e_2, e_3$ is the standard basis on $\mathbb{R}^3$.

\paragraph{$\bullet$ \underline{Unit sphere.}} Consider the two-dimensional unit sphere $\mathbb{S}^2$. 
Its tangent space are the perpendicular spaces:
\begin{equation*}
    T(\mathbb{S}^2, t) = t^{\perp} := \{\tilde t \in \mathbb{R}^3 : \langle t , \tilde{t} \rangle = 0\} \subseteq \mathbb{R}^3.
\end{equation*}
The Riemannian metric's inner product arises by restricting of the Euclidean inner product on $\mathbb{R}^3$ .  We fix
\begin{equation} \label{eq:S2-orthonormal}
    t^{\perp}_1, t^{\perp}_2 \in \mathbb{R}^3
\end{equation}
to be an orthonormal basis for $T(\mathbb{S}^2, t)$.

\paragraph{$\bullet$ \underline{Projective space.}}
Consider the real projective space of $3 \times 3$ matrices, $\mathbb{P}(\mathbb{R}^{3 \times 3})$.
The map 
\begin{equation} \label{eq:Riem-map}
    \mathbb{S}^8 = \{M \in \mathbb{R}^{3 \times 3} : \| M \|_F = 1 \} \longrightarrow \mathbb{P}(\mathbb{R}^{3 \times 3}), \,\,\,\, M \mapsto [M]
\end{equation}
witnesses $\mathbb{P}(\mathbb{R}^{3 \times 3})$ as a quotient of $\mathbb{S}^8$ by  $\mathbb{Z}/2\mathbb{Z}$ acting via a sign flip.  
By \cite[Exam.~2.34 and Prop.~2.32]{lee2006riemannian}, this induces the structure of a Riemannian manifold on $\mathbb{P}(\mathbb{R}^{3 \times 3})$ such that \eqref{eq:Riem-map} is locally an isometry.
At a given point in $\mathbb{P}(\mathbb{R}^{3 \times 3})$ we can choose a representative $M \in \mathbb{S}^8$ and the tangent space can be identified as follows:
\begin{equation} \label{eq:proj-space}
T(\mathbb{P}(\mathbb{R}^{3 \times 3}), [M])  \cong  T(\mathbb{S}^8, M) = M^{\perp} = \{\tilde M \in \mathbb{R}^{3 \times 3} : \langle M, \tilde M \rangle = 0\} \subseteq \mathbb{R}^{3 \times 3}.
\end{equation}
The Riemannian metric's inner product is the Frobenius inner product on $M^{\perp}$.

\paragraph{$\bullet$ \underline{Essential matrices.}} 
Consider the manifold of real essential matrices,
\begin{equation*}
\mathcal{E} \subseteq \mathbb{P}(\mathbb{R}^{3 \times 3}).  
\end{equation*}
(This departs from the notation in the main body.) 
It is known that $\mathcal{E}$ is a compact smooth real manifold of dimension $5$.

\begin{lemma} \label{lem:E-submer}
At each point in $\operatorname{SO}(3) \times \mathbb{S}^2$, the differential of the map
\begin{equation*}
\operatorname{SO}(3) \times \mathbb{S}^2 \rightarrow \mathcal{E}, \,\, (R,t) \mapsto R [t]_{\times}
\end{equation*}
has rank $5$.  Thus the map is a submersion onto the manifold of real essential matrices $\mathcal{E} \subseteq \mathbb{P}(\mathbb{R}^{3 \times 3})$.
\end{lemma}
\begin{proof}
The map is linear separately in $R$ and $t$.  So by the product rule, at $(\delta R, \delta t) \in T(\operatorname{SO}(3), R) \times  T(\mathbb{S}^2, t)$ its differential evaluates to $\frac{1}{\sqrt{2}} (\delta R)[t]_{\times} + \frac{1}{\sqrt{2}} R[\delta t]_{\times} \in T(\mathcal{E}, R[t]_{\times}) \subseteq T(\mathbb{P}(\mathbb{R}^{3 \times 3}), R[t]_{\times}) = T(\mathbb{S}^8, \frac{1}{\sqrt{2}} R[t]_{\times}) = (R [t]_{\times})^{\perp} \subseteq \mathbb{R}^{3 \times 3}$, where we used \eqref{eq:proj-space}. 
We need to show that this quantity equals $0$ only if $\delta R = 0$ and $\delta t = 0$.
By \eqref{eq:SO3-tangentspace}, $\delta R = R [s]_{\times}$ for some $s \in \mathbb{R}^3$ and $\delta t$ is perpendicular to $t$.  Substituting these in gives the condition
\begin{equation*}
    \tfrac{1}{\sqrt{2}} R [s]_{\times} [t]_{\times} + \tfrac{1}{\sqrt{2}} R [\delta t]_{\times} = 0.
\end{equation*}
Left-multiplying by $\sqrt{2} R^{\top}$, this is equivalent to
\begin{equation} \label{eq:getting-better}
    [s]_{\times} [t]_{\times} + [\delta t]_{\times} = 0.
\end{equation}
If we multiply on the right by $t$, it follows that $[\delta t]_{\times}t = 0$.  But if $\delta t \neq 0$, then $[\delta t]_{\times}$ is rank-$2$ with kernel spanned by $\delta t$ which is perpendicular to $t$. The last two sentences give a contradiction.  Thus we must have $\delta t = 0$.  
So now \eqref{eq:getting-better} reads
\begin{equation} \label{eq:near-done}
    [s]_{\times} [t]_{\times} = 0.
\end{equation}
Assume $s \neq 0$. Then $[s]_{\times}$ is a rank $2$ matrix of size $3  \times 3$.  Since $[t]_{\times}$ is rank $2$ and $3 \times 3$ as well (recall $t \in \mathbb{S}^2$ so that $t \neq 0$), the product $ [s]_{\times} [t]_{\times}$ must have rank at least $1$. This contradicts \eqref{eq:near-done}, so $s=0$, and the lemma follows.
\end{proof}

\medskip

 Lemma~\ref{lem:E-submer} lets us write down tangent spaces to the essential matrices:  

\begin{equation*}
    T(\mathcal{E}, R[t]_{\times}) \, = \, \{ \tfrac{1}{\sqrt{2}}R [s]_{\times} [t]_{\times} + \tfrac{1}{\sqrt{2}} R [\delta t]_{\times} : s \in \mathbb{R}^3, \delta t \in \mathbb{R}^3, \langle \delta t, t \rangle = 0\} \, 
    \subseteq \, \mathbb{R}^{3 \times 3}.
\end{equation*}
The Riemannian metric's inner product is the restriction of the Frobenius inner product on $\mathbb{R}^{3 \times 3}$.
We get an orthonormal basis for the tangent space 
by orthonormalizing the image of  \eqref{eq:SO3-orthonormal} and \eqref{eq:S2-orthonormal}, \ie, by orthonormalizing 
\begin{equation} \label{eq:not-orthonormal}
    \frac{1}{2}R[e_1]_{\times} [t]_{\times}, \,\,\,\,\,\, \frac{1}{2}R[e_2]_{\times} [t]_{\times}, \,\,\,\,\,\, \frac{1}{2}R[e_3]_{\times} [t]_{\times}, \,\,\,\,\,\,
    \frac{1}{\sqrt{2}} R[t_{1}^{\perp}]_{\times}, \,\,\,\,\,\,
    \frac{1}{\sqrt{2}}  R[t_{2}^{\perp}]_{\times}.
\end{equation}
Elementary linear algebra implies that if $\alpha \in \mathbb{R}^5$ expresses an element of $T(\mathcal{E}, R[t]_{\times})$ in terms of the basis \eqref{eq:not-orthonormal} then $G^{1/2} \alpha$ expresses the same tangent vector in terms of an orthonormal basis for $T(\mathcal{E}, R[t]_{\times})$, where $G$ is the Grammian matrix for the matrices in \eqref{eq:not-orthonormal} with respect to the Frobenius inner product.  Explicitly,  $G$ equals
\begin{tiny}
\begin{equation} \label{eq:crazy-grammian}
     \begin{pmatrix} 
    \frac{1}{2}t_1^2 + \frac{1}{4}t_2^2 + \frac{1}{4}t_3^2 & \frac{1}{4}t_1t_2 & \frac{1}{4}t_1t_3 & -\frac{1}{2\sqrt{2}}t_3(t^{\perp}_1)_2 + \frac{1}{2\sqrt{2}}t_2(t^{\perp}_1)_3 & -\frac{1}{2\sqrt{2}}t_3(t^{\perp}_2)_2 + \frac{1}{2\sqrt{2}}t_2(t^{\perp}_2)_3  \\[5pt] 
    \frac{1}{4}t_1t_2 & \frac{1}{4}t_1^2 + \frac{1}{2}t_2^2 + \frac{1}{4}t_3^2 & \frac{1}{4}t_2t_3 & \frac{1}{2\sqrt{2}}t_3(t^{\perp}_1)_1-\frac{1}{2\sqrt{2}}t_1(t^{\perp}_{1})_3 & \frac{1}{2\sqrt{2}}t_3(t^{\perp}_2)_1-\frac{1}{2\sqrt{2}}t_1(t^{\perp}_{2})_3 \\[5pt]
    \frac{1}{4}t_1t_3 & \frac{1}{4}t_2t_3 &
      \frac{1}{4}t_1^2 + \frac{1}{4}t_2^2 + \frac{1}{2}t_3^2 & -\frac{1}{2\sqrt{2}}t_2(t^{\perp}_1)_1+\frac{1}{2\sqrt{2}}t_1(t^{\perp}_1)_2 & -\frac{1}{2\sqrt{2}}t_2(t^{\perp}_2)_1+\frac{1}{2\sqrt{2}}t_1(t^{\perp}_2)_2 \\[5pt]
      -\frac{1}{2\sqrt{2}}t_3(t^{\perp}_1)_2 + \frac{1}{2\sqrt{2}}t_2(t^{\perp}_1)_3 & \frac{1}{2\sqrt{2}}t_3(t^{\perp}_1)_1-\frac{1}{2\sqrt{2}}t_1(t^{\perp}_{1})_3 & -\frac{1}{2\sqrt{2}}t_2(t^{\perp}_1)_1+\frac{1}{2\sqrt{2}}t_1(t^{\perp}_1)_2 & 1 & 0 \\[5pt]
      -\frac{1}{2\sqrt{2}}t_3(t^{\perp}_2)_2 + \frac{1}{2\sqrt{2}}t_2(t^{\perp}_2)_3  & \frac{1}{2\sqrt{2}}t_3(t^{\perp}_2)_1-\frac{1}{2\sqrt{2}}t_1(t^{\perp}_{2})_3 & -\frac{1}{2\sqrt{2}}t_2(t^{\perp}_2)_1+\frac{1}{2\sqrt{2}}t_1(t^{\perp}_2)_2 &  0 & 1
    \end{pmatrix}
\end{equation}
\end{tiny}

\paragraph{$\bullet$ \underline{Fundamental matrices.}} 
Consider the manifold of real fundamental matrices,
\begin{equation*}
\mathcal{F} \subseteq \mathbb{P}(\mathbb{R}^{3 \times 3}) .  
\end{equation*}
(This departs from the notation in the main body.)
It is known that $\mathcal{F}$ is a non-compact smooth real manifold of dimension $7$.

We will work with $\mathcal{F}$  using the parameterization from $\mathbb{R}^7$ given by Eq.~\eqref{eq:F-def}.   This  sends $b \in \mathbb{R}^{7}$ to

\begin{equation} \label{eq:fund-b}
F(b) := \begin{pmatrix} 
b_4 & b_5 &  b_6 \\ 
-1 & -b_1 &  -b_2 \\
-b_3b_4+b_7 & -b_3b_5+b_1b_7 & -b_3b_6+b_2b_7
\end{pmatrix}.
\end{equation}

\begin{lemma} \label{lem:F-param}
At each point $b \in \mathbb{R}^7$ where the camera matrix $M(b)$  in \eqref{eq:Mb} has full rank, the differential of the map $F: \mathbb{R}^7 \dashrightarrow \mathcal{F}$  has rank $7$. 
Thus $F$ is a submersion on the open set where it is defined.
\end{lemma}
\begin{proof}
The differential of $F$ at $b$ evaluated at $\delta b \in \mathbb{R}^7$ equals
\begin{footnotesize}
\begin{equation} \label{eq:F-diff}
    \begin{pmatrix}
    (\delta b)_4 & (\delta b)_5 & (\delta b)_6 \\[5pt]
    0 & -(\delta b)_1 & -(\delta b)_2 \\[5pt]
    -(\delta b)_3 b_4 - b_3(\delta b)_4 + (\delta b)_7 & -(\delta b)_3 b_5 - b_3 (\delta b)_5  + (\delta b)_1 b_7 + b_1 (\delta b)_7 & -(\delta b)_3 b_6 - b_3 (\delta b)_6  + (\delta b)_2 b_7 + b_2 (\delta b)_7
    \end{pmatrix}.
\end{equation}
\end{footnotesize}
Equating this with $0$, the first two rows show that $0 = (\delta b)_1 = (\delta b)_2 = (\delta b)_4 = (\delta b)_5 = (\delta b)_6$. 
Then the last row reads:
\begin{equation} \label{eq:getting-thereee}
    \begin{pmatrix} -b_4 & 1 \\ -b_5 & b_1 \\ -b_6 & b_2 \end{pmatrix} \begin{pmatrix} (\delta b)_3 \\ (\delta b)_7 \end{pmatrix} = 0.
\end{equation}
The coefficient matrix in \eqref{eq:getting-thereee} consists of the first two rows of $F(b)$ transposed and negated.  However the first two rows of $F(b)$ span the row space of $F(b)$, since the third row of $F(b)$ is $-b_3$ times the first row added to $-b_7$ times the second row.  
Because $F(b)$ has rank $2$,
 $(\delta b)_3 = (\delta b)_7 = 0$.  
 All together, $\delta b = 0$ whence $DF(b)$ is injective.
\end{proof}

\medskip

Lemma~\ref{lem:F-param} lets us write down the tangent spaces to fundamental matrices.  They are spanned by the matrices \eqref{eq:F-diff} as $\delta b$ ranges over a standard basis $e_1, \ldots, e_7$ for $\mathbb{R}^7$. 
The Riemannian metric's inner product is the restriction of the Frobenius inner product.
We get an orthonormal basis for $T(\mathcal{F}, F(b))$ by orthonormalizing 
\begin{equation} \label{eq:F-first-basis}
    \frac{1}{\| F(b) \|_F} \frac{\partial F(b)}{\partial b_1}, \,\,\,\,\, \ldots \,\,\,\,\,,  \frac{1}{\| F(b) \|_F} \frac{\partial F(b)}{\partial b_7}. 
\end{equation}
Elementary linear algebra implies that if $\alpha \in \mathbb{R}^7$ expresses an element of $T(\mathcal{F}, F(b))$ in terms of the basis \eqref{eq:F-first-basis} then $G^{1/2} \alpha$ expresses the same tangent vector in terms of an orthonormal basis for $T(\mathcal{F}, F(b))$, where $G$ is the Grammian matrix for the matrices in 
\eqref{eq:F-first-basis} with respect to the Frobenius inner.  Explicitly,  $G$ equals
\begin{equation} \label{eq:less-crazy-grammian}
\begin{footnotesize}
  \frac{1}{\| F(b) \|_F^2}  \begin{pmatrix}
    b_7^2+1 & 0 & -b_5b_7 & 0 & -b_3b_7 & 0 & b_1b_7 \\[1pt]
    0 & b_7^2+1 & -b_6b_7 & 0 & 0 & -b_3b_7 & b_2b_7 \\[1pt] 
    -b_5b_7 & -b_6b_7 & b_4^2+b_5^2+b_6^2 & b_3b_4 & b_3b_5 & b_3b_6 & -b_1b_5-b_2b_6-b_4 \\[1pt] 
    0 & 0 & b_3b_4 & b_3^2+1 & 0 & 0 &
     -b_3 \\[1pt]
     -b_3b_7 & 0 & b_3b_5 & 0 & b_3^2+1 & 0 & -b_1b_3 \\[1pt]
     0 & -b_3b_7 & b_3b_6 & 0 & 0 & b_3^2+1 & -b_2b_3 \\[1pt]
     b_1b_7 & b_2b_7 & -b_1b_5-b_2b_6-b_4 & -b_3 & -b_1b_3 & -b_2b_3 & b_1^2+b_2^2+1
    \end{pmatrix}.
    \end{footnotesize}
\end{equation}

\smallskip

\subsection{Proof of Proposition~\ref{prop:formula-E}}

\begin{proof}
Uniqueness of the reconstruction map is by Lemma~\ref{lem:first} (which is a restatement of the inverse function theorem). 
This is because we are assuming that the world scene $w \in \operatorname{SO}(3) \times \mathbb{S}^2 \times (\mathbb{R}^3)^{\times 5}$ is not ill-posed.
Eq.~\eqref{eq:cond-formula}  expresses the condition number of $\mathbf{S}$ as the largest singular value of the product of a $5 \times 20$ matrix and  the inverse of a $20 \times 20$ matrix:
\begin{equation*}
    \| D \Psi(w) \circ D \Phi(w)^{-1} \|.
\end{equation*}

We need to make this formula explicit.
Here the forward map is given by Eq.~\eqref{eq:E-forward} of the main body:
\begin{equation*}
    \Phi(R, t, X_1, \ldots, X_5) = ((\beta(X_1), \beta(RX_1 + t)), \ldots (\beta(X_5),\beta(RX_5 + t))),
\end{equation*}
where $\beta : \mathbb{R}^3 \dashrightarrow \mathbb{R}^2$ is the projection $\beta(a_1, a_2, a_3) = (a_1/a_3, a_2/a_3)$ defined whenever $a_3 \neq 0$.
It is natural to factor $\Phi = \Phi_2 \circ \Phi_1$ as the composition of a map $\Phi_1 : \operatorname{SO}(3) \times \mathbb{S}^2 \times (\mathbb{R}^3)^{\times 5} \rightarrow (\mathbb{R}^{3} \times \mathbb{R}^3)^{\times 5}$ given by
\begin{equation*}
    \Phi_1(R, t, X_1, \ldots, X_5) := ((X_1, RX_1 + t), \ldots, (X_5, RX_5 + t)),
\end{equation*}
followed by an almost-everywhere-defined map $\Phi_2: (\mathbb{R}^{3} \times \mathbb{R}^3)^{\times 5} \dashrightarrow (\mathbb{R}^2 \times \mathbb{R}^2)^{\times 5}$ given by
\begin{equation*}
\Phi_2((Z_1,\tilde{Z}_1), \ldots, (Z_5, \tilde{Z}_5)) := ((\beta(Z_1), \beta(\tilde{Z}_1)), \ldots, (\beta(Z_5),\beta(\tilde Z_5))).
\end{equation*}

By the chain rule,  $D\Phi(w) = D\Phi_2(\Phi(w)) \circ D\Phi_1(w)$.  
This writes the forward Jacobian matrix as the product of a $20 \times 30$ matrix multiplied by a $30 \times 20$ matrix.  
Let us explicitly write down $D\Phi_1(w)$ in terms of the orthonormal bases for the tangent spaces from the previous section, with columns ordered according to $\delta X_1, \ldots, \delta X_5, \delta R, \delta t$ (corresponding to an orthonormal basis for $T(\operatorname{SO}(3) \times \mathbb{S}^2 \times (\mathbb{R}^3)^{\times 5}, w)$), and rows ordered according to $\delta Z_1, \ldots, \delta Z_5, \delta \tilde Z_1, \ldots, \delta \tilde Z_5$ (corresponding to a standard basis on $(\mathbb{R}^3 \times \mathbb{R}^3)^{\times 5}$).  
Since $\Phi_1$ is separately linear in $X_1, \ldots, X_5, R, t$, we can compute the following \nolinebreak block \nolinebreak form:

\begin{equation}\label{eqn:jac-1}
D\Phi_{1}(w) \,\, = \,\,
  \begin{pmatrix}
  & &  & \rvline & & & & & \\
  & &  & \rvline & & & & & \\
  &  \bigeye_{15 \times 15} & &  \rvline & \hspace{8em} \bigzero_{15 \times 5} & & & 
  & \\ & & &  \rvline & & & & & \\[0.7em]
\hline 
& & & \rvline &&& \\[-0.5em]
R & & & \rvline & 
  \frac{1}{\sqrt{2}} R \begin{pmatrix} [{e}_{1}]_{\times} X_{1} & [{ e}_{2}]_{\times} X_{1} & [{ e}_{3}]_{\times} X_{1} \end{pmatrix} & t_1^{\perp} & t_2^{\perp} \\[1em]
    & \ddots & & \rvline 
  &\vdots & \vdots  & \vdots \\[1em]
  & & R & \rvline &
  \frac{1}{\sqrt{2}} R \begin{pmatrix} [{e}_{1}]_{\times} X_{5} & [{ e}_{2}]_{\times} X_{5} & [{ e}_{3}]_{\times} X_{5} \end{pmatrix} & t_1^{\perp} & t_2^{\perp} 
\end{pmatrix}_{\! 30 \times 20}.
\end{equation}

The Jacobian matrix $D\Phi_{2}$ has the following block-diagonal form with respect to the standard bases:

\begin{equation}\label{eqn:jac-2}
D \Phi_{2}(\Phi_1(w)) \,\, = \,\, 
\begin{pmatrix}
\frac{\partial \beta} {\partial Z_1} & & & & & \\
& \ddots & & & & \\
& & \frac{\partial \beta} {\partial Z_5} & & & \\
& & & \frac{\partial \beta} {\partial \tilde Z_1} & & \\
& & & & \ddots & \\
& & & & & \frac{\partial \beta} {\partial \tilde Z_5}
\end{pmatrix}_{\!20 \times 30}.
\end{equation}

\noindent Here, \eg $\frac{\partial \beta}{\partial Z_1}$ denotes the $2 \times 3$ Jacobian matrix of $\beta(Z_1) = \begin{pmatrix} \frac{(Z_1)_1}{(Z_1)_3} & \frac{(Z_1)_2}{(Z_1)_3} \end{pmatrix}^{\!\!\top}$ with respect to $Z_1 $. Explicitly,
\begin{equation*} \label{eq:easy-block}
 \frac{\partial \beta}{\partial Z_1} =
 \begin{pmatrix}
 \frac{1}{(Z_1)_3} & 0 & \frac{-(Z_1)_1}{(Z_{1})_3^{2}} \\[8pt]
 0 & \frac{1}{(Z_1)_3} & \frac{-(Z_1)_2}{(Z_{1})_3^{2}}
 \end{pmatrix}
\end{equation*}
and likewise for the other blocks.
In \eqref{eqn:jac-2}, the Jacobian is evaluated at $\Phi(w)$, \ie $Z_1 = X_1, \ldots, Z_5 = X_5$ and $\tilde Z_1 = R X_1 + t, \ldots, \tilde Z_5 = R X_5 + t$.
Multiplying \eqref{eqn:jac-1} with \eqref{eqn:jac-2} and then inverting gives the $20 \times 20$ matrix $(D\Phi(w))^{-1}$.

Next consider differential of the epipolar map, \ie the $5 \times 20$ matrix $D\Psi(w)$.
Here $\Psi$ factors as the coordinate projection $(X_1, \ldots, X_5, R, t) \mapsto (R,t)$ followed by the map $(R,t) \mapsto R[t]_{\times}$.  
Of course, the Jacobian of the projection is 
\begin{equation*}
\begin{pmatrix} 0_{5 \times 15} & I_{5 \times 5} \end{pmatrix}.
\end{equation*}
As for $(R, t) \mapsto R[t]_{\times}$, if we express its Jacobian  so that the rows correspond to the non-orthonormal basis \eqref{eq:not-orthonormal} for the tangent space $T(\mathcal{E}, R[t]_{\times})$, then we simply get $I_5$.  
Then re-expressing this in terms of an orthonormal basis for the tangent space, we need to multiply by a positive-definite square root $G^{1/2}$ for the $5 \times 5$ Grammian matrix in \eqref{eq:crazy-grammian}.

All together, the product $D\Psi(w) \circ (D\Phi(w))^{-1}$ is computed by multiplying \eqref{eqn:jac-2} with \eqref{eqn:jac-1} (in that order); inverting the product; selecting the last 5 rows of the inverse; and finally multiplying on the left by $G^{1/2}$.  The condition number of the solution map is the largest singular value of the resulting $ 5 \times 20$ matrix.  This finishes Proposition~\ref{prop:formula-E}. 
\end{proof}

\smallskip

Before proceeding, we record an easy fact that will be useful in Section~\ref{sec:sec4_2}.

\begin{remark} \label{rem:kernel}
The kernel of the $2 \times 3$ matrices $\frac{\partial \beta}{\partial Z_i}$ and $\frac{\partial \beta}{\partial \tilde Z_i}$ in \eqref{eqn:jac-2} are spanned by $Z_i$ and $\tilde Z_i$ respectively.
\end{remark}

\smallskip

\subsection{Proof of Proposition~\ref{prop:formula-F}}

\begin{proof}
This is very similar to Proposition~\ref{prop:formula-E}. 
Uniqueness of the reconstruction map is by Lemma~\ref{lem:first}.
We obtain explicit Jacobian formulas by first factoring $\Phi = \Phi_2 \circ \Phi_1$ where $\Phi_1 : \mathbb{R}^7 \times (\mathbb{R}^3)^{\times 7} \rightarrow (\mathbb{R}^3 \times \mathbb{R}^3)^{\times 7}$ is given by 
\begin{equation*}
\Phi_1(b, X_1, \ldots, X_7) = ((X_1, M(b) \begin{pmatrix} X_1 \\ 1 \end{pmatrix}), \ldots, (X_7, M(b) \begin{pmatrix} X_7 \\ 1 \end{pmatrix}))
\end{equation*}
and $\Phi_2 : (\mathbb{R}^3 \times \mathbb{R}^3)^{\times 7} \dashrightarrow (\mathbb{R}^2 \times \mathbb{R}^2)^7$ is given by 
\begin{equation*}
\Phi_2((Z_1, \tilde Z_1), \ldots, (Z_7, \tilde Z_7)) = ((\beta(Z_1), \beta(\tilde Z_1)), \ldots, (\beta(Z_7), \beta(\tilde Z_7))).
\end{equation*}
The chain rule gives $D\Phi(w) = D\Phi_2(w) \circ D\Phi_1(w)$.  Here all spaces involved in the forward map are Euclidean spaces, so we use the standard orthonormal bases to write down the matrices.

The first matrix $D\Phi_1(w)$ is $42 \times 28$.  
Ordering its columns according to $\delta X_1, \ldots, \delta X_7, \delta b$ and its rows according to $\delta Z_1, \ldots, \delta Z_7, \delta \tilde Z_1, \ldots, \delta \tilde Z_7$, it reads

\begin{equation} \label{eq:large-displaye}
D\Phi_{1}(w)  = 
  \begin{pmatrix}
  & &  & \rvline & & & & & \\
  & &  & \rvline & & & & & \\
  &  \bigeye_{21 \times 21} & &  \rvline & \hspace{8em} \bigzero_{21 \times 7} & & & 
   \\ & & &  \rvline & & &  \\[0.7em]
\hline 
& & & \rvline &&& \\[-0.5em]
M(b)(1:3,1:3) & & & \rvline & \frac{\partial M(b)}{\partial b_1} \begin{pmatrix} X_1 \\ 1 \end{pmatrix} & \cdots & \frac{\partial M(b)}{\partial b_7} \begin{pmatrix} X_1 \\ 1  \end{pmatrix} \\[1em]
    & \ddots & & \rvline 
  &\vdots & & \vdots \\[1em]
  & & M(b)(1:3,1:3) & \rvline & \frac{\partial M(b)}{\partial b_1} \begin{pmatrix} X_7 \\ 1  \end{pmatrix}  & \cdots & \frac{\partial M(b)}{\partial b_7} \begin{pmatrix} X_7 \\ 1  \end{pmatrix} \\[1em]
\end{pmatrix}_{\! 42 \times 28}
\end{equation}

\noindent The bottom-left $21 \times 21$ submatrix is block-diagonal with seven $3 \times 3$  blocks, each of which is $M(b)(1:3, 1:3)$ denoting the first three columns of $M(b)$.
In the bottom-right $21 \times 7$ submatrix, note that each matrix $\frac{\partial M(b)}{\partial b_i}$ is zero in all but one entry where it takes the value of $1$.

The second Jacobian matrix $D\Phi_2(\Phi_1(w))$ is $28 \times 42$.  
It is block-diagonal with fourteen blocks each of size $3 \times 2$, analogously to \eqref{eqn:jac-2} with Remark~\ref{rem:kernel} still applying:
\begin{equation}\label{eqn:jac-2-second}
D \Phi_{2}(\Phi_1(w)) \,\, = \,\, 
\begin{pmatrix}
\frac{\partial \beta} {\partial Z_1} & & & & & \\
& \ddots & & & & \\
& & \frac{\partial \beta} {\partial Z_7} & & & \\
& & & \frac{\partial \beta} {\partial \tilde Z_1} & & \\
& & & & \ddots & \\
& & & & & \frac{\partial \beta} {\partial \tilde Z_7}
\end{pmatrix}_{\!28 \times 42}.
\end{equation}
Multiplying \eqref{eq:large-displaye} with \eqref{eqn:jac-2-second} and inverting the product gives the $28 \times 28$ matrix $(D\Phi(w))^{-1}$.

Next we consider the differential of the epipolar map, \ie the $7 \times 28$ matrix $D \Psi(w)$. 
Here $\Psi$ factors as the coordinate projection 
$(X_1, \ldots, X_7, b) \mapsto b$ followed by the map $b \mapsto F(b)$ given by \eqref{eq:fund-b}.
Of course, the Jacobian of the projection is
\begin{equation*}
\begin{pmatrix}
0_{7 \times 21} & I_{7 \times 7}
\end{pmatrix}.
\end{equation*}
The Jacobian matrix of $b \mapsto F(b)$ is simply $I_7$, if we express it with respect to bases so that the rows correspond to the non-orthonormal basis \eqref{eq:F-first-basis} for the tangent space $T(\mathcal{F}, F(b))$. 
Re-expressing it in terms of an orthonormal basis for the tangent space, we need to multiply by a positive-definite square root $G^{1/2}$ for the $7 \times 7$ Grammian matrix in \eqref{eq:less-crazy-grammian}.

All together, the product $D\Psi(w) \circ (D\Phi(w))^{-1}$ is computed by multiplying \eqref{eq:large-displaye} with \eqref{eqn:jac-2-second} (in that order); inverting the product; selecting the last $7$ rows of the inverse; and finally multiplying on the left by $G^{1/2}$. 
The condition number of the solution map is the largest singular value of the resulting $7 \times 18$ matrix. 
This finishes Proposition~\ref{prop:formula-F}.
\end{proof}

\medskip
\medskip

\section{Proofs for ``Section~\ref{sec:ill-pose-world}: Ill-Posed World Scenes"}
\label{sec:sec4_2}
In this section, we characterize the degenerate world scenes for the 5-point and 7-point minimal problems in terms of quadric surfaces in $\mathbb{R}^3$.

\begin{remark} \label{rem:careful}
Our definition of ``quadric surface" given in the main body in Eq.~\eqref{eq:def-quad} includes the case of affine planes (which occur when the top-left $3 \times 3$ submatrix of $Q$ in \eqref{eq:def-quad} is zero).
This choice is deliberate, and needed for full accuracy in Theorems~\ref{thm:illposed-world} and \ref{thm:F-ill-posed-world}.  Likewise, by ``circle" in the statement of Theorem~\ref{thm:illposed-world} we mean a plane conic defined by 
\begin{equation} \label{eq:def-circle}
    \left\{ (a_1, a_2) \in \mathbb{R}^2 : \begin{pmatrix} a & 1 \end{pmatrix} Q \begin{pmatrix} a \\ 1 \end{pmatrix} = 0 \right\},
\end{equation}
for some symmetric matrix $Q \in \mathbb{R}^{3 \times 3}$ such that $Q_{11} = Q_{22}$ and $Q_{12} = Q_{21} =0$.   Eq.~\eqref{eq:def-circle} includes the cases of affine lines and points, interpreted as circles of radius $\infty$ and $0$ respectively.
\end{remark}

\smallskip

\subsection{Proof of Theorem~\ref{thm:illposed-world}}

\begin{proof}
The assumption that the forward map $\Phi$ is defined at the world scene $w$ implies that the points $X_i$ and $RX_i + t$ in $\mathbb{R}^3$ do not have a vanishing third coordinate for each $i = 1, \ldots, 5$.

Let $\Delta := \begin{pmatrix} \delta X_1 & \ldots & \delta X_5 & \delta r_1 & \delta r_2 & \delta r_3 & \delta t_1 & \delta t_2 \end{pmatrix}^{\!\top} \in \mathbb{R}^{20}$.  Our task is characterize for which scenes $w$ does there a nonzero solution to the linear system $D\Phi(w) \Delta = 0$, where the variable is $\Delta$.  Let us massage this equation repeatedly.

Fistly using the factorization $D\Phi = D\Phi_2 \circ D\Phi_1$ from the previous section, the explicit Jacobian matrix expressions \eqref{eqn:jac-1} and \eqref{eqn:jac-2}, and Remark~\ref{rem:kernel} characterizing the kernel of $D\Phi_2$, we equivalently have the system of equations
\begin{equation} \label{eq:good-stuff}
\begin{cases}
    \delta X_i \,\, \propto \,\, X_i  \quad \textup{ for all $i$,} \\
  R(\delta X_i) + R[s]_{\times} X_i + t^{\perp}_* \,\, \propto \,\, RX_i + t \quad \textup{ for all $i$}.
  \end{cases} 
\end{equation}
Here `$\propto$' indicates a proportionality, $s := \frac{1}{\sqrt{2}} \begin{pmatrix} \delta r_1 & \delta r_2 & \delta r_3 \end{pmatrix}^{\!\top} \in \mathbb{R}^3$ and $t^{\perp}_{*} := \delta t_1 t^{\perp}_1 + \delta t_2 t^{\perp}_2 \in \mathbb{R}^3$.  
We need to characterize when \eqref{eq:good-stuff} admits a nonzero solution in the variables $\delta X_1, \ldots, \delta X_5, s, t_*^{\perp}$.

Let $\lambda_i \in \mathbb{R}$ denote the proportionality constants in the first line of \eqref{eq:good-stuff}, and likewise $\mu_i \in \mathbb{R}$  for the second line.
Then the first line of \eqref{eq:good-stuff} reads $\delta X_i = \lambda_i X_i$.  Substituting this into the second line of \eqref{eq:good-stuff} gives 
\begin{equation} \label{eq:real-good}
\lambda_i R X_i + R [s]_{\times} X_i + t^{\perp}_{*} = \mu_iRX_i + \mu_i t \quad \textup{for all $i$}.
\end{equation}
We need to characterize when \eqref{eq:real-good} admits a solution in  $\lambda_1, \ldots, \lambda_5, \mu _1, \ldots, \mu_5, s, t_*^{\perp}$  nonzero  in $\lambda_1, \ldots, \lambda_5, s, t_*^{\perp}$. (Note $\lambda_i \neq 0  \Leftrightarrow   \delta X_i \neq 0$ since  $(X_i) \neq 0$.)
It is the same to ask the solution to \eqref{eq:real-good} be not all-zero in $\lambda_1, \ldots, \lambda_5, \mu_1, \ldots, \mu_5, s, t_*^{\perp}$ (with $\mu$'s included), for if $\lambda_1, \ldots, \lambda_5, s, t$ are all zero then \eqref{eq:real-good} implies $\mu_i = 0$, since $(RX_i + t)_3 \neq 0$. 

We can simplify Eq.~\eqref{eq:real-good} by changing notation as follows:
\begin{equation}\label{eq:change-coords}
\begin{cases}
    \lambda_i  & \longleftarrow  \quad  \lambda_i - \mu_i \\
    \mu_i  & \longleftarrow  \quad \mu_i \\
    t_*^{\perp}  & \longleftarrow  \quad -t_*^{\perp} \\
    [s]_{\times} & \longleftarrow  \quad R[s]_{\times}R^{\top} \\
    X_i  & \longleftarrow  \quad R X_i \\
    R & \longleftarrow \quad I_3.
\end{cases}
\end{equation}
The first four lines in \eqref{eq:change-coords} describe an invertible linear change of variables for \eqref{eq:real-good}.  This does not affect whether there exists a nonzero solution to \eqref{eq:real-good}.
The last lines in \eqref{eq:change-coords} rotate the world points $X_1, \ldots, X_5$, and this operation does not affect whether there exists a quadric surface in $\mathbb{R}^3$ satisfying the claimed condition in Theorem~\ref{thm:illposed-world}.  So indeed, the transformation \eqref{eq:change-coords} is without loss of generality.  In updated notation, \eqref{eq:real-good} reads

\begin{equation} \label{eq:super}
    \lambda_i X_i  \, + \, [s]_{\times} X_i = \mu_i t + t_{*}^{\perp} \quad \textup{for all $i$}.
\end{equation}

Applying a further rotation in $\mathbb{R}^3$, we can assume that $t = e_3$ and $t_1^{\perp} = e_1$ and $t_2^{\perp} = e_2$ in \eqref{eq:super} without loss of generality.
Since $t$ and $t^{\perp}_*$ are perpendicular, we can eliminate $\mu_i$ from \eqref{eq:super}, because it is equivalent to equate the first two coordinates of both sides of \eqref{eq:super}:

\begin{equation} \label{eq:fantastic} 
    \lambda_i \! \begin{pmatrix} (X_i)_1 \\[1pt] (X_i)_2 \end{pmatrix} \,\, + \,\, \begin{pmatrix} -s_3 (X_i)_2 + s_2 (X_i)_3  \\[1pt] s_3 (X_i)_1 - s_1 (X_i)_3  \end{pmatrix} 
    = \begin{pmatrix} \delta t_1 \\[1pt] \delta t_2 \end{pmatrix} \quad \textup{for all $i$}.
\end{equation}
We need to characterize when the system \eqref{eq:fantastic} has a nonzero solution in $\lambda_1, \ldots, \lambda_5, s_1, s_2, s_3, \delta t_1, \delta t_2$.

Rewrite \eqref{eq:fantastic} as follows:
\begin{equation} \label{eq:excellent}
    \lambda_i \! \begin{pmatrix} (X_i)_1 \\[1pt] (X_i)_2 \end{pmatrix} \, + \, s_3 \! \begin{pmatrix} -(X_i)_2 \\[1pt] (X_i)_1 \end{pmatrix} \, + \, (X_i)_3  \begin{pmatrix} s_2  \\[1pt] -s_1 \end{pmatrix} \, - \, \begin{pmatrix} \delta t_1 \\[1pt] \delta t_2 \end{pmatrix} \, = \, 0 \quad \textup{for all $i$}.
\end{equation}
Now eliminate $\lambda_i$ from \eqref{eq:excellent}.
Indeed, we claim that  \eqref{eq:excellent} admits a nonzero solution in $\lambda_1, \ldots, \lambda_5, s_1, s_2, s_3, \delta t_1, \delta t_2$ if and only if the system  obtained by multiplying \eqref{eq:excellent} on the left by $\begin{pmatrix} -(X_i)_2 & (X_i)_1 \end{pmatrix}$ (each $i$) admits a nonzero solution in $s_1, s_2, s_3, \delta t_1, \delta t_2$.  That is, we claim we can reduce to:
\begin{equation} \label{eq:amazing}
    s_3((X_i)_1^2 + (X_i)_2^2) \, - \, s_2 (X_i)_2 (X_i)_3 \, - \, s_1 (X_i)_1 (X_i)_3 \, + \, \delta t_1 (X_i)_2 \, - \, \delta t_2 (X_i)_1 \, = \, 0 \quad \textup{for all $i$}.
\end{equation}
To justify this, note that if $\begin{pmatrix} (X_i)_1 \\ (X_i)_2 \end{pmatrix} \neq 0$ for each $i$, then the vectors 
$\begin{pmatrix} (X_i)_1 \\ (X_i)_2 \end{pmatrix}$ and $\begin{pmatrix} -(X_i)_2 \\ (X_i)_1 \end{pmatrix}$ give an orthogonal basis for $\mathbb{R}^2$ for each $i$.
In this case, changing to this basis from the standard basis,  \eqref{eq:excellent} becomes \eqref{eq:amazing} together with 
\begin{equation} \label{eq:elim-lam}
    \lambda_i ((X_i)_1^2 + (X_i)_2^2) \, + \, s_2 (X_i)_1 (X_i)_3 \, - \, s_1 (X_i)_2 (X_i)_3 \, - \delta t_1 (X_i)_1 \, - \, \delta t_2 (X_i)_2 \, = \, 0 \quad \textup{for all $i$}.
 \end{equation}
Clearly \eqref{eq:elim-lam} determines $\lambda_i$ in terms of $s_1, \ldots, \delta t_2$, so  \eqref{eq:amazing} and \eqref{eq:elim-lam} have a nonzero solution in $\lambda_1, \ldots, \delta t_2$ if and only if \eqref{eq:amazing} does in $s_1, \ldots, \delta t_2$.
Meanwhile, if 
$\begin{pmatrix} (X_i)_1 \\ (X_i)_2 \end{pmatrix} = 0$ for some $i$, then both \eqref{eq:excellent} and \eqref{eq:amazing} admit nonzero solutions: for \eqref{eq:excellent}, we can explicitly set $\lambda_i = 1$ and all other nine variables equal to $0$; for \eqref{eq:amazing}, once we remove the $i$-th equation (which is trivial) this leaves an undetermined linear system of four equations in five unknowns, which must have a nonzero solution. 
Thus, we need to characterize when \eqref{eq:amazing} admits a nonzero solution in $s_1, s_2, s_3, \delta t_1, \delta t_2$.

To complete the proof, we argue that we simply need to geometrically reinterpret  \eqref{eq:amazing}.
Letting $z_1, z_2, z_3$ be variables on $\mathbb{R}^3$, consider the following linear subspace of inhomogeneous quadratic polynomials:
\begin{equation}\label{eq:subspace}
     \operatorname{span}\!\left\{z_1^2 + z_2^2, \, z_2 z_3, \, z_1 z_3, \, z_2, \, z_1 \right\} \,\, \subseteq \,\, \mathbb{R}[z_1, z_2, z_3] 
\end{equation}
Then \eqref{eq:amazing} states that there exists a quadric surface $\mathcal{Q} \subseteq \mathbb{R}^3$, cut out by some nonzero polynomial in \eqref{eq:subspace}, passing through the points $X_1, \ldots, X_5 \in \mathbb{R}^3$.  
However, \eqref{eq:subspace} precisely describes the quadric surfaces in $\mathbb{R}^3$ that contain the baseline $\operatorname{Span}(-R^{\top}t) = \operatorname{Span}(e_3) \subseteq \mathbb{R}^3$, and are such that intersecting the quadric with any affine plane in $\mathbb{R}^3$ which is perpendicular to the baseline results in a circle (with the caveats of Remark~\ref{rem:careful} applying).  Indeed \eqref{eq:subspace} exactly corresponds to the subspace of $4 \times 4$ real symmetric matrices of the following form:
\begin{equation*}
  Q =   \begin{pmatrix}
    q_1 & 0 & q_2 & q_3 \\
    0 & q_1 & q_4 & q_5 \\
    q_2 & q_4 & 0 & 0 \\
    q_3 & q_5 & 0 & 0 
    \end{pmatrix} \quad \textup{for some }q_1, \ldots, q_5 \in \mathbb{R}.
\end{equation*}
Precisely such matrices give quadrics containing  $\operatorname{Span}(\mathbb{R}^3)$ (because of the zero bottom-right $2 \times 2$ submatrix), and also intersecting planes parallel to $\operatorname{Span}(e_1, e_2)$ in circles (because of the top-right $2 \times 2$ submatrix).
This finishes Theorem~\ref{thm:illposed-world}.
\end{proof}

\smallskip

\subsection{Proof of Theorem~\ref{thm:F-ill-posed-world}}
\begin{proof}
This argument is similar to the proof of Theorem~\ref{thm:illposed-world}, although somewhat more computational.
Here the forward map $\Phi$ is given by Eq.~\eqref{eq:F-forward}, and the assumption that $\Phi$ is defined at $w$ implies that the points $X_i$ and $M(b)\begin{pmatrix} X_i \\ 1 \end{pmatrix}$ in $\mathbb{R}^3$ do not have a vanishing third coordinate for each $i= 1, \ldots, 7$.
Let $\Delta := \begin{pmatrix} \delta X_1 & \ldots & \delta X_7 & \delta b_1 & \ldots & \delta b_7 \end{pmatrix}^{\!\top} \in \mathbb{R}^{28}$.  
Our task is characterize for which scenes $w$ does there a nonzero solution to the linear system $D\Phi(w) \Delta = 0$, where the variable is $\Delta$.

Using the factorization $D\Phi = D\Phi_2 \circ D\Phi_1$ from the previous section, the explicit Jacobian matrix expressions \eqref{eq:large-displaye} and \eqref{eqn:jac-2-second}, and  Remark~\ref{rem:kernel} characterizing the kernel of $D\Phi_2$, we equivalently have the system 
\begin{equation} \label{eq:F-proportion}
    \begin{cases}
    \delta X_i \,\, \propto \,\, X_i \quad \textup{for all } i \\[8pt]
    \begin{pmatrix} 
    1 & b_1 & b_2 \\
    b_4 & b_5 & b_6 \\
    0 & 0 & 0
    \end{pmatrix} \delta X_i \, + \, \begin{pmatrix} 
    0 & \delta b_1 & \delta b_2 & \delta b_3 \\
    \delta b_4 & \delta b_5 & \delta b_6 & \delta b_7 \\
    0 & 0 & 0 & 0
    \end{pmatrix} \begin{pmatrix} 
    X_i \\
    1
    \end{pmatrix} \,\, \propto \,\, \begin{pmatrix} 1 & b_1 & b_2 & b_3 \\
    b_4 & b_5 & b_6 & b_7 \\
    0 & 0 & 0 & 1
    \end{pmatrix} \begin{pmatrix} 
    X_i \\
    1
    \end{pmatrix} \quad \textup{for all } i.
    \end{cases}
\end{equation}
Comparing the third coordinate of both sides, we see that the proportionality constant in the second line of \eqref{eq:F-proportion} must be $0$ for each $i$.
Let $\lambda_i \in \mathbb{R}$ be the proportionality constant in the first line of \eqref{eq:F-proportion} for each $i$.  
Rewrite \eqref{eq:F-proportion} as 
\begin{equation} \label{eq:on-track}
    \begin{cases}
    \delta X_i \,\, = \,\, \lambda_i X_i \quad \textup{for all $i$}, \\[8pt]
    \begin{pmatrix}
    1 & b_1 & b_2 \\
    b_4 & b_5 & b_6
    \end{pmatrix} \delta X_i  \, + \, \begin{pmatrix} 
    0 & \delta b_1 & \delta b_2 & \delta b_3 \\
    \delta b_4 & \delta b_5 & \delta b_6 & \delta b_7 
    \end{pmatrix} \begin{pmatrix} 
    X_i \\
    1
    \end{pmatrix} \,\, = \,\, 0 \quad \textup{for all $i$}.
    \end{cases}
\end{equation}
In \eqref{eq:on-track}, we substitute the first line into the first term of the second line and we rearrange the second term in the second line:
\begin{equation} \label{eq:awesome}
   \lambda_i \begin{pmatrix}
    1 & b_1 & b_2 \\
    b_4 & b_5 & b_6
    \end{pmatrix}  X_i  \, + \, \begin{pmatrix} (X_i)_2 & (X_i)_3 & 1 & 0 & 0 & 0 & 0 \\
    0 & 0 & 0 & (X_i)_1 & (X_i)_2 & (X_i)_3 & 1
    \end{pmatrix} \delta b \,\, = \,\, 0 \quad \textup{for all $i$}.
\end{equation}
We need to characterize when the system \eqref{eq:awesome} has a nonzero solution in $\lambda_1, \ldots, \lambda_7, \delta b_1, \ldots, \delta b_7$.
(That this is equivalent to the characterization for \eqref{eq:F-proportion} uses that $\delta X_i \neq 0 \Leftrightarrow \lambda_i \neq 0$, because $\delta X_i = \lambda_i X_i$ and $(X_i)_3 \neq 0$.)

Now we eliminate $\lambda_i$ from \eqref{eq:awesome}, following what we did for \eqref{eq:excellent} above.  
Here it is equivalent to multiply \eqref{eq:awesome} on the left by the transpose of
\begin{equation*} \label{eq:bX-perp}
    \begin{pmatrix}
    -b_4(X_i)_1 - b_5(X_i)_2 - b_6(X_i)_3 \\
    (X_i)_1 + b_1(X_i)_2 + b_2(X_i)_3
    \end{pmatrix},
\end{equation*}
which is normal to $\begin{pmatrix} 1 & b_1 & b_2 \\ b_4 & b_5 & b_6 \end{pmatrix} X_i$.  
Then we need to characterize when the resulting $7 \times 7$ linear system has a nonzero solution in  $\delta b_1, \ldots, \delta b_7$.
So we reduce to:
\begin{equation} \label{eq:i-like-this}
    \begin{pmatrix}
    -b_4(X_i)_1(X_i)_2 - b_5(X_i)_2^2 - b_6(X_i)_2(X_i)_3 \\[2pt]
    -b_4(X_i)_1(X_i)_3 - b_5(X_i)_2(X_i)_3 - b_6(X_i)_3^2 \\[2pt]
    -b_4(X_i)_1 - b_5(X_i)_2 - b_6(X_i)_3 \\[2pt]
    (X_i)_1^2 + b_1(X_i)_1 (X_i)_2 + b_2(X_i)_1 (X_i)_3 \\[2pt]
    (X_i)_1(X_i)_2 + b_1(X_i)_2^2 + b_2(X_i)_2(X_i)_3 \\[2pt]
    (X_i)_1(X_i)_3 + b_1(X_i)_2(X_i)_3 + b_2(X_i)_3^2 \\[2pt]
    (X_i)_1 + b_1(X_i)_2 + b_2(X_i)_3
    \end{pmatrix}^{\!\! \top} \,\, \begin{pmatrix} 
    \delta b_1 \\[2pt]
    \delta b_2 \\[2pt]
    \delta b_3 \\[2pt]
    \delta b_4 \\[2pt]
    \delta b_5 \\[2pt]
    \delta b_6 \\[2pt]
    \delta b_7
    \end{pmatrix} \,\, = \,\, 0 \quad \textup{for each $i$}.
\end{equation}

To complete the proof, we only need to reinterpret \eqref{eq:i-like-this} geometrically. 
Let $z_1, z_2, z_3$ be variables on $\mathbb{R}^3$.  Then \eqref{eq:i-like-this} states that there exists a quadric surface $\mathcal{Q} \subseteq \mathbb{R}^3$ (with the caveats of Remark~\ref{rem:careful}), passing through  $X_1, \ldots, X_7 \in \mathbb{R}^3$ and cut out by a nonzero element of the following subspace of quadratic polynomials:
\begin{footnotesize}
\begin{align} \label{eq:F-subspace}
    \operatorname{span}\{  b_4z_1z_2 + b_5z_2^2 + b_6 z_2 z_3, 
     b_4z_1z_3 + b_5z_2z_3 + b_6 z_3^2, 
     b_4 z_1 + b_5 z_2 + b_6 z_3, 
     z_1^2 + b_1z_1z_2 + b_2 z_1 z_3, & \nonumber \\
     z_1z_2 + b_2 z_2^2 + b_2 z_2 z_3, 
     z_1z_3 + b_1z_2z_3 + b_2 z_3^2, 
     z_1 + b_1 z_2 + b_2 z_3 
    \} & \quad \subseteq \quad \mathbb{R}[z_1,z_2,z_3].
\end{align}
\end{footnotesize}

\noindent We just need to verify  the subspace \eqref{eq:F-subspace} consists precisely of the inhomogeneous quadratic polynomials vanishing on all of the baseline.  
Let check this by direct calculation over the next three paragraphs.

The center of the first camera $\begin{pmatrix} I_3 & 0 \end{pmatrix}$ is $\begin{pmatrix} 0 & 0 & 0 \end{pmatrix}^{\top} \in \mathbb{R}^3$.  By Cramer's rule, the center of the second camera $M(b)$ is the following point at infinity:
\begin{align*}
& \begin{pmatrix}
-\det \begin{pmatrix} b_1 & b_2 & b_3 \\ b_5 & b_6 & b_7 \\ 0 & 0 & 1 \end{pmatrix} & 
\det \begin{pmatrix} 1 & b_2 & b_3 \\ b_4 & b_6 & b_7 \\ 0 & 0 & 1 \end{pmatrix} & 
- \det \begin{pmatrix} 1 & b_1 & b_3 \\ b_4 & b_5 & b_7 \\ 0 & 0 & 1 \end{pmatrix} & 
\det \begin{pmatrix} 1 & b_1 & b_2 \\ b_4 & b_5 & b_6 \\ 0 & 0 & 0 \end{pmatrix}
\end{pmatrix}^{\!\! \top} \nonumber \\[8pt]
& \quad \quad \quad  \quad = \,\, \begin{pmatrix} 
b_2b_5 - b_1b_6 & -b_2b_4 + b_6 & b_1b_4 - b_5 & 0
\end{pmatrix}^{\!\top} \quad \in \quad \mathbb{P}^3.
\end{align*}
Thus the baseline is
\begin{equation} \label{eq:baseline}
    \{ \lambda \begin{pmatrix} b_2b_5 - b_1b_6 &
    -b_2b_4 + b_6 &
    b_1b_4 - b_5
    \end{pmatrix}^{\! \top}  : \lambda \in \mathbb{R} \} \quad \subseteq \quad \mathbb{R}^3.
\end{equation}
Substituting \eqref{eq:baseline} into  \eqref{eq:F-subspace} shows that all these polynomials indeed vanish identically on the baseline.

Next, note that the seven polynomials in \eqref{eq:F-subspace} are linearly independent in $\mathbb{R}[z_1, z_2, z_3]$.
Indeed, suppose $\alpha \in \mathbb{R}^7$ satisfies
\begin{multline} \label{eq:alpha-cond}
    \alpha_1 \left(b_4z_1z_2 + b_5z_2^2 + b_6 z_2 z_3\right) + 
    \alpha_2 \left( b_4z_1z_3 + b_5z_2z_3 + b_6 z_3^2 \right) +
    \alpha_3 \left(b_4 z_1 + b_5 z_2 + b_6 z_3\right) +
    \alpha_4 \left(z_1^2 + b_1z_1z_2 + b_2 z_1 z_3\right)  \\[1.5pt] +
    \alpha_5 \left(z_1z_2 + b_2 z_2^2 + b_2 z_2 z_3\right) + 
    \alpha_6 \left( z_1z_3 + b_1z_2z_3 + b_2 z_3^2 \right) +
    \alpha_7 \left(z_1 + b_1 z_2 + b_2 z_3 \right) \,\, = \,\, 0  \quad \in \quad \mathbb{R}[z_1, z_2, z_3].
\end{multline}
From the coefficient of $z_1^2$, we see $\alpha_4 = 0$.
Since we are assuming that $M(b)$ has rank $3$, it follows that $\begin{pmatrix} 1 & b_1 & b_2 \\ b_4 & b_5 & b_6 \end{pmatrix}$ has rank $2$. 
So implies that 
third and seventh polynomials in \eqref{eq:alpha-cond} are linearly independent, and since their monomial support is disjoint from that of the other polynomials in \eqref{eq:alpha-cond}, we have $\alpha_3 = \alpha_7 = 0$.
This leaves the first, second, fifth and sixth polynomials in \eqref{eq:alpha-cond}.
Writing out what remains in terms of the monomials $z_1z_2, z_1z_3, z_2^2, z_2z_3, z_3^2$ gives
\begin{equation} \label{eq:coeff-mat}
    \begin{pmatrix}
    b_4 & 0 & 1 & 0 \\[1pt]
    0 & b_4 & 0 & 1 \\[1pt]
    b_5 & 0 & b_1 & 0 \\[1pt]
    b_6 & b_5 & b_2 & b_1 \\[1pt]
    0 & b_6 & 0 & b_2
    \end{pmatrix} 
    \begin{pmatrix} 
    \alpha_1 \\[1pt]
    \alpha_2 \\[1pt]
    \alpha_5 \\[1pt]
    \alpha_6
    \end{pmatrix} \quad = \quad 0.
\end{equation}
Actually, assumption that $\operatorname{rank}(M(b)) = 3$  implies that the coefficient matrix in \eqref{eq:coeff-mat} has rank $4$. 
Indeed, one verifies using computer algebra, \eg Macaulay2 \cite{M2}, that in the ring $\mathbb{R}[z_1,z_2,z_3]$ the radical of the ideal generated by the $4 \times 4$ minors of the matrix in \eqref{eq:coeff-mat} equals the ideal generated by the $3 \times 3$ minors of $M(b)$.  This forces $\alpha_1 = \alpha_2 = \alpha_5 = \alpha_6 = 0$.

Last, notice that requiring a quadric in $\mathbb{R}^3$ to contain a given line is a codimension $3$ condition on the quadric.  Indeed by projective symmetry, the codimension is independent of the specific choice of fixed line; and if we choose the $z_3$-axis, this amounts to requiring the vanishing of the bottom-left $2 \times 2$ submatrix of the quadric's $4 \times 4$ symmetric coefficient matrix.

Combining the last three pagragraphs,  \eqref{eq:F-subspace} consists of the  quadratic polynomials vanishing on the baseline as desired.
 \end{proof}

\medskip
\medskip

\section{Proofs for ``Section~\ref{sec:ImageData}: Ill-Posed Image Data"}
\label{sec:sec4_3}
In this section, we describe the locus of ill-posed image data for the 5-point and 7-point minimal problems in terms of the X.5-point curves.  
The logic is to use the classical epipolar relations in multiview geometry to relate these minimal problems to the task of intersecting a fixed complex projective algebraic variety with a varying linear subspace of complementary dimension.  
Then we apply tools from computational algebraic geometry, which were developed to analyze this task \cite{sturmfels2017hurwitz, burgisser2017condition}.

\smallskip

\subsection{Background from algebraic geometry}
Consider complex projective space $\mathbb{P}^n_{\mathbb{C}}$.
The set of subspaces of $\mathbb{P}^n_{\mathbb{C}}$ of codimension $d$ is naturally an irreducible projective algebraic variety, called the \nolinebreak \textit{Grassmannian}:
\begin{equation*}
\operatorname{Gr}(\mathbb{P}^{n-d}_{\mathbb{C}}, \mathbb{P}^{n}_{\mathbb{C}}) = \left\{ L \subseteq \mathbb{P}_{\mathbb{C}}^n : \dim(L) = n-d \right\}\!.
\end{equation*}
We use two classic coordinate systems for the Grassmannian.  If a point $L \in \operatorname{Gr}(\mathbb{P}^{n-d}_{\mathbb{C}}, \mathbb{P}^{n}_{\mathbb{C}})$ is written as the kernel of a full-rank matrix $M \in \mathbb{C}^{d \times (n+1)}$, then the \textit{primal Pl\"ucker coordinates} for $L$ are defined to be the maximal minors of $M$:
\begin{equation} \label{eq:primal-plucker-def}
p(L) \, = \, \left( p_{\mathcal{I}}(L) = \det(M(:,\mathcal{I})) \, : \, \mathcal{I} \in \binom{[n+1]}{d} \right)
\end{equation}
This gives a well-defined point in $\mathbb{P}_{\mathbb{C}}^{\binom{n+1}{d}-1}$ independent of the choice of $M$.
Meanwhile, if we write $L$ as the row span of a full-rank matrix $N \in \mathbb{C}^{(n-d+1) \times (n+1)}$, then the \textit{dual Pl\"ucker coordinates} for $L$ are defined to be the maximal minors of $N$:
\begin{equation} \label{eq:dual-plucker}
q(L) \, = \, \left( q_{\mathcal{J}}(L) = \det(N(:,\mathcal{J})) \, : \, \mathcal{J} \in \binom{[n+1]}{n+1-d} \right).
\end{equation}
Again this gives a well-defined point $\mathbb{P}_{\mathbb{C}}^{\binom{n+1}{d}-1}$ independent of the choice of $N$.  
The primal and dual coordinates agree up to permutation and sign flips, namely for each $L \in \operatorname{Gr}(\mathbb{P}_{\mathbb{C}}^{n-d}, \mathbb{P}_{\mathbb{C}}^{n})$ it holds 
\begin{equation} \label{eq:switchplucker}
 \left( p_{\mathcal{I}}(L) : \mathcal{I} \in \binom{[n+1]}{d} \right) = \left( (-1)^{n+|\mathcal{I}|}q_{[n+1] \setminus \mathcal{I}}(L) : \mathcal{I} \in \binom{[n+1]}{d}\right),
\end{equation}
where $|\mathcal{I}| := \sum_{i \in \mathcal{I}} i$.

Next let $X \subseteq \mathbb{P}_{\mathbb{C}}^n$ be an irreducible complex projective algebraic variety  of dimension $d$.  
There exists a positive integer $p$, called the \textup{degree} of $X$, such that for Zariski-generic subspaces of complementary dimension, $L \in \operatorname{Gr}(\mathbb{P}_{\mathbb{C}}^{n-d}, \mathbb{P}_{\mathbb{C}}^n)$, the intersection of $L \cap X$ consists  precisely $p$ reduced (complex) intersection points.  
Here, one says that an intersection point $x \in L \cap X$ is \textup{reduced} if $L \cap T_{x}X$ consists of one point, where $T_{x}X$ is the Zariski tangent space to $X$ at $x$ given by 
\begin{equation*}
T_x X \,\, =  \,\, \left\{ v \in \mathbb{P}_{\mathbb{C}}^n \, : \, \left( \frac{\partial f_i(x)}{\partial x_j} \right)_{\substack{i=1, \ldots, t \\ j = 0, \ldots, n}} \, v \,= \, 0 \right\} \,\, \subseteq \,\, \mathbb{P}_{\mathbb{C}}^n
\end{equation*}
for generators $f_1, \ldots, f_t \in \mathbb{C}[x_0, \ldots, x_n]$ of the prime ideal of $X$.

The \textit{Hurwitz form} of $X$ is  defined to be the set of linear subspaces which are exceptional with respect to the property in the preceding paragraph.  More precisely, it is
\begin{equation*}
\mathcal{H}_X \, = \, \left\{ L \in \operatorname{Gr}(\mathbb{P}_{\mathbb{C}}^{n-d}, \mathbb{P}_{\mathbb{C}}^{n})  \, : \, L \cap X  \textup{ does not consist of } p \textup{ reduced intersection points}\right\} \, \subseteq \, \operatorname{Gr}(\mathbb{P}_{\mathbb{C}}^{n-d}, \mathbb{P}_{\mathbb{C}}^n).
\end{equation*}
We will use the following result in the proofs of  Theorems~\ref{thm:image-E} and \ref{thm:F-image}.
\begin{theorem}  \label{thm:sturmfels}
\!\! \textup{\cite[Thm.~1.1]{sturmfels2017hurwitz}}
Let $X$ be an irreducible subvariety of $\mathbb{P}_{\mathbb{C}}^n$ with dimension $d$, degree $p$ and sectional genus $g$.
Assume that $X$ is not a linear subspace. 
Then $\mathcal{H}_X$ is an irreducible hypersurface in  $\operatorname{Gr}(\mathbb{P}_{\mathbb{C}}^{n-d}, \mathbb{P}_{\mathbb{C}}^n)$, and there exists a homogeneous polynomial $\operatorname{Hu}_X$ in the (primal) Pl\"ucker coordinates for $L \in \operatorname{Gr}(\mathbb{P}_{\mathbb{C}}^{n-d}, \mathbb{P}_{\mathbb{C}}^n)$ such that 
\begin{equation*}
L \in \mathcal{H}_X  \quad \Leftrightarrow \quad \operatorname{Hu}_X\left( p(L) \right) = 0.  
\end{equation*}
Further if the singular locus of $X$ has codimension at least $2$, then the degree of  $\operatorname{Hu}_X$ in Pl\"ucker coordinates is $2p + 2g - 2$.
\end{theorem}

\smallskip

\subsection{Proof of Theorem~\ref{thm:image-E}}
\begin{proof}
Let $\mathcal{E}_{\mathbb{C}} \subseteq \mathbb{P}_{\mathbb{C}}^8$ be the Zariski closure of the set of real essential matrices $\mathcal{E}$ inside complex projective space.  
It is known that $\mathcal{E}_{\mathbb{C}}$  is an irreducible complex projective variety, and its prime ideal is minimally generated by the ten cubic polynomials in~\eqref{eq:essential-ideal}.
By a computer algebra calculation, $\mathcal{E}_{\mathbb{C}}$ has complex dimension $d=5$, degree $p=10$ and sectional genus $g=6$.
By \cite[Prop.~2(i)]{floystad2018chow}, the singular locus of $\mathcal{E}_{\mathbb{C}}$ is a surface  isomorphic to $\mathbb{P}_{\mathbb{C}}^1 \times \mathbb{P}_{\mathbb{C}}^1$ with no real points, and in particular has codimension $3$ in $\mathcal{E}_{\mathbb{C}}$.
Therefore Theorem~\ref{thm:sturmfels} applies, and tells us that the Hurwitz form  $\mathcal{H}_{\mathcal{E}_{\mathbb{C}}}$ is a hypersurface in the Grassmannian $\operatorname{Gr}(\mathbb{P}_{\mathbb{C}}^3, \mathbb{P}_{\mathbb{C}}^8)$ cut out by a polynomial $\operatorname{Hu}_{\mathcal{E}_{\mathbb{C}}}$ which is degree $2 \cdot 10 + 2 \cdot 6 - 2 = 30$ in Pl\"ucker coordinates.

For each $x = \left( (x_1, y_1), \ldots, (x_5, y_5) \right) \in (\mathbb{R}^2 \times \mathbb{R}^2)^{\times 5}$, we define the subspace
\begin{align} \label{eq:big-59-matrix}
 L(x)  \,:= \, 
\operatorname{kernel} \! \begin{pmatrix} 
(x_1)_1(y_1)_1 & (x_1)_2(y_1)_1 & (y_1)_1 & (x_1)_1(y_1)_2 & (x_1)_2(y_1)_2 & (y_1)_2 & (x_1)_1 & (x_1)_2 & 1 \\
& & & & & & & & \\
\vdots & \vdots & \vdots & \vdots & \vdots & \vdots & \vdots & \vdots & \vdots \\
& & & & & & & & \\
(x_5)_1(y_5)_1 & (x_5)_2(y_5)_1 & (y_5)_1 & (x_5)_1(y_5)_2 & (x_5)_2(y_5)_2 & (y_5)_2 & (x_5)_1 & (x_5)_2 & 1 
    \end{pmatrix}_{\!\! 5 \times 9}  \subseteq \, \mathbb{P}_{\mathbb{C}}^8.
\end{align}

Then we define $\mathbf{P}$ as follows:
\begin{equation*}
    \mathbf{P}(x_1,  \ldots, y_5) \, := \, \operatorname{Hu}_{\mathcal{E}_{\mathbb{C}}}(p(L(x))),
\end{equation*}
where $p(L(x))$ are the primal Pl\"ucker coordinates of $L(x)$.
In other words, $\mathbf{P}$ is obtained by substituting the $\binom{9}{5} = 126$ maximal minors of the matrix in \eqref{eq:big-59-matrix} into $\operatorname{Hu}_{\mathcal{E}_{\mathbb{C}}}$.  
Note $\bf{P}$ has degree $30$ separately in each of the ten points $x_1, \ldots, y_5$, because $\operatorname{Hu}_{\mathcal{E}_{\mathbb{C}}}$ has degree $30$ in Pl\"ucker coordinates and each Pl\"ucker coordinate is separately linear in each $x_i$ and each $y_i$.

We shall verify $\bf{P}$ has the property in the third sentence of the theorem statement.
For the remainder of the proof, fix image data $x = \left((x_1, y_1), \ldots, (x_5, y_5) \right) \in (\mathbb{R}^2 \times \mathbb{R}^2)^{\times 5}$ such that $\mathbf{P}(x_1, \ldots, y_5) \neq 0$.
We need to show that at every world scene that is compatible with $x$ the forward Jacobian is invertible.

We show this using the epipolar constraints for two-view geometry \cite[Part II]{hartleyzisserman}.
We consider the following system in $w$:
\begin{equation}
\begin{cases}  \label{eq:w-system}
\Phi(w) = x \\
w \in \mathcal{W},
\end{cases}
\end{equation}
and the following system in $E$:
\begin{equation} 
\begin{cases} \label{eq:epi-E}
  \begin{pmatrix} y_i \\ 1 \end{pmatrix}^{\!\!\top} \!\! E \begin{pmatrix} x_i \\ 1 \end{pmatrix} \, = \, 0 \quad \forall \, i = 1, \ldots, 5 \\[4pt]
   E \in \mathcal{E}.
  \end{cases}
\end{equation}
Each solution to \eqref{eq:epi-E} corresponds to four solutions to \eqref{eq:w-system} via $\Psi$, and there are no other solutions to \eqref{eq:w-system}.  Moreover $w$ depends smoothly on $E$, see \cite[Result~9.19]{hartleyzisserman}.

However solutions to \eqref{eq:epi-E} are the real intersection points in
\begin{equation} \label{eq:slice}
L(x) \cap \mathcal{E}_{\mathbb{C}} \, \subseteq \, \mathbb{P}_{\mathbb{C}}^8,
\end{equation}
because $\mathcal{E} = \mathcal{E}_{\mathbb{C}}  \cap \mathbb{P}_{\mathbb{R}}^8$ (see \cite[Sec.~2.1]{floystad2018chow}).
But we know \eqref{eq:slice} consists of $10$ reduced intersection points, by definition of ${\bf P}$ and the Hurwitz form.
Denote these $\{E_1, \ldots,  E_{10} \}$ where the real intersection points are $E_1, \ldots, E_a$.
Using an appropriate version of the implicit function theorem (see \cite[App.~A]{sommese2005numerical}), there exists an open neighborhood $\mathcal{U}$ of $x$ in $\mathcal{X}$ and differentiable functions $\widetilde{E}_1, \ldots, \widetilde{E}_{10} : \mathcal{U} \rightarrow \mathcal{E}_{\mathbb{C}}$ such that: \textit{(i)} $\widetilde{E}_i(x) = E_i$ for each $i = 1, \ldots, 10$; \textit{(ii)}  for each $x' \in \mathcal{U}$ we have $L(x') \cap \mathcal{E}_{\mathbb{C}} = \{\widetilde{E}_1(x'), \ldots, \widetilde{E}_{10}(x')\}$, these intersection points are all reduced, and the only the first $a$ points are real.

Combining the last two paragraphs, there exists differentiable functions $\widetilde{w}_1, \ldots, \widetilde{w}_{4a} : \mathcal{U} \rightarrow \mathcal{W}$ such that for each $x' \in \mathcal{U}$,
\begin{equation}
\left\{ w \in \mathcal{W} : \Phi(w) = x' \right\} = \left\{ \widetilde{w}_1(x'), \ldots, \widetilde{w}_{4a}(x') \right\}.
\end{equation}
Therefore $\Phi \circ \widetilde{w}_i = \operatorname{id}_{\mathcal{U}}$ for each $i$.  
Differentiating this and evaluating at $x \in \mathcal{X}$ gives
\begin{equation*}
(D\Phi)(\widetilde{w}_i(x)) \circ (D\widetilde{w}_i)(x) = I_{20},
\end{equation*}
so that, in particular, $(D\Phi)(\widetilde{w}_i(x))$ is invertible for each $i$. 

This proves that every world scene compatible with $x$ has an invertible forward Jacobian as we needed.
For a discussion of how to plot the $4.5$-point curve using homotopy continuation, see ``Numerical Computation of the X.5-Point Curves".
\end{proof}

\smallskip

\subsection{Proof of Theorem~\ref{thm:F-image}}

\begin{proof}
This is similar to the proof of Theorem~\ref{thm:image-E}, although somewhat easier.
Let $\mathcal{F}_{\mathbb{C}} \subseteq \mathbb{P}_{\mathbb{C}}^8$ be the Zariski closure of the set of real fundamental matrices $\mathcal{F}$ inside complex projective space. 
Then $\mathcal{F}_{\mathbb{C}}$ consists of all rank-deficient $3 \times 3$ matrices.  
So $\mathcal{F}_{\mathbb{C}}$ is an irreducible complex projective hypersurface, defined by the determinantal cubic equation.
It has dimension $d=7$, degree $p=3$ and sectional genus $g = 1$.  The singular locus of $\mathcal{F}_{\mathbb{C}}$ consists of all rank $1$ matrices, and in particular has codimension $3$.
Therefore Theorem~\ref{thm:sturmfels} applies, and tells us that the Hurwitz form  $\mathcal{H}_{\mathcal{F}_{\mathbb{C}}}$ is a hypersurface in the Grassmannian $\operatorname{Gr}(\mathbb{P}_{\mathbb{C}}^1, \mathbb{P}_{\mathbb{C}}^8)$ cut out by a polynomial $\operatorname{Hu}_{\mathcal{F}_{\mathbb{C}}}$ which is degree $2 \cdot 3 + 2 \cdot 1 - 2 = 6$ in Pl\"ucker coordinates.

For each $x = \left( (x_1, y_1), \ldots, (x_7, y_7) \right) \in (\mathbb{R}^2 \times \mathbb{R}^2)^{\times 7}$, we define the subspace
\begin{align} \label{eq:big-79-matrix}
 L(x)  \,:= \, 
\operatorname{kernel} \! \begin{pmatrix} 
(x_1)_1(y_1)_1 & (x_1)_2(y_1)_1 & (y_1)_1 & (x_1)_1(y_1)_2 & (x_1)_2(y_1)_2 & (y_1)_2 & (x_1)_1 & (x_1)_2 & 1 \\
& & & & & & & & \\
\vdots & \vdots & \vdots & \vdots & \vdots & \vdots & \vdots & \vdots & \vdots \\
& & & & & & & & \\
(x_7)_1(y_7)_1 & (x_7)_2(y_7)_1 & (y_7)_1 & (x_7)_1(y_7)_2 & (x_7)_2(y_7)_2 & (y_7)_2 & (x_7)_1 & (x_7)_2 & 1 
    \end{pmatrix}_{\!\! 7 \times 9}  \subseteq \, \mathbb{P}_{\mathbb{C}}^8.
\end{align}

Then we define $\mathbf{P}$ as follows:
\begin{equation*}
    \mathbf{P}(x_1,  \ldots, y_7) \, := \, \operatorname{Hu}_{\mathcal{F}_{\mathbb{C}}}(p(L(x))),
\end{equation*}
where $p(L(x))$ are the primal Pl\"ucker coordinates of $L(x)$.
In other words, $\mathbf{P}$ is obtained by substituting the $\binom{9}{7} = 36$ maximal minors of the matrix in \eqref{eq:big-79-matrix} into $\operatorname{Hu}_{\mathcal{F}_{\mathbb{C}}}$.  
Note $\bf{P}$ has degree $6$ separately in each of the fourteen points $x_1, \ldots, y_7$, because $\operatorname{Hu}_{\mathcal{F}_{\mathbb{C}}}$ has degree $6$ in Pl\"ucker coordinates and each Pl\"ucker coordinate is separately linear in each $x_i$ and each $y_i$.

The argument that ${\bf P}$ does the job is analogous to that for the calibrated case.
One uses the epipolar constraints for the fundamental matrix, and the correspondence between fundamental matrices and world scenes (which this time is  1-1 due to \cite[Sec.~9.5.2]{hartleyzisserman}).  
Given $x \in \mathcal{X}$ such that ${\bf P}(x) \neq 0$, each compatible fundamental matrix $F \in \mathcal{F}$ is a locally defined smooth function of $x$, by  our choice of ${\bf P}$ and the definition of the Hurwitz form. 
The corresponding world scene $w \in \mathcal{W}$ is smooth as a function of $F$, and therefore a locally defined smooth function of $x$.
Then we conclude with the chain rule again.

In the uncalibrated case, we can explicitly compute the Hurwitz form as follows.  Recall the definition of dual Pl\"ucker coordinates \eqref{eq:dual-plucker}.
Applying elementary row operations, each $L \in \operatorname{Gr}(\mathbb{P}_{\mathbb{C}}^1, \mathbb{P}_{\mathbb{C}}^8)$ with $q_{11,11} \neq 0$ can be written as
\begin{equation*}
L = \operatorname{rowspan} \begin{pmatrix} 1 & 0 & -q_{12,13} & -q_{12,21} & -q_{12,22} & -q_{12,23} & -q_{12,31} & -q_{12,32} & -q_{12,33} \\
0 & 1 & q_{11,13} & q_{11,21} & q_{11,22} & q_{11,23} & q_{11,31} & q_{11,32} & q_{11,33} 
\end{pmatrix},
\end{equation*}
or, upon un-vectorizing, 
\begin{small}
\begin{equation} \label{eq:eq-F1F2}
    L = \operatorname{span} \left\{ F_1, F_2 \right\} \quad \textup{where} \,\,\, F_1 = \begin{pmatrix} 1 & 0 & -q_{12,13} \\ -q_{12,21} & -q_{12,22} & -q_{12,23} \\ -q_{12,31} & -q_{12,32} & -q_{12,33} \end{pmatrix} \,\,\, \textup{and} \,\,\, F_2 = \begin{pmatrix} 0 & 1 & q_{11,13} \\ q_{11,21} & q_{11,22} & q_{11,23} \\ q_{11,31} & q_{11,32} & q_{11,33} \end{pmatrix}.
\end{equation} 
\end{small}
Let $F(t) = F_1 + tF_2$ where $t \in \mathbb{C}$. 
Assume $\det(F_2) \neq 0$.  Then   $L \cap \mathcal{F}_{\mathbb{C}} = \left\{ F(t) : \det(F(t)) = 0 \right\}$
and 
\begin{align*}
  L \in \mathcal{H}_{\mathcal{F}_{\mathbb{C}}} \quad \quad \Leftrightarrow  \quad \quad \det(F(t)) = 0 \textup{ has a double root in } t \quad \quad \Leftrightarrow \quad \quad \operatorname{disc}_t(\det(F(t))) = 0.
\end{align*}
Here $\operatorname{disc}_t$ denotes the usual discriminant for a cubic polynomial in $t$, that is, for $a,b,c,d \in \mathbb{C}$ it is
\begin{equation*}
\operatorname{disc}_t(at^3 + bt^2 + ct + d) = b^2c^2 - 4ac^3 - 4b^3d - 27a^2d^2 + 18abcd.
\end{equation*}

Thus we compute $\operatorname{Hu}_{\mathcal{F}_{\mathbb{C}}}$ by substituting \eqref{eq:eq-F1F2} into the definition of $F(t)$ and then expanding $\operatorname{disc}_t(\det(F(t))$.
We switch from dual Pl\"ucker coordinates to primal Pl\"ucker coordinates using \eqref{eq:switchplucker}.  
Then we reduce the resulting polynomial by the Pl\"ucker relations using Gr\"obner bases.
We carried this calculation out in Macaulay2, and it terminated quickly.  
The result is an explicit expression for $\operatorname{Hu}_{\mathcal{F}_{\mathbb{C}}}$ as a homogeneous degree-$6$ polynomial in the $36$ Pl\"ucker coordinates $p_{11,11}, \ldots, p_{33,33}$ with 1668 terms and integer coefficients with a maximal absolute value of $72$.
The polynomial is posted on our GitHub repository.

Last, consider the $6.5$-point curve. 
Given thirteen specified image points $x_1, y_1, \ldots, x_7 \in \mathbb{R}^2$, we can evaluate the primal Pl\"ucker coordinates \eqref{eq:primal-plucker-def} where only $y_7$ is symbolic.  
Substituting these into the formula for $\operatorname{Hu}_{\mathcal{F}_{\mathbb{C}}}$ gives the defining equation for the $6.5$-point curve as a polynomial in $y_7$. Alternatively the curve can be computed numerically, see the next section.
\end{proof}

\smallskip

\subsection{Numerical Computation of the X.5-Point Curves} \label{subsec:compute-x5}
In the body of the paper, we specified that for the 5-point problem and 7-point problem, if we fix ``4.5" and ``6.5" correspondences as described in the  body, then we can find a degree 30 and 6 curve on the second image indicating the ill-posed positions for the last image point. 
In this section, the generation of the \textbf{X.5-point curves} will be introduced in depth.

Suppose we have two images $I$ and $\bar{I}$.  Consider the essential matrix $E$ and fundamental matrix $F$ representing their relative pose in calibrated and uncalibrated cases respectively. 

For the \textbf{uncalibrated case}, to generate the 6.5-point curve, six known point correspondences $(x^{\top}_i = [(x_i)_1, (x_i)_2], y^{\top}_i = [(y_i)_1, (y_i)_2]), i = 1, \ldots, 6$, are needed. 
These point correspondences satisfy
\begin{align*}
    \begin{pmatrix} y_i \\ 1 \end{pmatrix}^{\!\! \top} \!\! F \begin{pmatrix} x_i \\ 1 \end{pmatrix} = 0, \qquad i = 1, \ldots, 6.
\end{align*}
This equation can then be rewritten as a linear system:
\begin{align}
\left\{ \begin{matrix}
f^{\top} = \left(F_{11}, F_{12}, F_{13}, F_{21}, F_{22}, F_{23}, F_{31}, F_{32}, F_{33}\right) \\[1pt]
    w^T_i f = 0, \qquad  i = 1, \ldots, 6 \\[1pt]
    w_i^{\top} = \left((x_i)_1 ({y}_i)_1, (x_i)_2 ({y}_i)_1, ({y}_i)_1, (x_i)_1 ({y}_i)_2, (x_i)_2 ({y}_i)_2, ({y}_i)_2, (x_i)_1, (x_i)_2, 1\right).	
    \end{matrix}
\label{Eq:linearSystem}
\right. 
\end{align}
Note that the equations $w^T_i \tilde{F} = 0, i = 1, \ldots, 6$ can then be rearranged into the form $W f = 0$, where $W$ is a $6 \times 9$ matrix. 
We extract a basis for the null space of this linear system, by computing three right singular vectors of $W$ with $0$ singular value. 
So the fundamental matrix can be reconstructed as 
\begin{align}
    F = \alpha_1 F_1 + \alpha_2 F_2 + F_3,
\end{align}
where $F_1$, $F_2$ and $F_3$ are the basis of the nullspace  reshaped into $3 \times 3$ matrices; and $\alpha_1$ and $\alpha_2$ are free parameters in building the fundamental matrix $F$. 
From~\cite[Thm.~1]{nister:PAMI:2004}, the fundamental matrix should satisfy 
\begin{align*}
    \det (F) = 0,
\end{align*}
so that we have a polynomial with respect to $\alpha_1$ and $\alpha_2$:
\begin{align}   \label{Eq:nullspace}
    \det (\alpha_1 F_1 + \alpha_2 F_2 + F_3) = 0.
\end{align}
Now consider the final seventh ``0.5" point correspondence between the images, $(x_7 = \left((x_7)_1, (x_7)_2\right), y^{\top}_7 = \left((y_7)_1,(y_7)_2\right))$,
where $x_7$ is known but $y_7$ is not.
We seek the values of $y_7$ such that all of the image data becomes ill-posed.
To cut down on the number of variables, our strategy is to fix various numerical values for $(y_7)_1$ and then determine the corresponding degenerate values for the last coordinate $(y_7)_2$.
Firstly we have the additional linear constraint
\begin{align}   \label{Eq:extralinear}
    \begin{pmatrix} y_7 \\ 1 \end{pmatrix}^{\!\! \top} \!\! (\alpha_1 F_1 + \alpha_2 F_2 + F_3) \begin{pmatrix} x_7 \\ 1 \end{pmatrix} = 0.
\end{align}
Combining \eqref{Eq:nullspace} and~\eqref{Eq:extralinear}, we  have two equations in the three unknowns $\alpha_1, \alpha_2, (y_7)_2$:
\begin{align}
    \left\{
    \begin{matrix}
    \det (\alpha_1 F_1 + \alpha_2 F_2 + F_3) = 0 \\
    ((y_7)_1,(y_7)_2,1) (\alpha_1 F_1 + \alpha_2 F_2 + F_3) (x_7, 1)^{\!\top} = 0.
    \end{matrix}
    \right.
    \label{Eq:twosystem}
\end{align}

To find the points on the 6.5-point degenerate curve, we need 
to also enforce rank-deficiency of the Jacobian of the system \eqref{Eq:twosystem} with respect to $\alpha_1, \alpha_2$.  This Jacobian reads
\begin{align}
J = \begin{pmatrix}
    \frac{\partial \det (\alpha_1 F_1 + \alpha_2 F_2 + F_3)}{\partial \alpha_1}& \frac{\partial \det (\alpha_1 F_1 + \alpha_2 F_2 + F_3)}{\partial \alpha_2}  \\[4pt]
    \frac{\partial ((y_7)_1,(y_7)_2,1) (\alpha_1 F_1 + \alpha_2 F_2 + F_3) (x_7, 1)^{\!\top}}{\partial \alpha_1}& \frac{\partial [(y_7)_1,(y_7)_2,1] (\alpha_1 F_1 + \alpha_2 F_2 + F_3) (x_7, 1)^{\!\top}}{\partial \alpha_2} 
    \end{pmatrix}.
\end{align}
To express rank-deficiency, we introduce a dummy scalar variable $d_1$ to represent the non-trivial null space for $J$.
All together we have the following system of equations now:
\begin{align*}
    \left\{
    \begin{matrix}
    \det (\alpha_1 F_1 + \alpha_2 F_2 + F_3) = 0 \\[3pt]
    ((y_7)_1,(y_7)_2,1) (\alpha_1 F_1 + \alpha_2 F_2 + F_3) (x_7, 1)^{\top} = 0 \\[3pt]
   J (d_1, 1)^{\!\top} = 0.
    \end{matrix}
    \right.
    \label{Eq:finalsystem}
\end{align*}

To plot the 6.5-point curve, we set the parameter $(y_7)_1$ to different real values, \eg to horizontally range over all the pixels in the second image.
For each fixed value of $(y_7)_1$, \eqref{Eq:finalsystem} becomes a square polynomial system in the variables $\alpha_1, \alpha_2, (y_7)_2$.
The real solutions correspond to the intersection of the 6.5-point curve and the corresponding column of the image. 
The curve inside the image boundaries can be obtained by solving these various systems independently, see Figure~\ref{fig:scanning}(a). 
We choose to solve the polynomials using homotopy continuation \cite{sommese2005numerical} as implemented in the  Julia package \cite{breiding2018homotopycontinuation}. 
By linearly connecting the intersection points, the 6.5-point curve is rendered. 
In some possible applications, a full plot of the curve may not be required. 
For example, consider checking the distance from a given point to the curve. 
In that case, we can simply compute the intersection points as $(y_7)_1$ ranges over a small interval around the correspondence candidate, see Figure~\ref{fig:scanning}(b). 
Finally, computations for different columns of the image  are independent, so the described procedures are easily parallelized.

\begin{figure}[h]
    \centering
    (a)\includegraphics[width=0.3\linewidth]{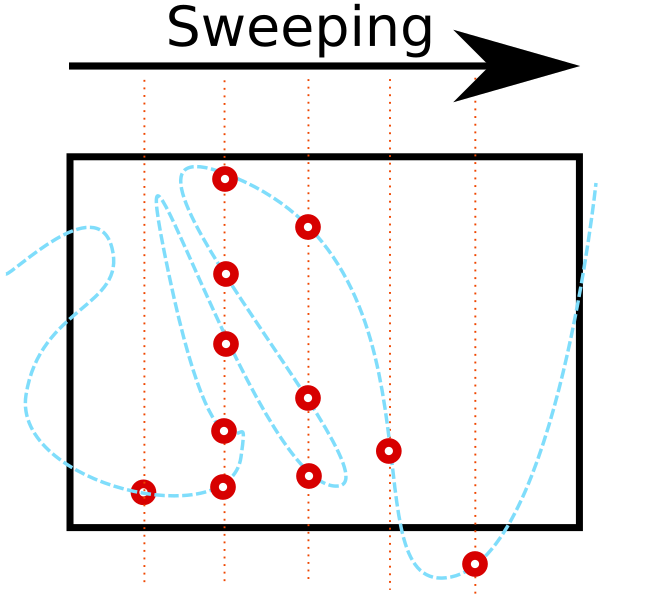}
    (b)\includegraphics[width=0.3\linewidth]{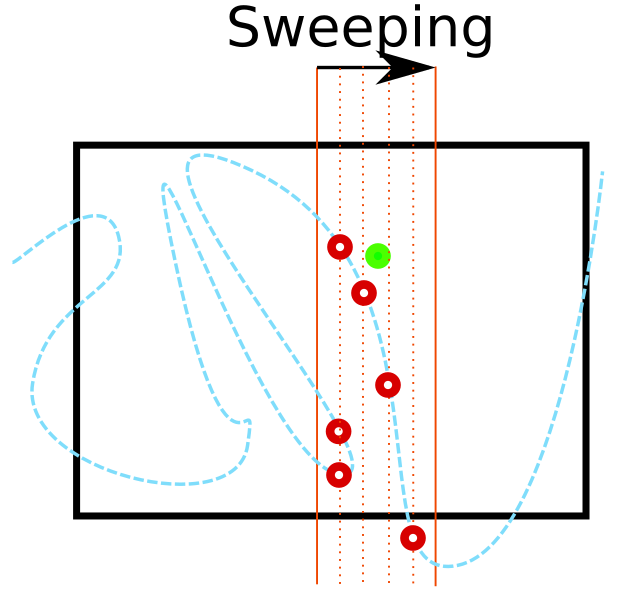} \\
    \includegraphics[width=0.8\linewidth]{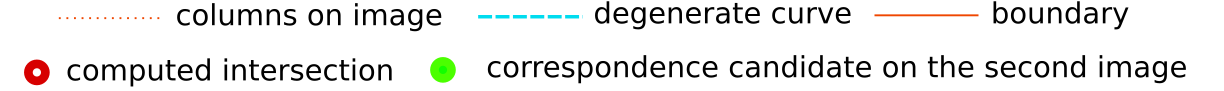}
    \caption{(a) To generate the X.5-point curves on the second image, we can sweep the image column-wise and compute the intersection with vertical lines by solving \eqref{Eq:twosystem}. (b) Given a candidate correspondence, we can  scan just a neighborhood around the candidate point.}
    \label{fig:scanning}
\end{figure}

\vspace{1em}

For the \textbf{calibrated case} (similarly to the uncalibrated case), with four known correspondences, we can build a linear system analogous to ~\eqref{Eq:linearSystem} in the variable $E$. 
Here the null space is $5$-dimensional, so we represent the essential matrix as 
\begin{align}
    E = \alpha_1 E_1 + \alpha_2 E_2 + \alpha_3 E_3 + \alpha_4 E_4 + E_5.
\end{align}
where $E_i$ provide a basis of the null space. 
From~\cite{nister:PAMI:2004}, the essential matrix should also satisfy the following polynomial constraints:
\begin{align*}
    \det(E) = 0 \qquad \textup{and} \qquad E(E^{\top}E) - \frac{1}{2} \operatorname{trace}(EE^{\top})E = 0,
\end{align*}
which are in total 10 cubic equations.
Similarly to uncalibrated case, to find the degenerate configurations, the Jacobian of these constraints with respect to $\alpha_1, \ldots, \alpha_4$ should be rank-deficient.
The Jacobian can be built as follows:
{\footnotesize
\begin{align*}
    J = \begin{pmatrix}
    \frac{\partial \det(E)}{\partial \alpha_1} & \frac{\partial \det (E)}{\partial \alpha_2} & \frac{\partial \det (E)}{\partial \alpha_3} & \frac{\partial \det (E)}{\partial \alpha_4}  \\[4pt]
    \frac{\partial ((y_7)_1,(y_7)_2,1) E (x_7, 1)^{!\top}}{\partial \alpha_1} & \frac{\partial ((y_7)_1,(y_7)_2,1) E (x_7, 1)^{\! \top}}{\partial \alpha_2} & \frac{\partial ((y_7)_1,(y_7)_2,1) E (x_7, 1)^{\! \top}}{\partial \alpha_3} & \frac{\partial ((y_7)_1,(y_7)_2,1) E (x_7, 1)^{\! \top}}{\partial \alpha_4} \\[4pt]
    \frac{\partial \operatorname{vec}(E(E^{\top}E) - \frac{1}{2} trace(EE^{\top})E)}{\partial \alpha_1} &
    \frac{\partial \operatorname{vec}(E(E^{\top}E) - \frac{1}{2} trace(EE^{\top})E)}{\partial \alpha_2} & 
    \frac{\partial  \operatorname{vec}(E(E^{\top}E) - \frac{1}{2} trace(EE^{\top})E)}{\partial \alpha_3} & 
    \frac{\partial  \operatorname{vec}( E(E^{\top}E) - \frac{1}{2} trace(EE^{\top})E)}{\partial \alpha_4}
    \end{pmatrix}
\end{align*}
}
Here $\operatorname{vec}(\cdot)$ represents the vectorization of a $3 \times 3$ matrix into a $9 \times 1$ vector, so that $J$ is a $11 \times 4$ matrix. 
Then, we introduce dummy variables $d_1, d_2, d_3$ to express rank-deficiency of $J$ and build the following system of equations:
{\footnotesize
\begin{align*}
\left\{
\begin{matrix}
\det(E) = 0 \\[4pt]
((y_7)_1,(y_7)_2,1) E (x_7, 1)^{\! \top} = 0 \\[4pt]
 E(E^{\top}E) - \frac{1}{2} trace(EE^{\top})E = 0 \\[4pt]
J (d_1, d_2, d_3, 1)^{\! \top} = 0
\end{matrix}
\right.
\end{align*}
}
where $E = \alpha_1 E_1 + \alpha_2 E_2 + \alpha_3 E_3 + \alpha_4 E_4 + E_5$. Note that we have in total 22 equations and 8 unknowns.
The variables are $\alpha_1$, $\alpha_2$, $\alpha_3$, $\alpha_4$, $(y_7)_1$, $(y_7)_2$, $d_1$, $d_2$ and $d_3$. 
The solutions to this system define the 4.5-point curve. 

By setting $(y_7)_1$ to various different values, we can find the zero-dimensional solution sets following the same approach as in Figure~\ref{fig:scanning}. 
The real solutions for $(y_7)_2$ correspond to the intersection of the 4.5-point curve and a column of the image. 
The solutions to these systems are easily computed using HomotopyContinuation.jl.
Note that the systems have 30 complex solutions, so that we will have at most 30 real intersection points with the various columns of the second image. 

\medskip
\medskip

\section{Additional Experimental Results}
\label{sec:exp}

\smallskip

\subsection{Extra Curve Samples}
The main body showed four sample X.5-point curves for the calibrated and uncalibrated minimal problems. Figure~\ref{fig:sample} shows more synthetic curves.  We have included different cases corresponding to stable and unstable problem instances.

\begin{figure}[h]
    \centering
    (a)\includegraphics[width = 0.3\linewidth]{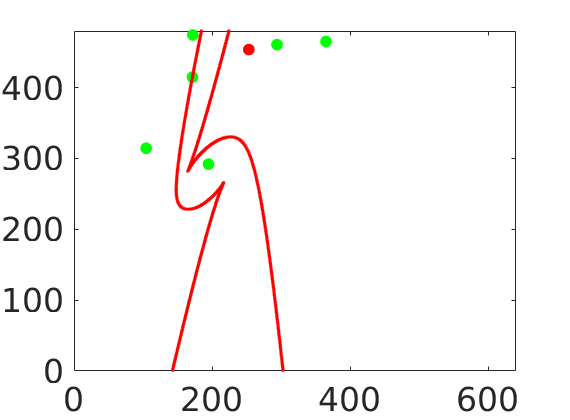}
    \includegraphics[width = 0.3\linewidth]{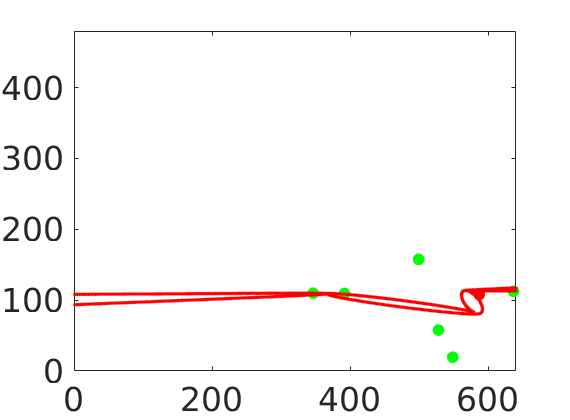}
    \includegraphics[width = 0.3\linewidth]{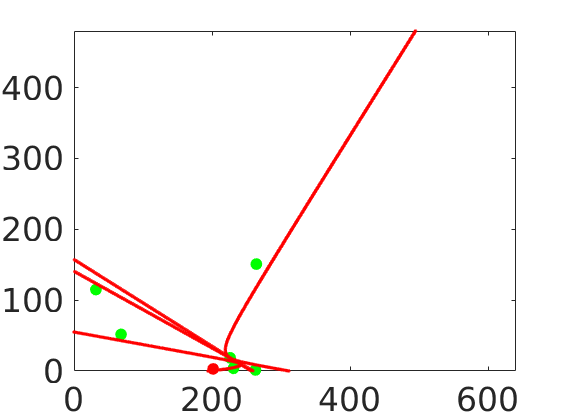} \\
    (b)\includegraphics[width = 0.3\linewidth]{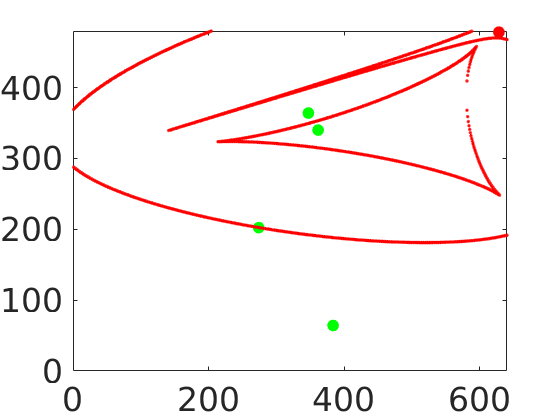}
    \includegraphics[width = 0.3\linewidth]{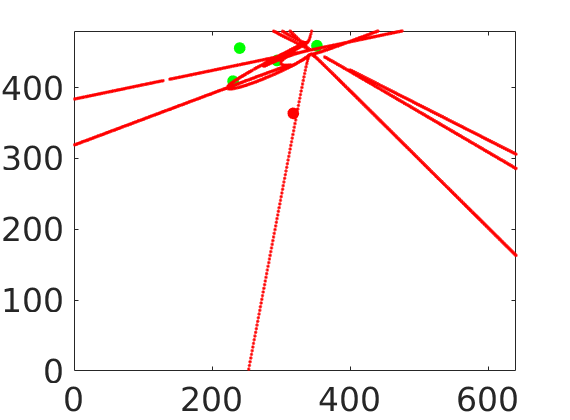}
    \includegraphics[width = 0.3\linewidth]{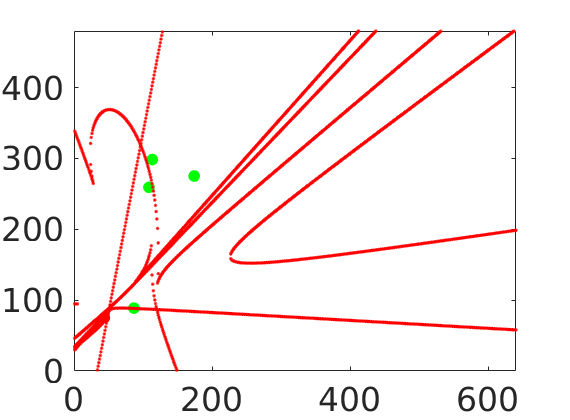}
    \caption{Sample renderings of the $X.5$-point curve. Points used in computing the curve are shown as green; the red point is the 5th/7th correspondence on the second image for calibrated/uncalibrated relative pose estimation; the red curve is the X.5-point curve we computed using homotopy continuation. (a) 6.5-point curves for the uncalibrated case. (b) 4.5-point curves for the calibrated case.}
    \label{fig:sample}
\end{figure}

\smallskip

\subsection{Stability of Curves for Calibrated Case}

\begin{figure}[h]
    \centering
    (a)\includegraphics[width = 0.3\linewidth]{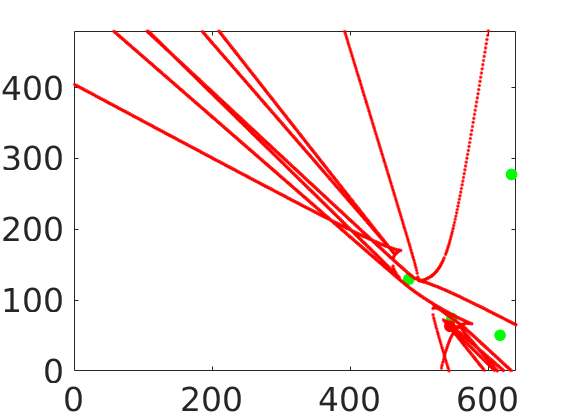}
    (b)\includegraphics[width = 0.3\linewidth]{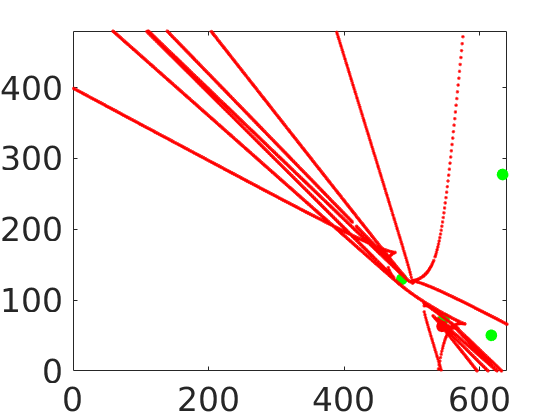}
    \includegraphics[width = 0.3\linewidth]{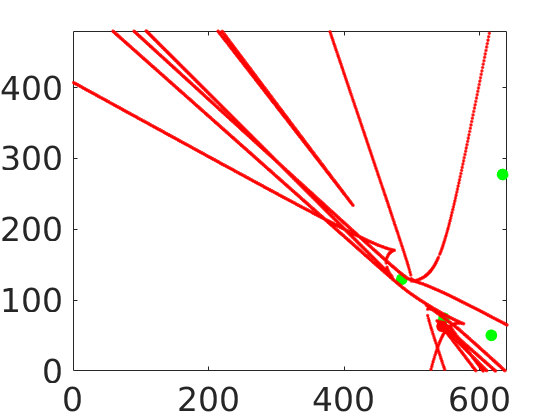}
    (c)\includegraphics[width = 0.3\linewidth]{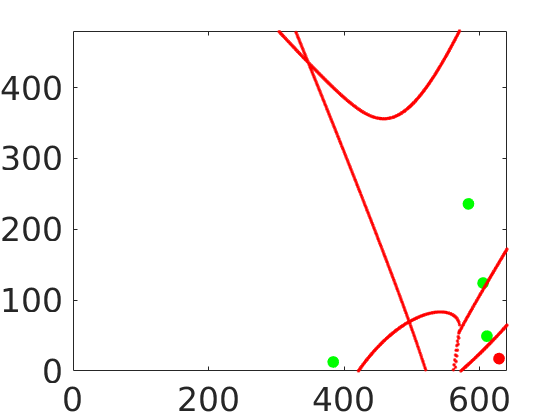}
    (d)\includegraphics[width = 0.3\linewidth]{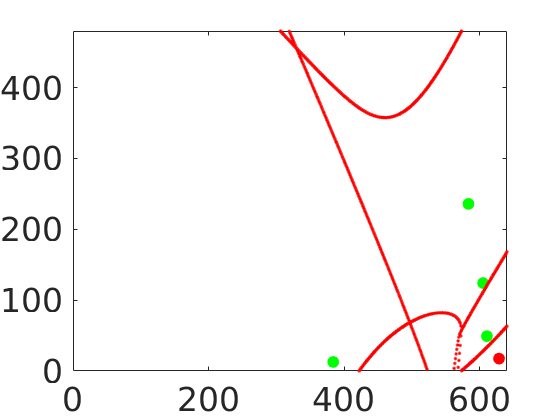}
    \includegraphics[width = 0.3\linewidth]{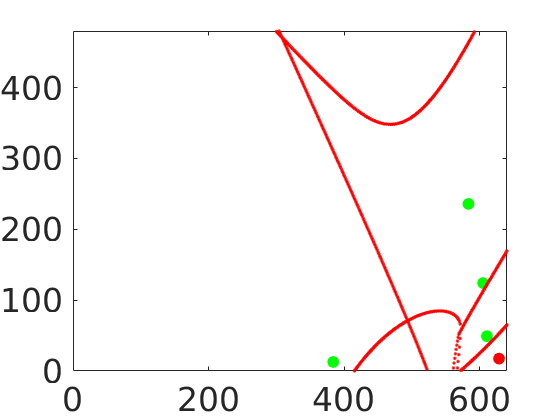}
    \caption{Illustrative result indicating the stability of the 4.5-point curve.  (a) The degenerate curve for an unstable instance of $E$ estimation. (c) The degenerate curve for a stable instance of $E$ estimation. (b) (d) Adding different noise, the curves do not change much. }
    \label{fig:noiseE}
\end{figure}

Here we show sample perturbation results for the 4.5-point curve.
For the synthetic dataset described in the main paper, we add $\mathcal{N}(0,0.5)$-noise on each of the correspondences, then compute the resulting degenerate curve. I
n the main paper, Figure~\ref{fig:my_label} shows the sample perturbation for the uncalibrated case. 
The corresponding cases for the calibrated case are in Figure~\ref{fig:noiseE}. The statistics for the calibrated cases were already included in Figure~\ref{fig:stability_syn}.

\smallskip

\subsection{More Real Image 
Examples}

Figure~\ref{fig:Real} in the  main body showed an example  using the X.5-point curve to indicate unstable configurations on real images. 
In this section, we display more examples on real images, see Tables~\ref{tab:realF} and~\ref{tab:realE}. 
Note that our method is an indication of the stability of a minimal problem instance. 
Here we selected only all-inlier minimal problem instances whose reprojection error using the ground-truth essential matrix is below a threshold (3 pixels are used). 
Then we computed the X.5-point curve on the second image. 
As mentioned in the body, the distance from the point to the curve can be used as a criteria to predict the stability. 
For highly unstable problem instances, the solution corresponding to the ground truth may suffer from large errors.

\begin{table}[h]
    \centering
    \begin{tabular}{c|c|c}
        Ground Truth Epipolar Geometry & Estimated Epipolar Geometry 
         & Degenerate Curve \\
         \hline
         \includegraphics[height=0.13\linewidth]{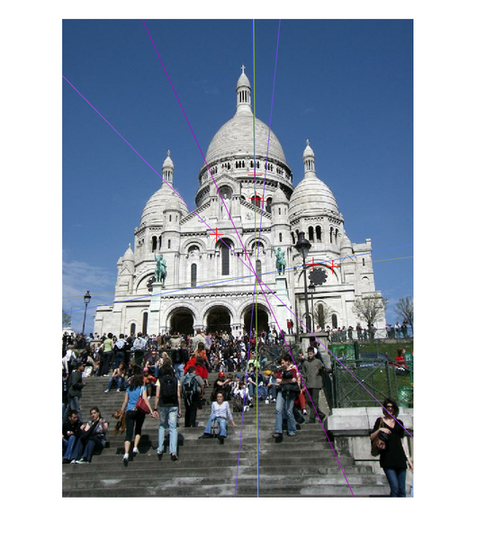}
    \includegraphics[height=0.13\linewidth]{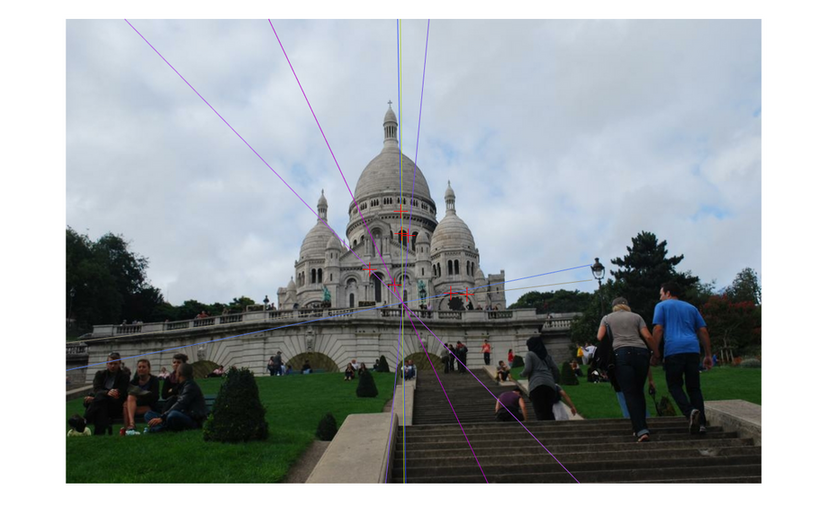} & \includegraphics[height=0.13\linewidth]{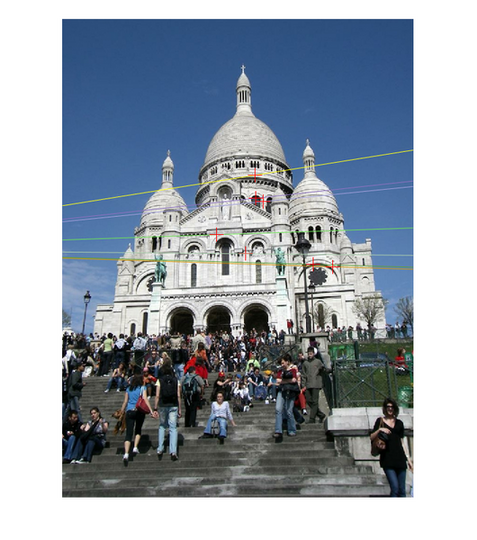}
    \includegraphics[height=0.13\linewidth]{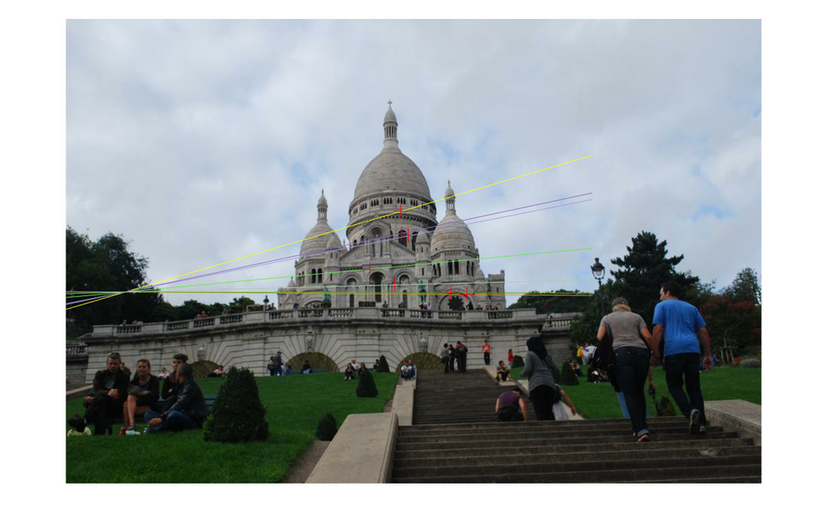} & \includegraphics[height=0.1\linewidth]{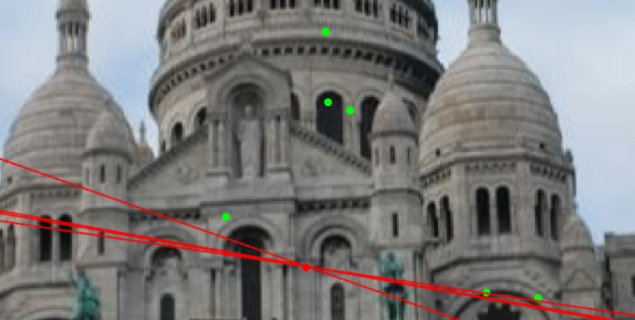} \\
    \hline
    \includegraphics[height=0.2\linewidth]{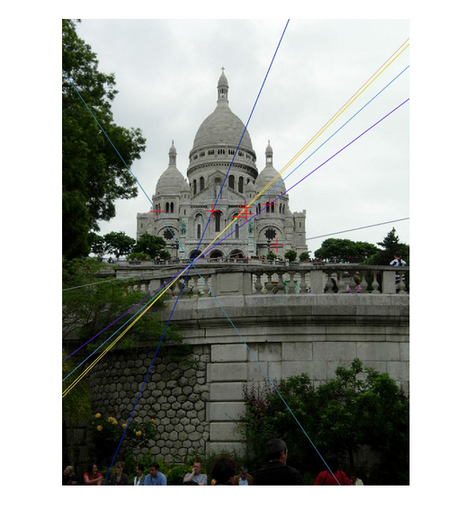}
    \includegraphics[height=0.2\linewidth]{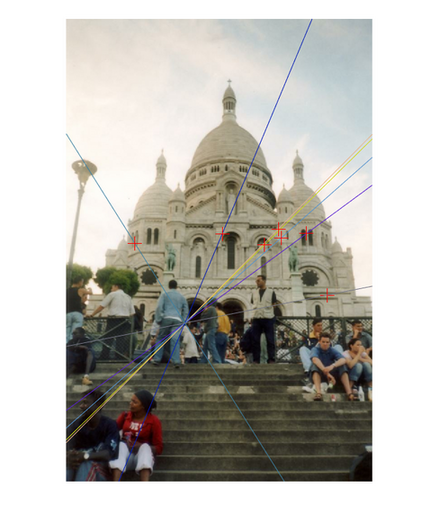} & \includegraphics[height=0.2\linewidth]{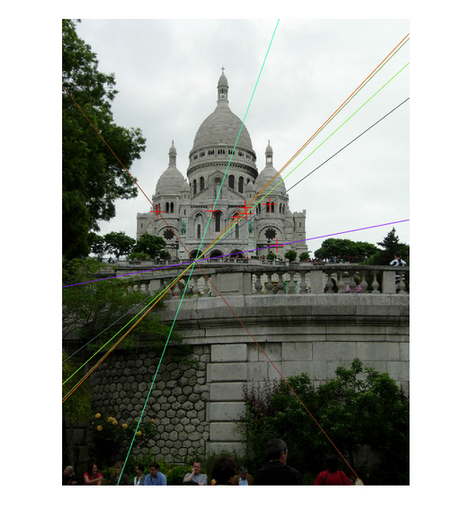}
    \includegraphics[height=0.2\linewidth]{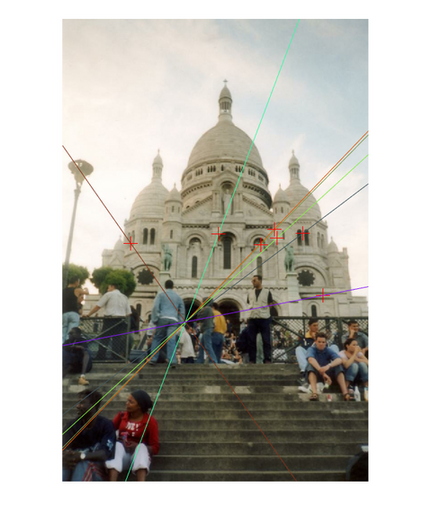} & \includegraphics[height=0.15\linewidth]{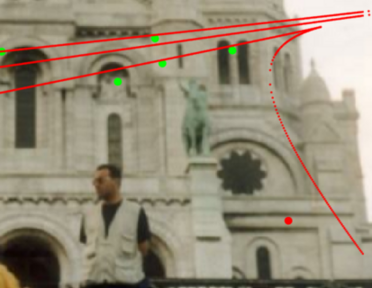} \\
    \hline
    \includegraphics[height=0.1\linewidth]{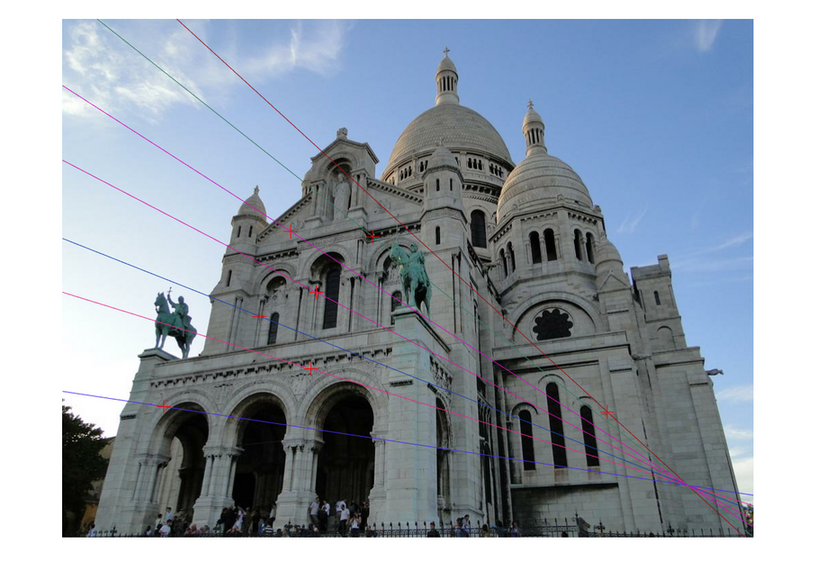}
    \includegraphics[height=0.1\linewidth]{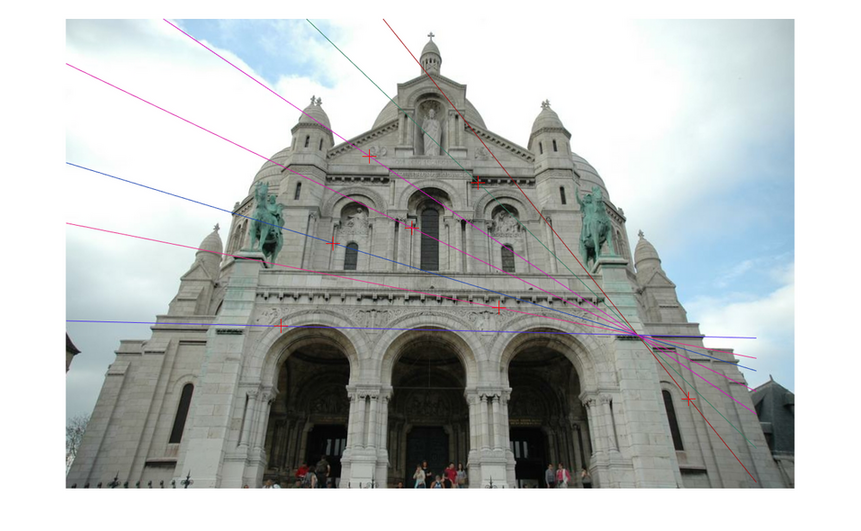} & \includegraphics[height=0.1\linewidth]{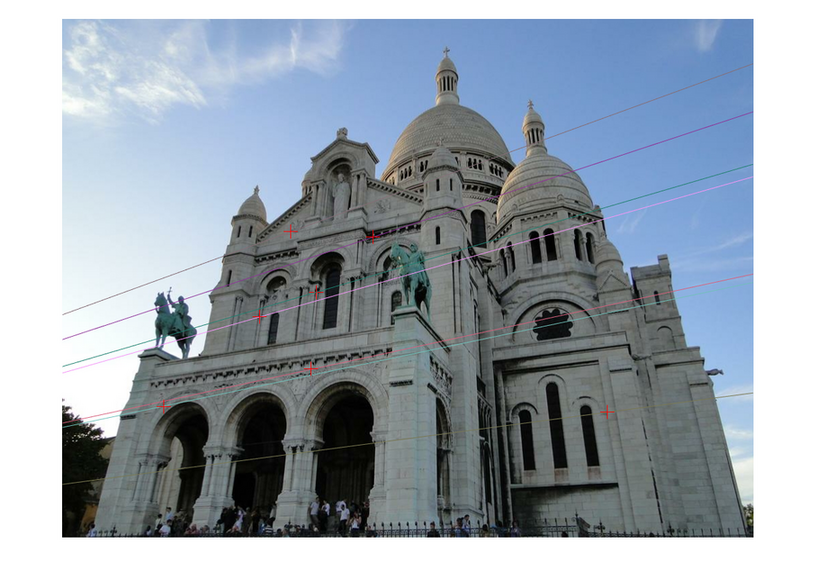}
    \includegraphics[height=0.1\linewidth]{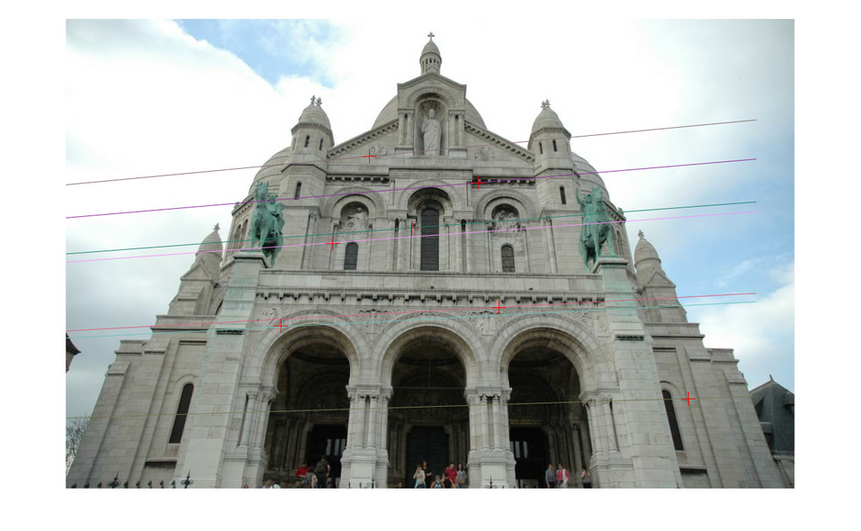} & \includegraphics[height=0.13\linewidth]{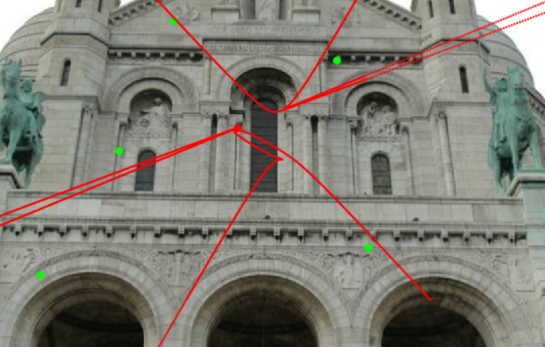} \\
    \hline
    \includegraphics[height=0.1\linewidth]{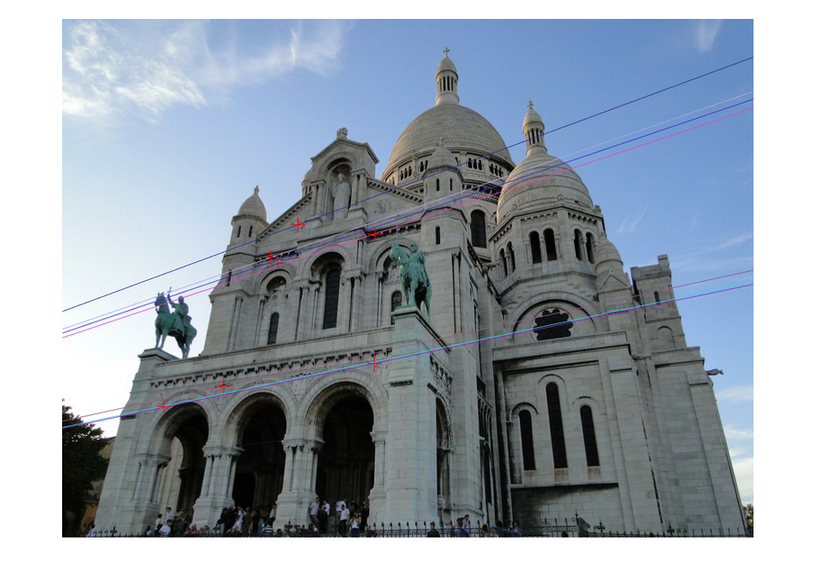}
    \includegraphics[height=0.1\linewidth]{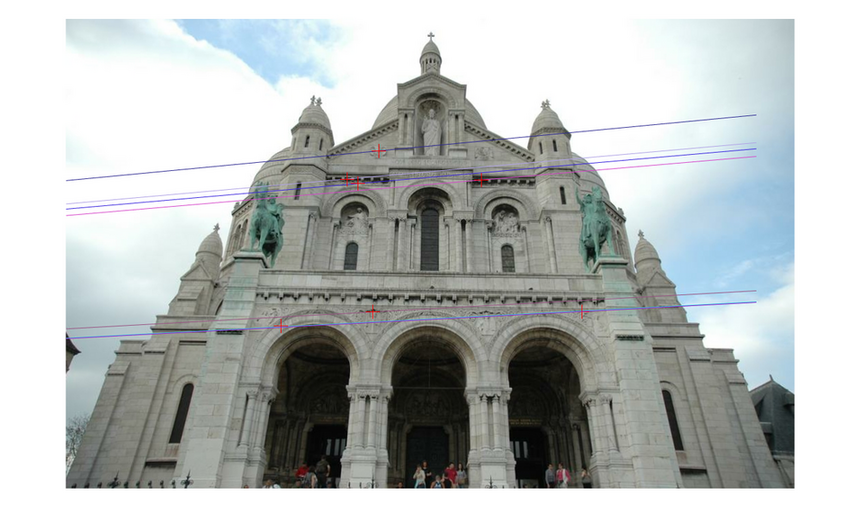} & \includegraphics[height=0.1\linewidth]{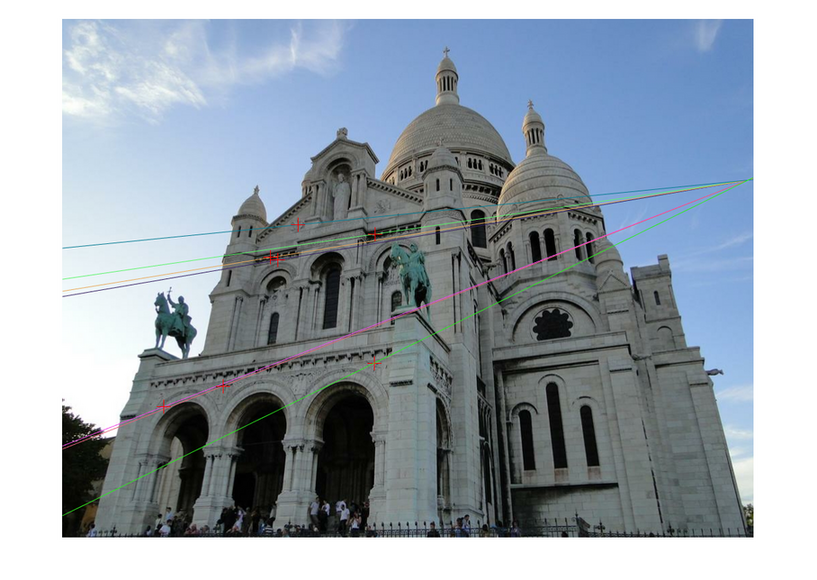}
    \includegraphics[height=0.1\linewidth]{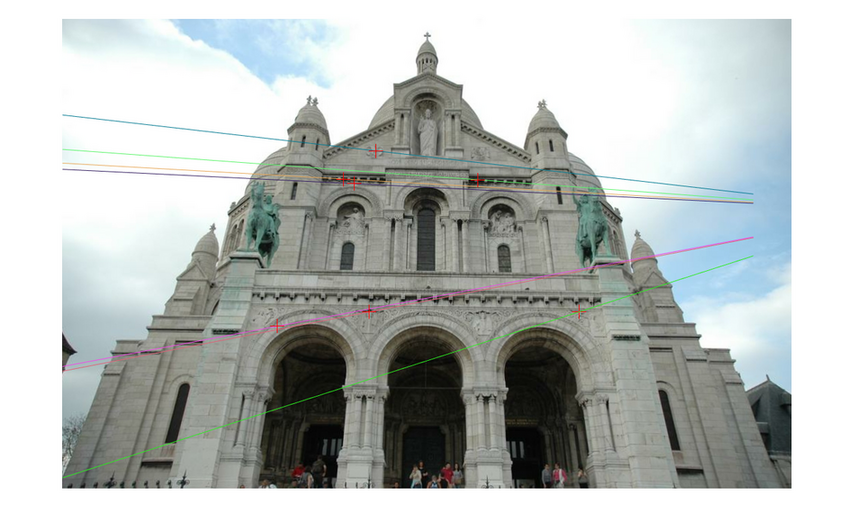} & \includegraphics[height=0.13\linewidth]{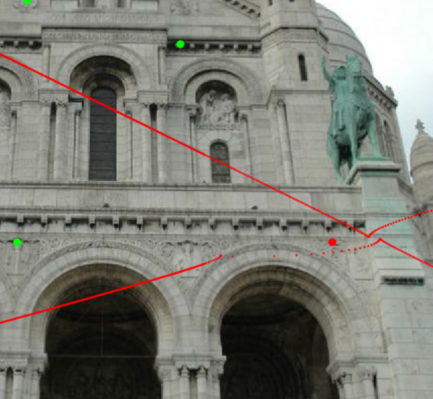} \\
    \hline
    \includegraphics[height=0.1\linewidth]{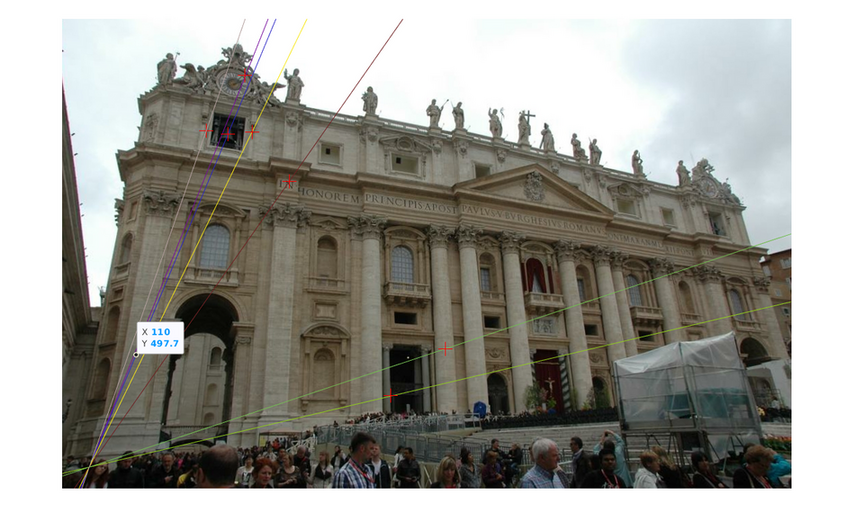}
    \includegraphics[height=0.1\linewidth]{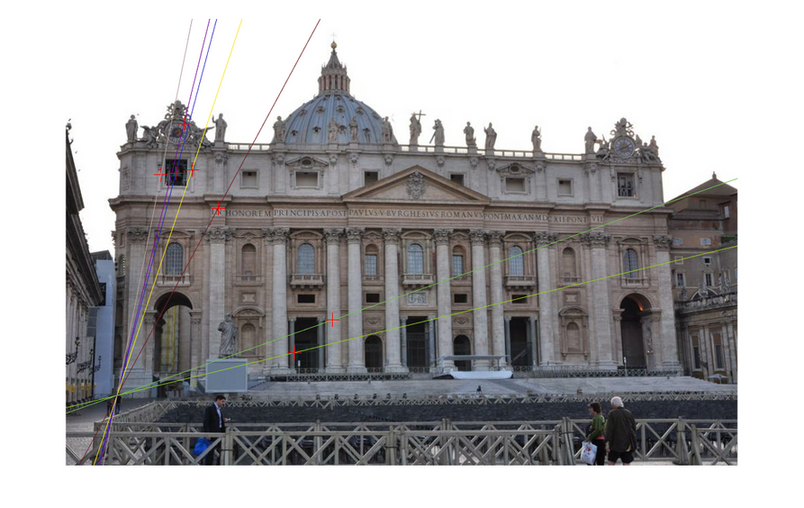} & \includegraphics[height=0.1\linewidth]{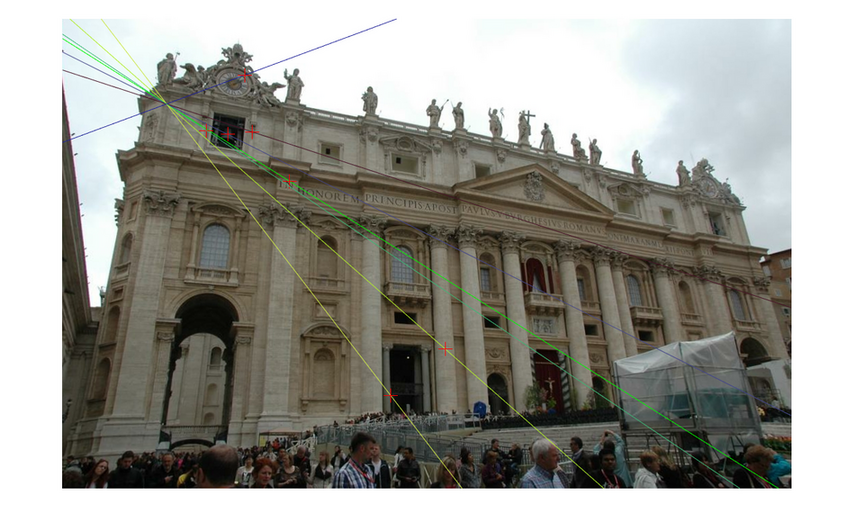}
    \includegraphics[height=0.1\linewidth]{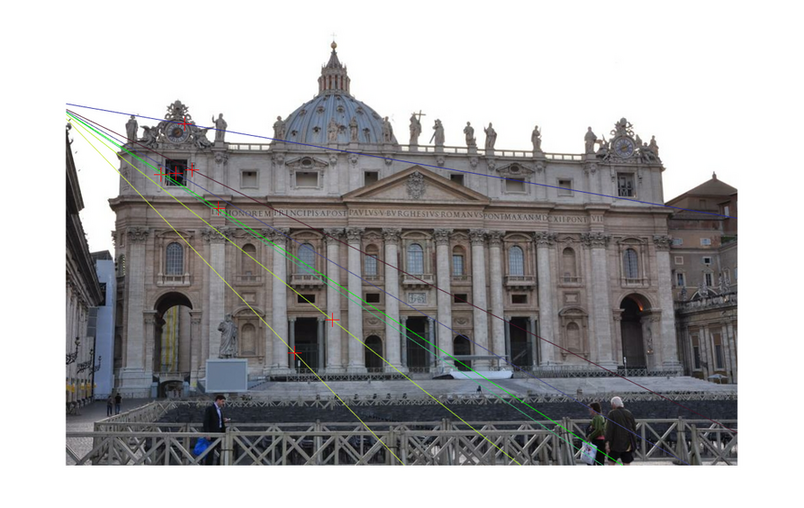} & \includegraphics[height=0.13\linewidth]{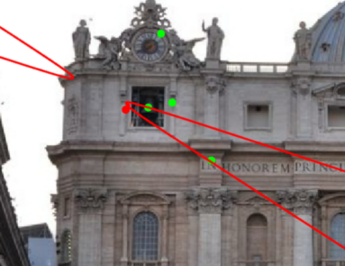} \\
    \hline
    \includegraphics[height=0.13\linewidth]{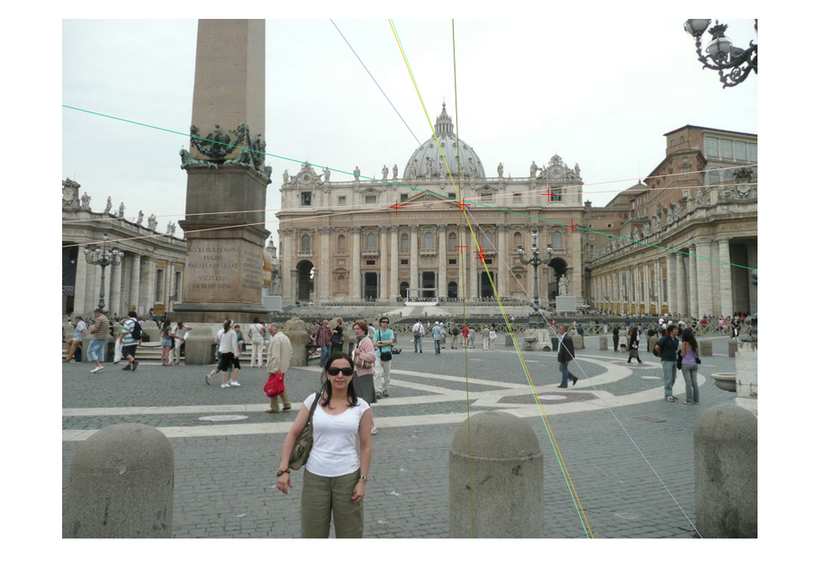}
    \includegraphics[height=0.13\linewidth]{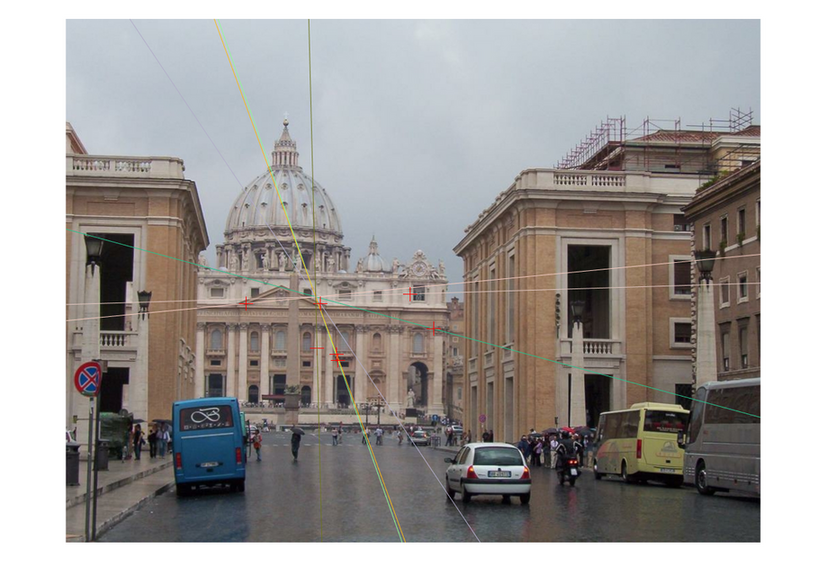} & \includegraphics[height=0.13\linewidth]{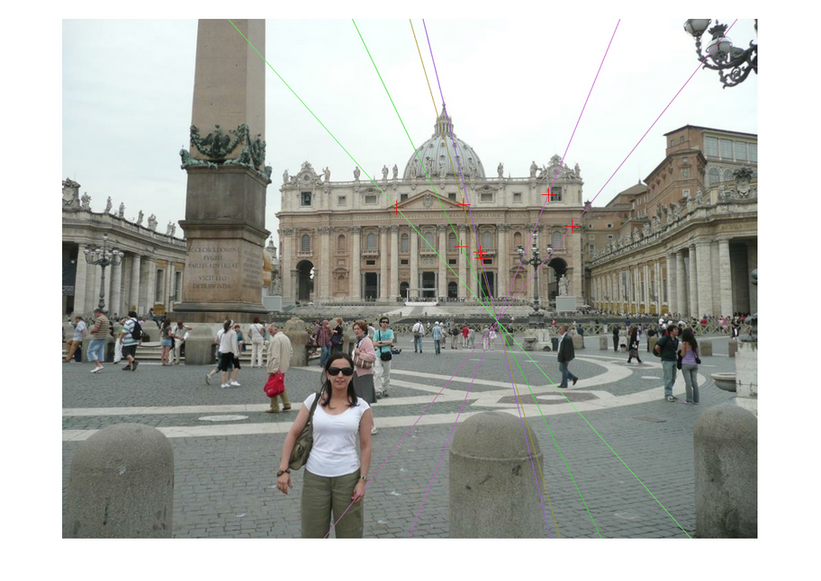}
    \includegraphics[height=0.13\linewidth]{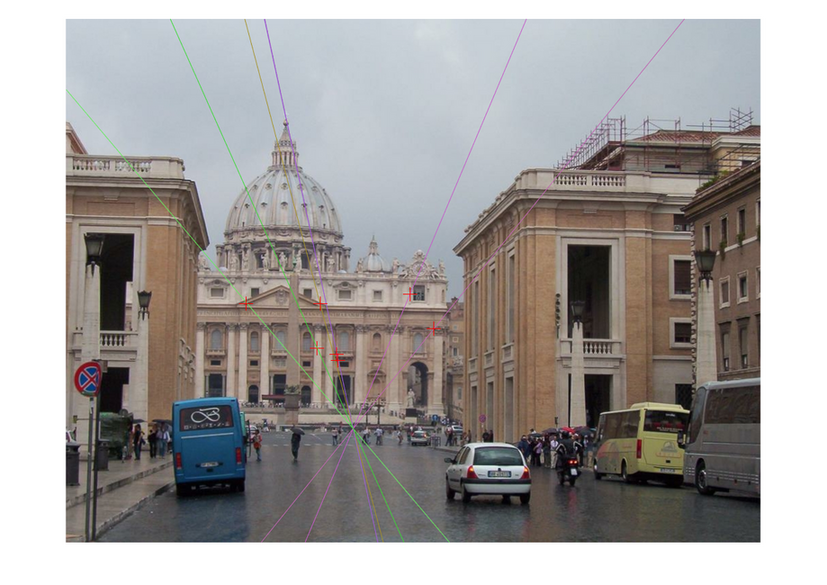} & \includegraphics[height=0.15\linewidth]{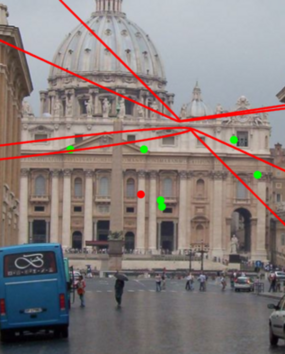} \\
    \hline
    \includegraphics[height=0.13\linewidth]{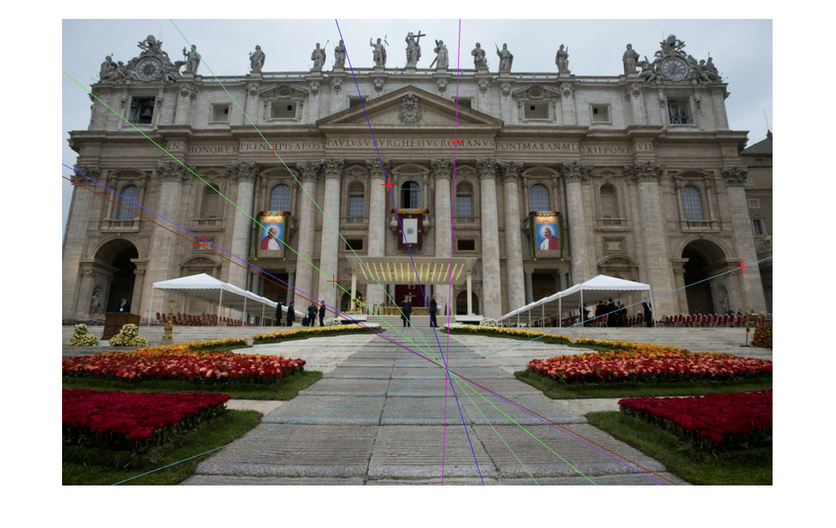}
    \includegraphics[height=0.13\linewidth]{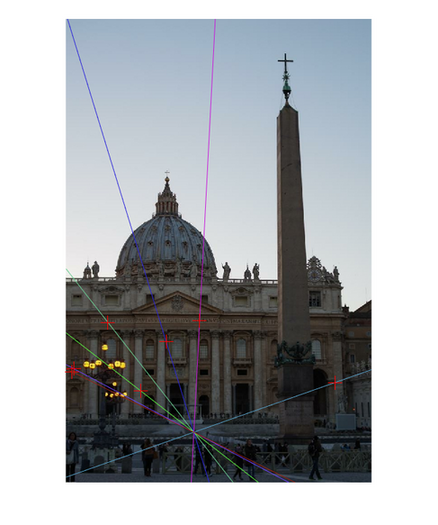} & \includegraphics[height=0.13\linewidth]{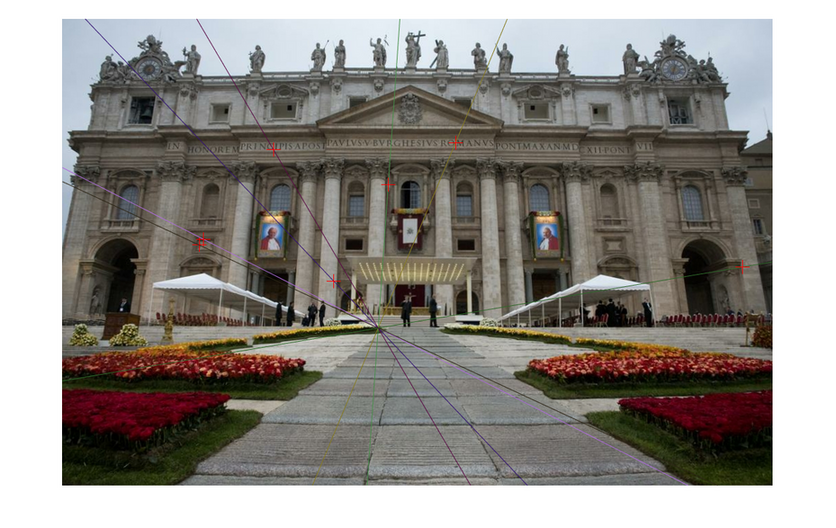}
    \includegraphics[height=0.13\linewidth]{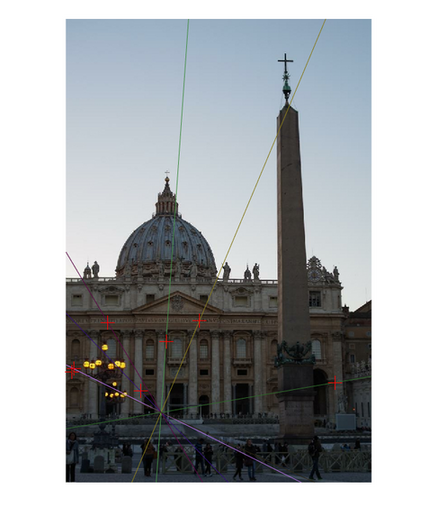} & \includegraphics[height=0.2\linewidth]{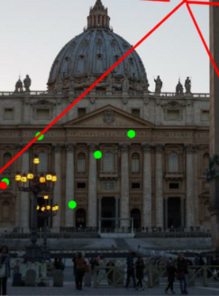} \\
    \end{tabular}
    \caption{Further samples of the 6.5-point degenerate curve on real images. The estimated epipolar geometry is the estimate closest to the ground truth among the multiple solutions to the minimal problem.}
    \label{tab:realF}
\end{table}

\begin{table}[h]
    \centering
    \begin{tabular}{c|c|c}
        Ground Truth Epipolar Geometry & Estimated Epipolar Geometry 
         & Degenerate Curve \\
         \hline
         \includegraphics[height=0.13\linewidth]{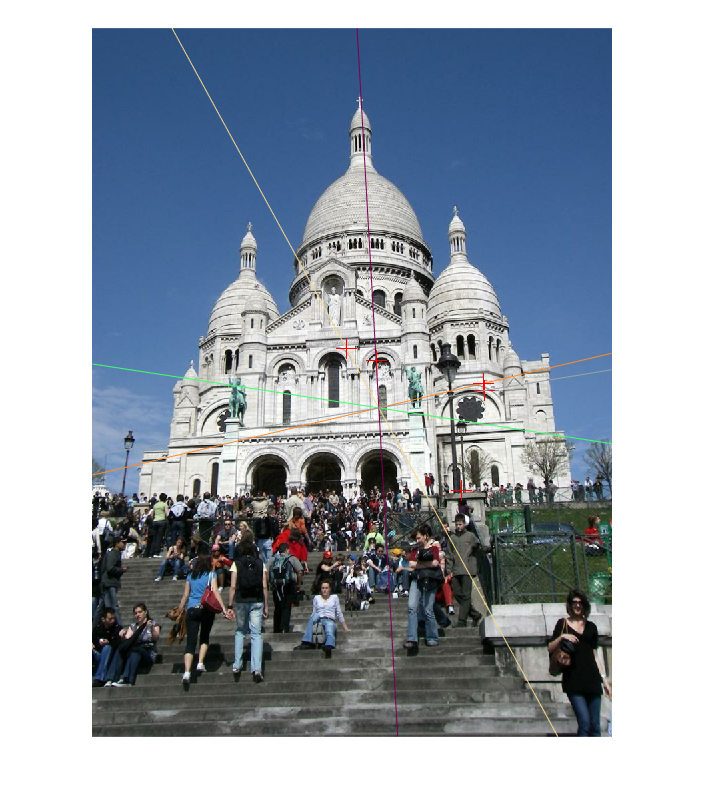}
    \includegraphics[height=0.13\linewidth]{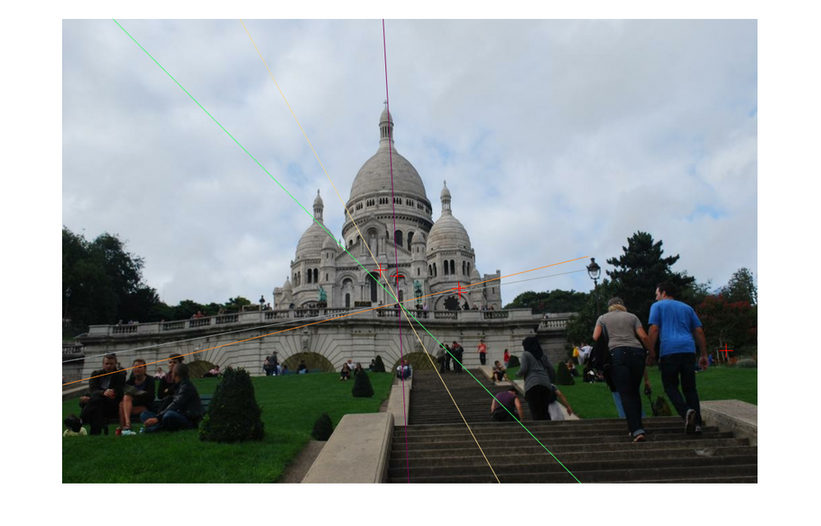} & \includegraphics[height=0.13\linewidth]{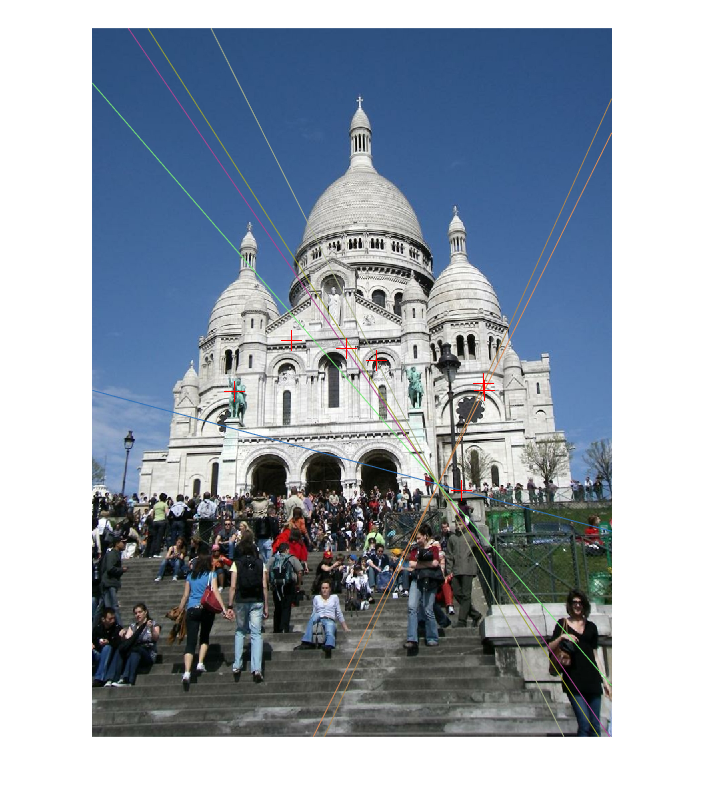}
    \includegraphics[height=0.13\linewidth]{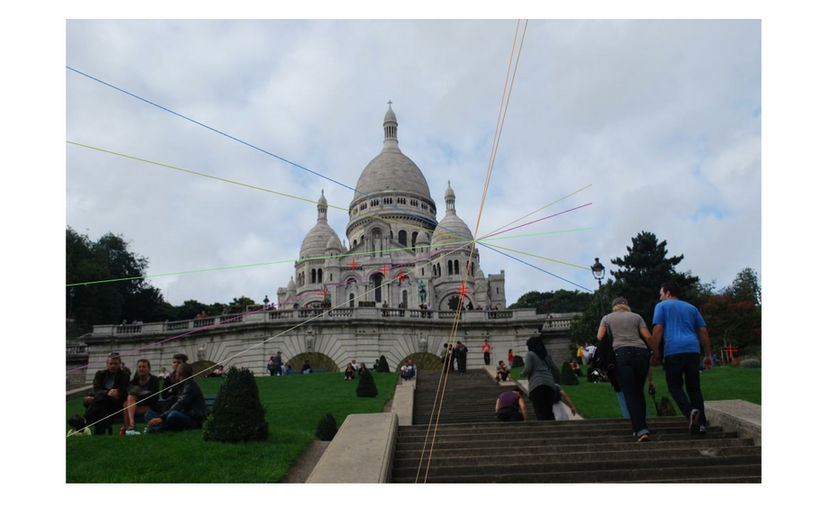} & \includegraphics[height=0.1\linewidth]{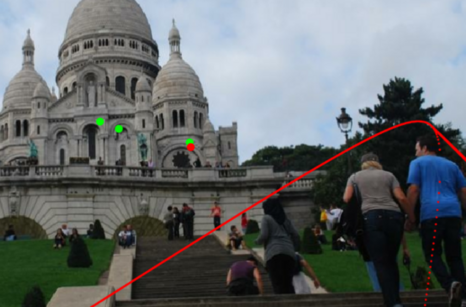} \\
    \hline
    \includegraphics[height=0.15\linewidth]{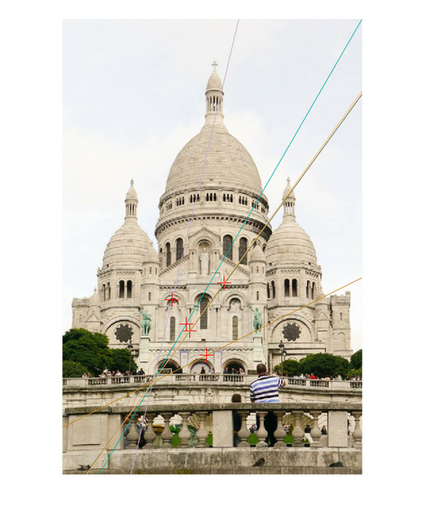}
    \includegraphics[height=0.15\linewidth]{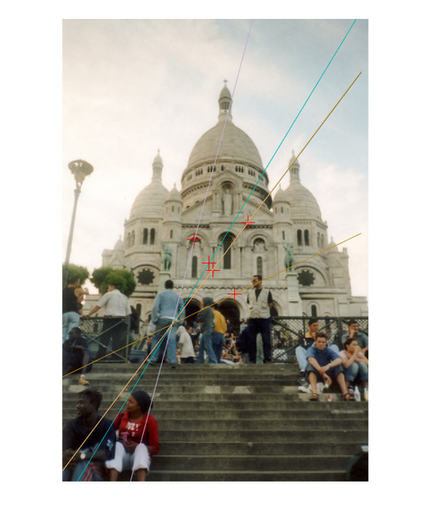} & \includegraphics[height=0.15\linewidth]{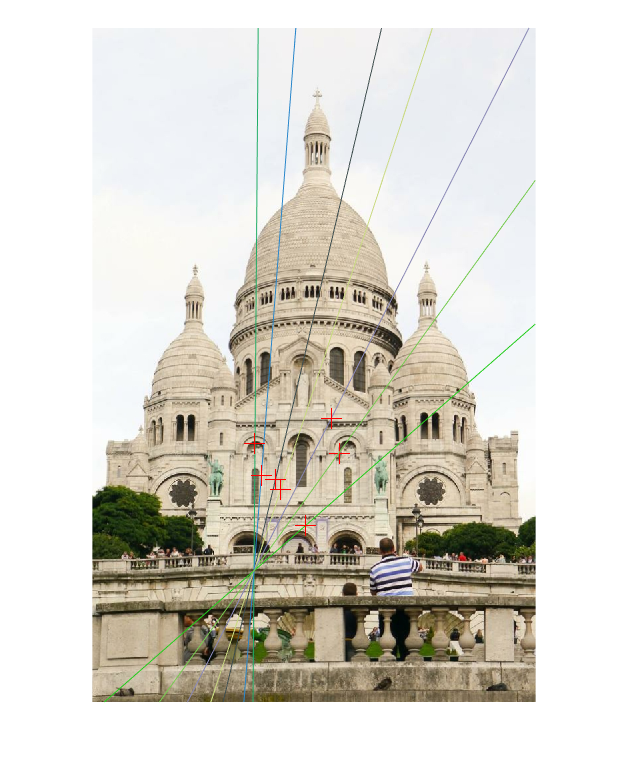}
    \includegraphics[height=0.15\linewidth]{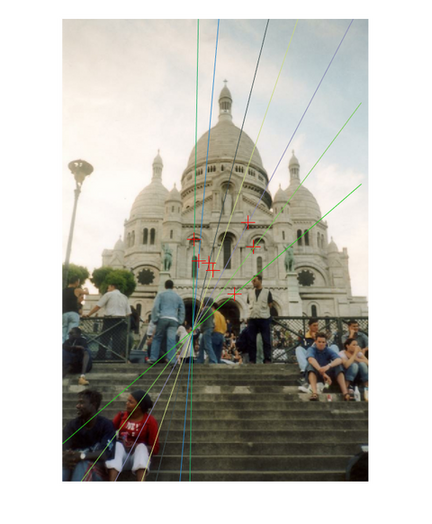} & \includegraphics[height=0.1\linewidth]{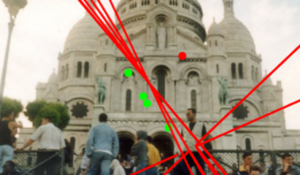} \\
    \hline
    \includegraphics[height=0.1\linewidth]{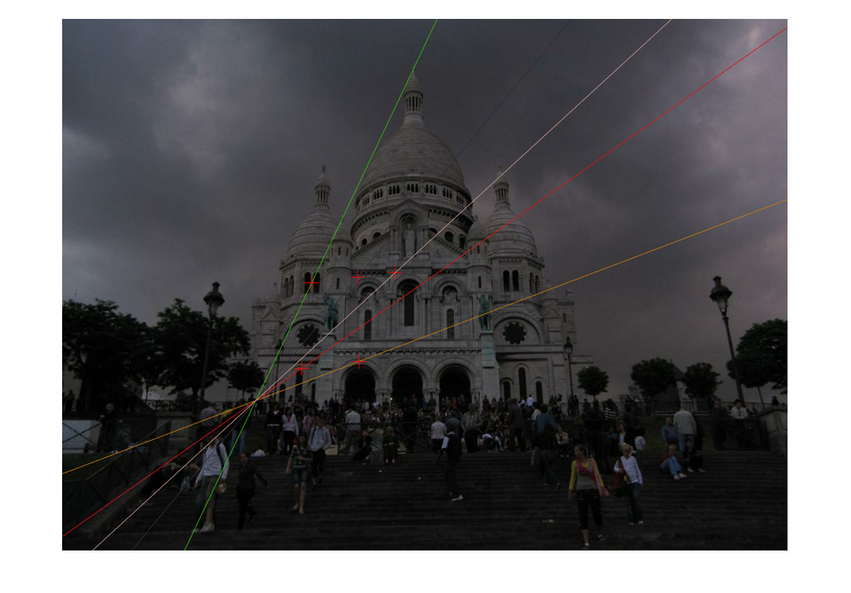}
    \includegraphics[height=0.1\linewidth]{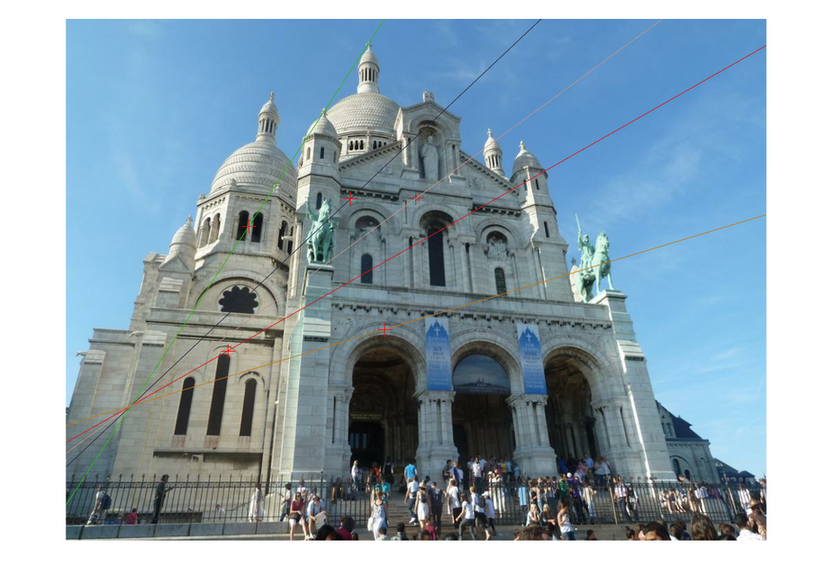} & \includegraphics[height=0.1\linewidth]{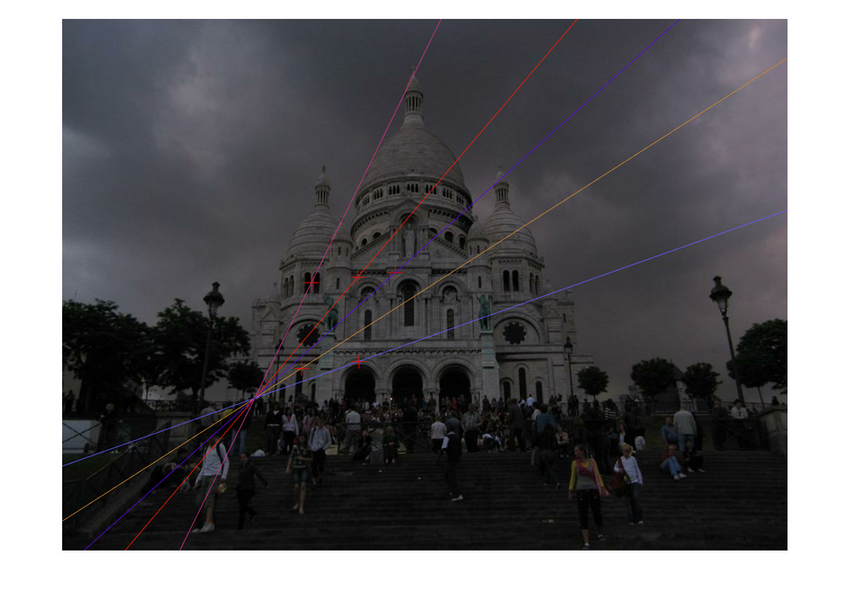}
    \includegraphics[height=0.1\linewidth]{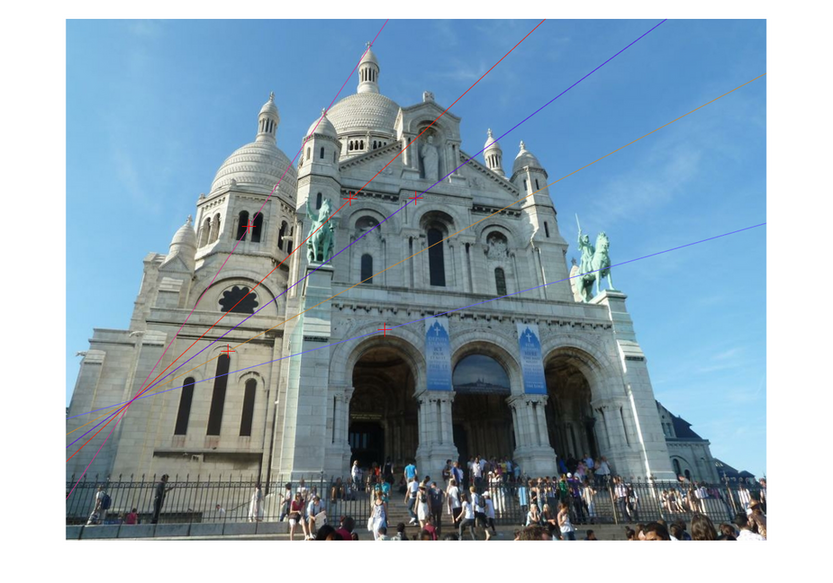} & \includegraphics[height=0.1\linewidth]{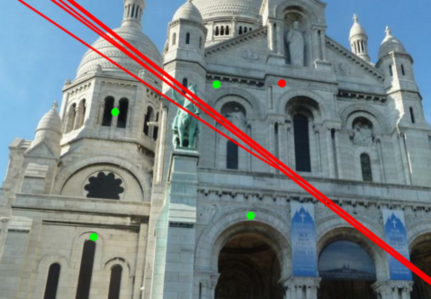} \\
    \hline
    \includegraphics[height=0.1\linewidth]{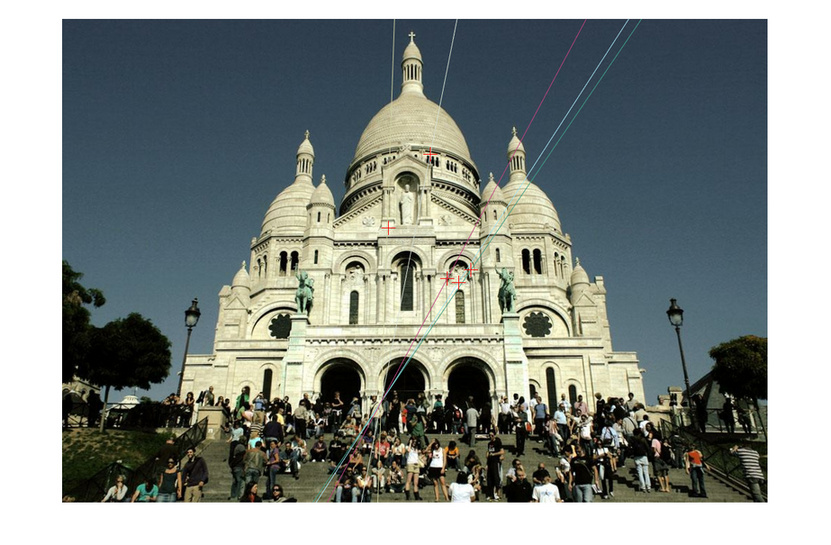}
    \includegraphics[height=0.1\linewidth]{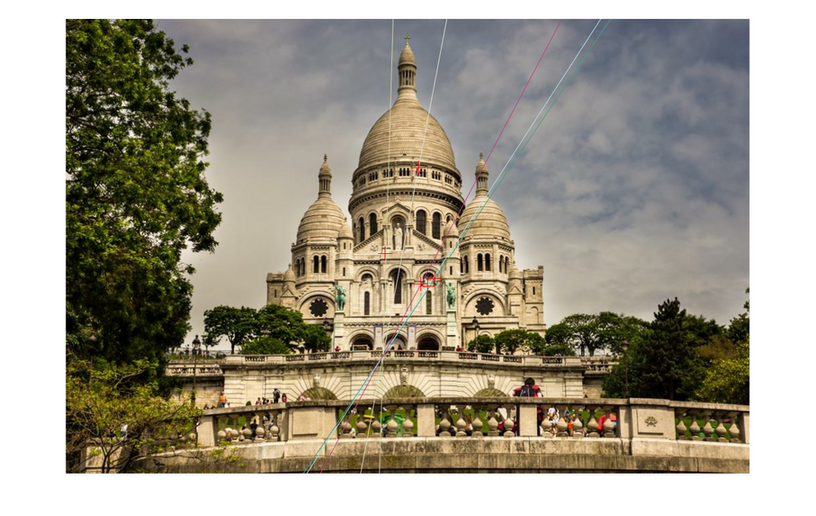} & \includegraphics[height=0.1\linewidth]{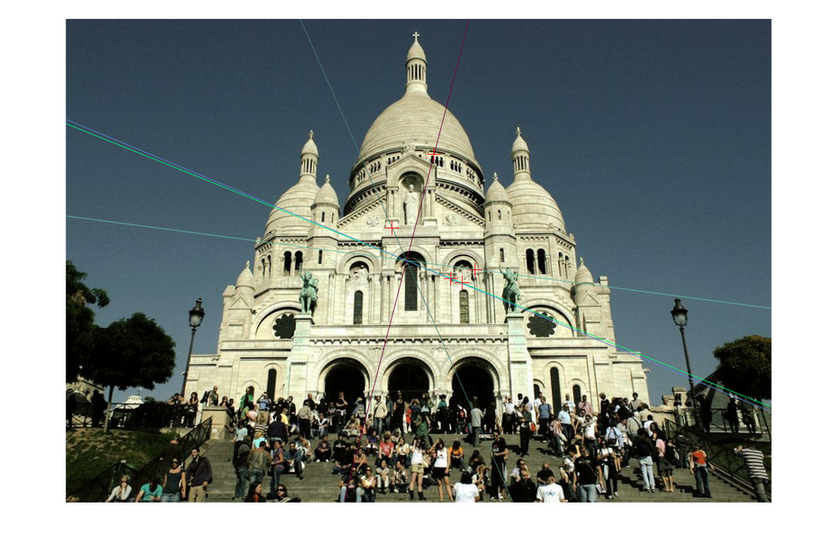}
    \includegraphics[height=0.1\linewidth]{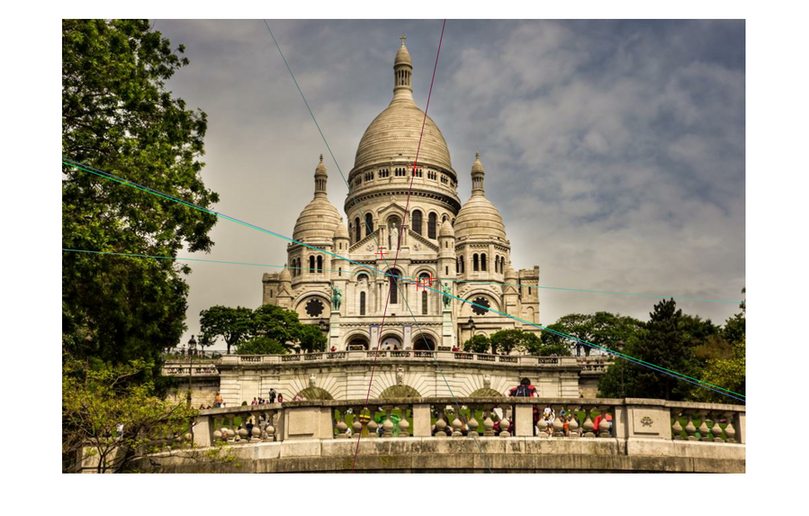} & \includegraphics[height=0.1\linewidth]{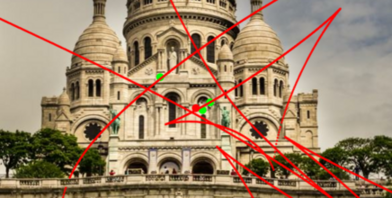} \\
    \hline
    \includegraphics[height=0.1\linewidth]{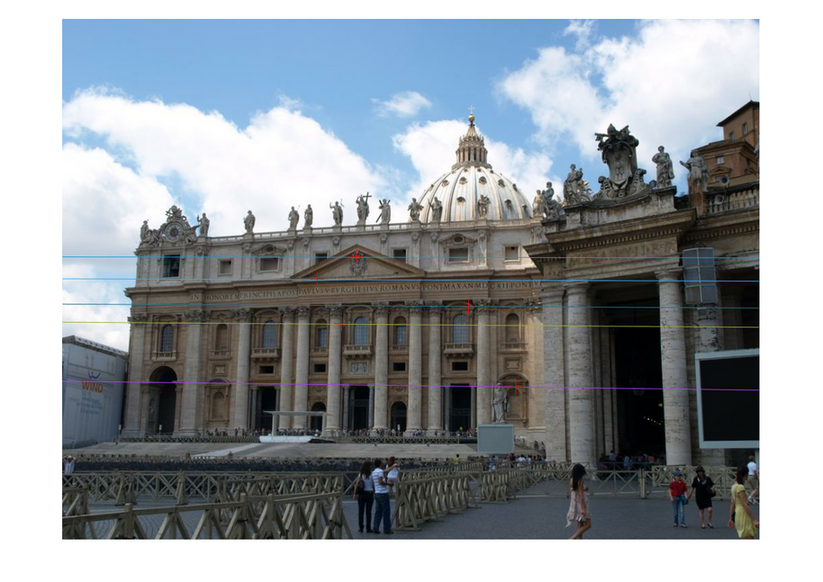}
    \includegraphics[height=0.1\linewidth]{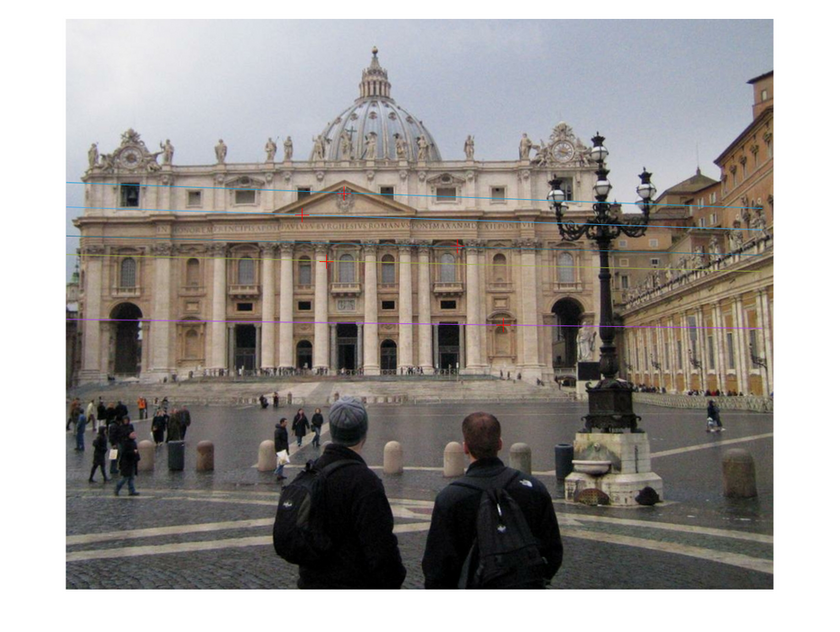} & \includegraphics[height=0.1\linewidth]{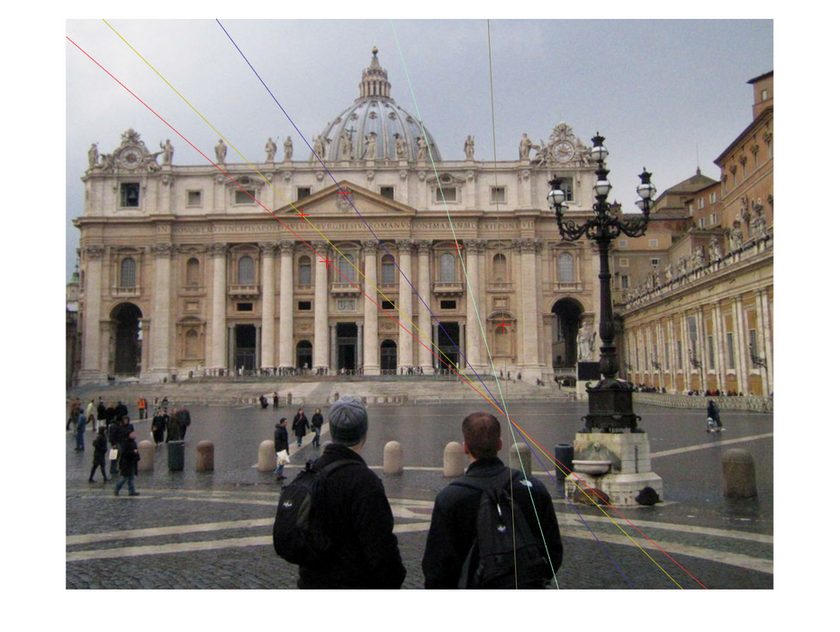}
    \includegraphics[height=0.1\linewidth]{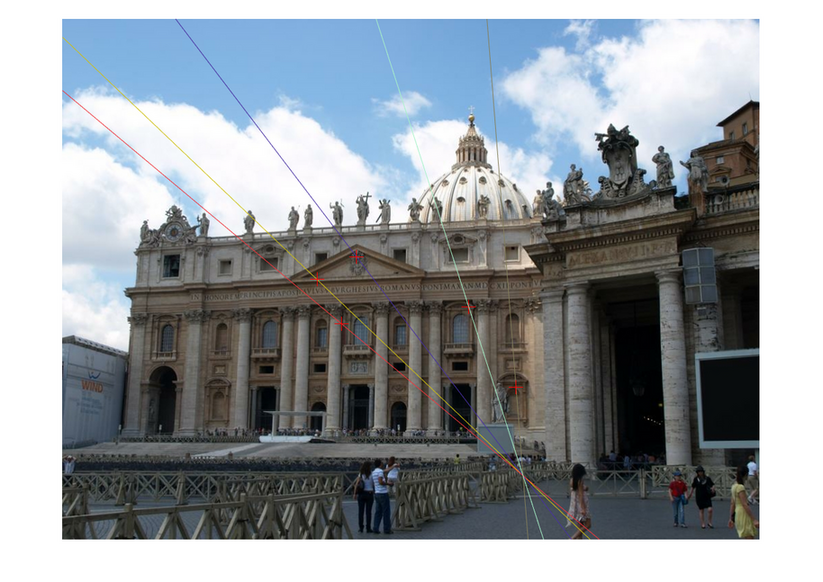} & \includegraphics[height=0.13\linewidth]{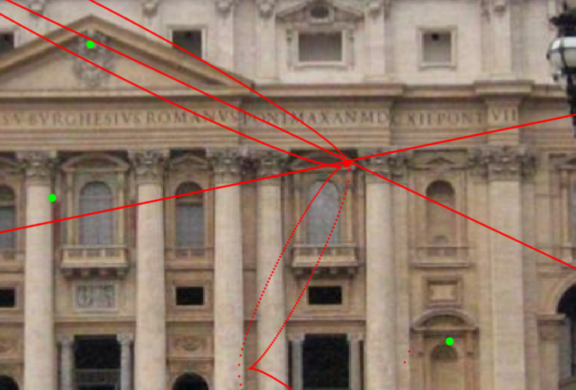} \\
    \hline
    \includegraphics[height=0.1\linewidth]{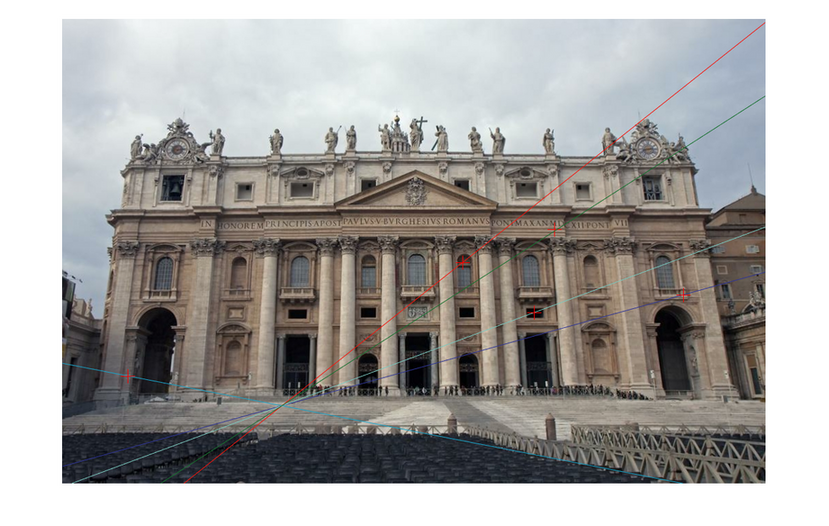}
    \includegraphics[height=0.1\linewidth]{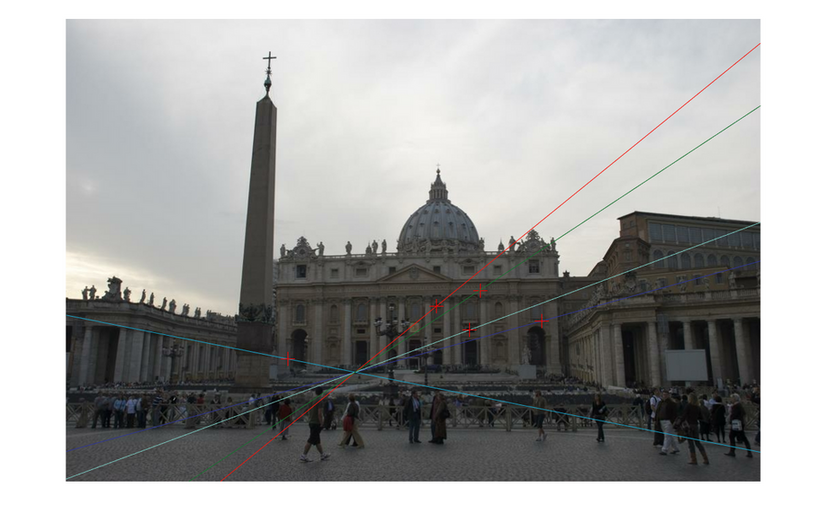} & \includegraphics[height=0.13\linewidth]{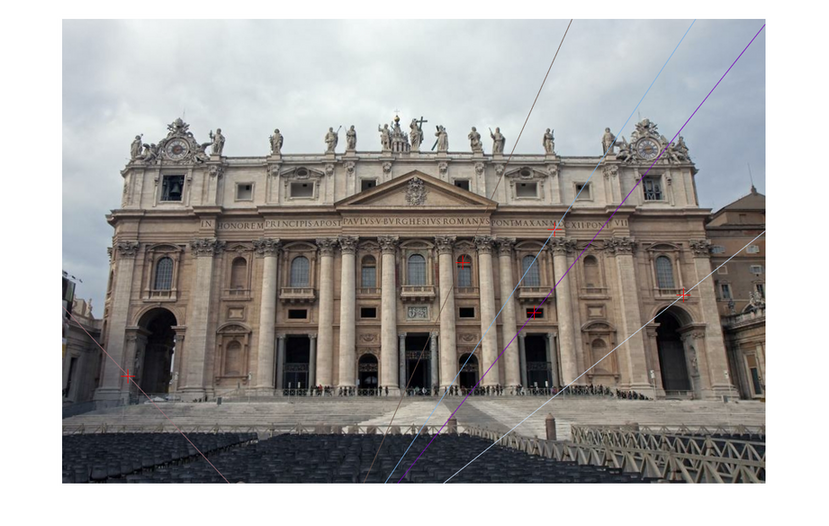}
    \includegraphics[height=0.1\linewidth]{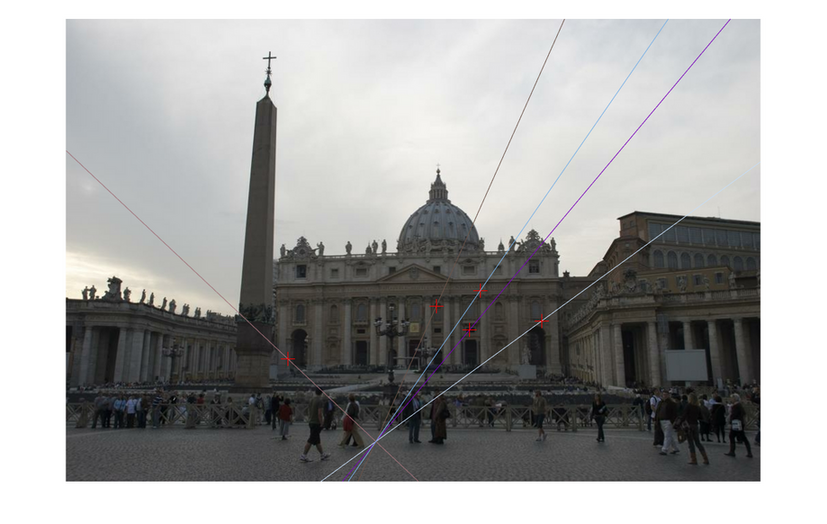} & \includegraphics[height=0.1\linewidth]{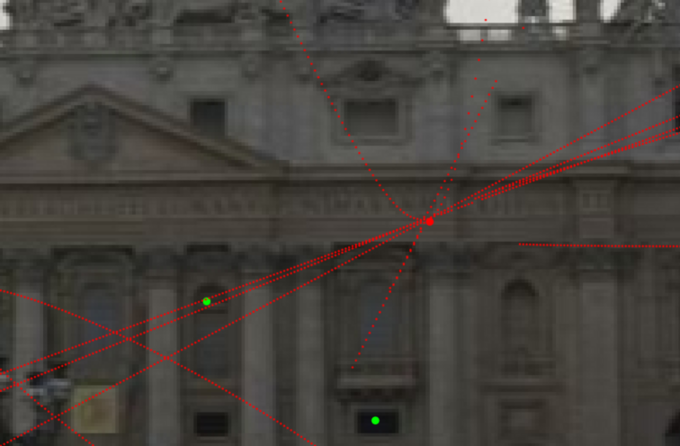} \\
    \hline
    \includegraphics[height=0.13\linewidth]{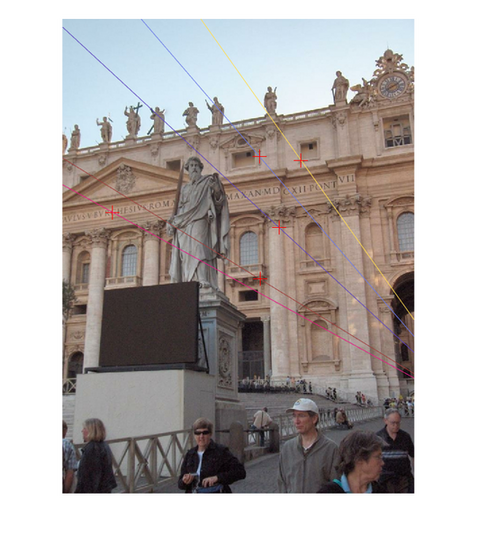}
    \includegraphics[height=0.13\linewidth]{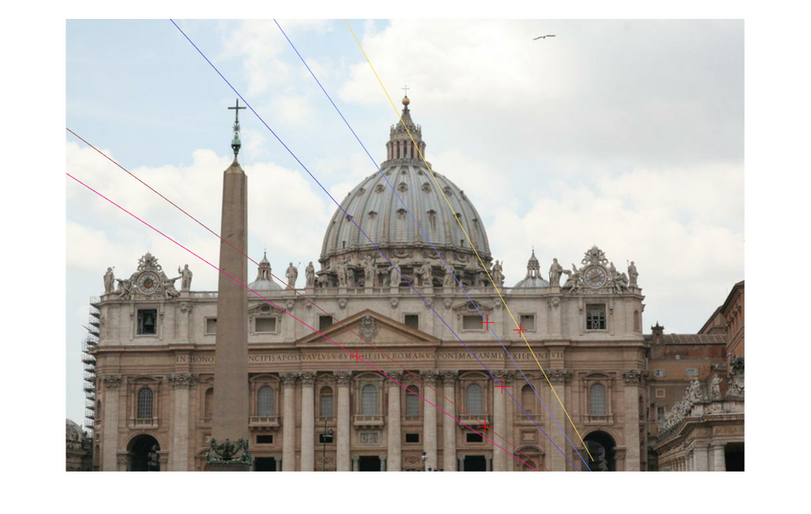} & \includegraphics[height=0.13\linewidth]{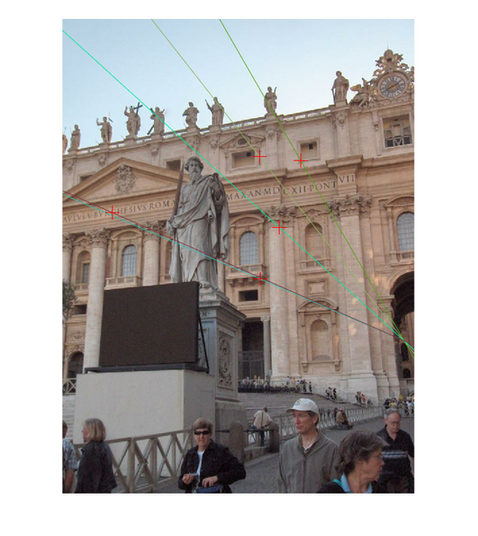}
    \includegraphics[height=0.13\linewidth]{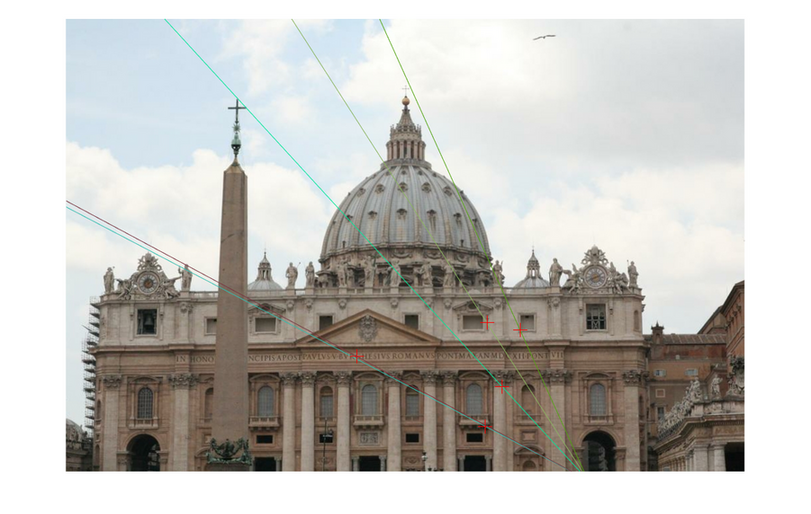} & \includegraphics[height=0.15\linewidth]{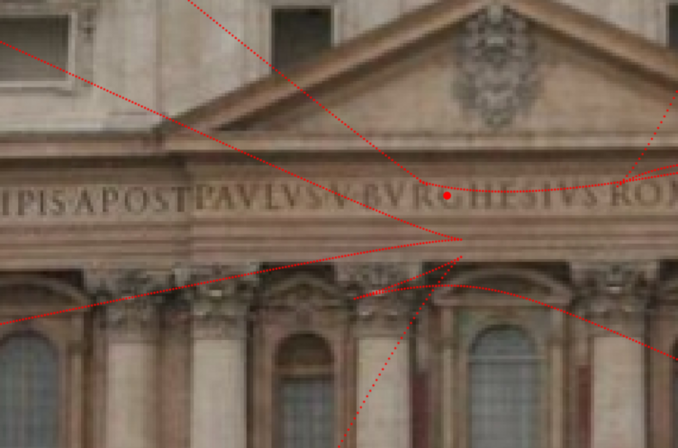} \\
    \end{tabular}
    \caption{Further samples of the 4.5-point degenerate curve on real images. The estimated epipolar geometry is the estimate closest to the ground truth among the multiple solutions to the minimal problem.}
    \label{tab:realE}
\end{table}

\end{document}